\documentclass[lettersize,journal,compsoc]{IEEEtran}
\usepackage{amsmath,amsfonts,booktabs}
\usepackage{algorithmic}
\usepackage{array,enumerate}
\usepackage[caption=false,font=normalsize,labelfont=sf,textfont=sf]{subfig}
\usepackage{textcomp}
\usepackage{stfloats}
\usepackage{url,xr}
\usepackage{verbatim}
\usepackage{graphicx}
\usepackage[section]{placeins}
\def\BibTeX{{\rm B\kern-.05em{\sc i\kern-.025em b}\kern-.08em
    T\kern-.1667em\lower.7ex\hbox{E}\kern-.125emX}}
\usepackage{balance}

\usepackage{amsthm}
\usepackage[table]{xcolor}
\usepackage{float}
\usepackage{paralist}
\usepackage{smartdiagram}
\usepackage{algorithm2e}
\usepackage{enumitem}
\usepackage{setspace}
\usepackage[mathlines, switch]{lineno}

\usepackage{tikz-qtree,tikz-qtree-compat}
\usepackage{tikz,pgf}
\usetikzlibrary{positioning,patterns}
\usetikzlibrary{calc,fadings,decorations.pathreplacing,arrows,
decorations.markings,
datavisualization.formats.functions,shapes.geometric,mindmap,shapes, positioning}
\usetikzlibrary{hobby}
\usetikzlibrary{matrix}
\usesmartdiagramlibrary{additions}

\usepackage{pgfplots}
\usepackage{pgfplotstable}
\pgfplotsset{compat=newest}
\usepgfplotslibrary{colormaps,fillbetween}

\def\wt{\widetilde}
\def\diag{\hbox{diag}}
\def\wh{\widehat}
\def\AIC{\hbox{AIC}}
\def\BIC{\hbox{BIC}}
\newcommand{\Appendix}
{
\def\thesection{Appendix~\Alph{section}}
\def\thesubsection{A.\arabic{subsection}}
}
\def\diag{\hbox{diag}}
\def\bias{\hbox{bias}}
\def\Siuu{\boldSigma_{i,uu}}

\def\dfrac#1#2{{\displaystyle{#1\over#2}}}
\def\VS{{\vskip 3mm\noindent}}
\def\boxit#1{\vbox{\hrule\hbox{\vrule\kern6pt
          \vbox{\kern6pt#1\kern6pt}\kern6pt\vrule}\hrule}}
\def\refhg{\hangindent=20pt\hangafter=1}
\def\refmark{\par\vskip 2mm\noindent\refhg}
\def\naive{\hbox{naive}}
\def\itemitem{\par\indent \hangindent2\pahttprindent \textindent}
\def\var{\hbox{var}}
\def\cov{\hbox{cov}}
\def\corr{\hbox{corr}}
\def\trace{\hbox{trace}}
\def\refhg{\hangindent=20pt\hangafter=1}
\def\refmark{\par\vskip 2mm\noindent\refhg}
\def\Normal{\hbox{Normal}}
\def\povr{\buildrel p\over\longrightarrow}
\def\ccdot{{\bullet}}
\def\bse{\begin{eqnarray*}}
\def\ese{\end{eqnarray*}}
\def\be{\begin{eqnarray}}
\def\ee{\end{eqnarray}}
\def\bq{\begin{equation}}
\def\eq{\end{equation}}
\def\bse{\begin{eqnarray*}}
\def\ese{\end{eqnarray*}}
\def\pr{\hbox{pr}}
\def\CV{\hbox{CV}}
\def\wh{\widehat}
\def\T{^{\rm T}}
\def\myalpha{{\cal A}}
\def\th{^{th}}

\newcommand{\corb}[1]{\textcolor{blue}{#1}}
\newcommand{\corred}[1]{\textcolor{black}{#1}}
\newcommand{\corblue}[1]{\textcolor{black}{#1}}

\newcommand{\jc}[1]{\textcolor{black}{#1}}
\newcommand{\hg}[1]{\textcolor{black}{#1}}
\newcommand{\cm}[1]{\textcolor{black}{#1}}

\newcommand{\jct}[1]{\textcolor{black}{#1}}
\newcommand{\hgt}[1]{\textcolor{black}{#1}}
\newcommand{\cmt}[1]{\textcolor{black}{#1}}
\newcommand{\po}[1]{\textcolor{black}{#1}} 

\renewcommand{\baselinestretch}{1} 
\newcommand{\bbR}{\mathbb{R}}
\newcommand{\bbX}{\mathbb{X}}
\newcommand{\bbN}{\mathbb{N}}
\newcommand{\bbE}{\mathbb{E}}
\newcommand{\bbF}{\mathbb{F}}
\newcommand{\bbS}{\mathbb{S}}
\newcommand{\bbK}{\mathbb{K}}
\newcommand{\bbC}{\mathbb{C}}
\newcommand{\bbJ}{\mathbb{J}}
\newcommand{\bbI}{\mathbb{I}}
\newcommand{\bbP}{\mathbb{P}}

\newcommand{\bR}{\mathbf{R}}
\newcommand{\bD}{\mathbf{D}}
\newcommand{\bI}{\mathbf{I}}
\newcommand{\bL}{\mathbf{L}}
\newcommand{\bG}{\mathbf{G}}
\newcommand{\bW}{\mathbf{W}}
\newcommand{\bP}{\mathbf{P}}
\newcommand{\bU}{\mathbf{U}}
\newcommand{\bC}{\mathbf{C}}
\newcommand{\bA}{\mathbf{A}}
\newcommand{\bB}{\mathbf{B}}
\newcommand{\bE}{\mathbf{E}}
\newcommand{\bF}{\mathbf{F}}
\newcommand{\bK}{\mathbf{K}}
\newcommand{\bM}{\mathbf{M}}
\newcommand{\bN}{\mathbf{N}}
\newcommand{\bJ}{\mathbf{J}}
\newcommand{\bH}{\mathbf{H}}
\newcommand{\bQ}{\mathbf{Q}}
\newcommand{\bS}{\mathbf{S}}
\newcommand{\bV}{\mathbf{V}}
\newcommand{\bX}{\mathbf{X}}
\newcommand{\bY}{\mathbf{Y}}
\newcommand{\bZ}{\mathbf{Z}}
\newcommand{\bh}{\mathbf{h}}
\newcommand{\bx}{\mathbf{x}}
\newcommand{\by}{\mathbf{y}}
\newcommand{\bv}{\mathbf{v}}
\newcommand{\bz}{\mathbf{z}}
\newcommand{\bs}{\mathbf{s}}
\newcommand{\ba}{\mathbf{a}}
\newcommand{\bb}{\mathbf{b}}
\newcommand{\bo}{\mathbf{o}}
\newcommand{\bc}{\mathbf{c}}
\newcommand{\bd}{\mathbf{d}}
\newcommand{\bbe}{\mathbf{e}}
\newcommand{\bff}{\mathbf{f}}
\newcommand{\bqq}{\mathbf{q}}
\newcommand{\bve}{\mathbf{e}}
\newcommand{\br}{\mathbf{r}}
\newcommand{\bu}{\mathbf{u}}
\newcommand{\bw}{\mathbf{w}}
\newcommand{\bg}{\mathbf{g}}
\newcommand{\bn}{\mathbf{n}}
\newcommand{\bk}{\mathbf{k}}
\newcommand{\bt}{\mathbf{t}}
\newcommand{\bbf}{\mathbf{f}}
\newcommand{\cS}{\cal{S}}
\newcommand{\bmu}{\boldsymbol{\mu}}
\newcommand{\bxi}{\boldsymbol{\xi}}
\newcommand{\bsigma}{\boldsymbol{\sigma}}
\newcommand{\bgamma}{\boldsymbol{\gamma}}
\newcommand{\btau}{\boldsymbol{\tau}}
\newcommand{\bneta}{\boldsymbol{\eta}}
\newcommand{\brho}{\boldsymbol{\rho}}
\newcommand{\blambda}{\boldsymbol{\lambda}}
\newcommand{\bdelta}{\mathbf{\delta}}
\newcommand{\btheta}{\boldsymbol{\theta}}
\newcommand{\bvartheta}{\boldsymbol{\vartheta}}
\newcommand{\bpsi}{\boldsymbol{\psi}}
\newcommand{\bphi}{\boldsymbol{\phi}}
\newcommand{\bepsilon}{\boldsymbol{\epsilon}}
\newcommand{\bvarepsilon}{\boldsymbol{\varepsilon}}
\newcommand{\balpha}{\boldsymbol{\alpha}}
\newcommand{\bbeta}{\boldsymbol{\beta}}
\newcommand{\bSigma}{\boldsymbol{\Sigma}}
\newcommand{\bLambda}{\boldsymbol{\Lambda}}
\newcommand{\bOmega}{\boldsymbol{\Omega}}
\newcommand{\0}{\mathbf{0}}
\newcommand{\1}{\mathbf{1}}
\newcommand{\binfty}{\boldsymbol{\infty}}
\newcommand{\E}{\mbox{E}}

\newcommand{\tc}[2]{\textcolor{#1}{#2}}
\def\scrU{\mathscr{U}}
\newcommand{\PP}{\mathbb{P}}
\newcommand{\idxset}{\Lambda}

\newcommand{\mcA}{{\mathcal A}}
\newcommand{\mcB}{{\mathcal B}}
\newcommand{\mcC}{{\mathcal C}}
\newcommand{\mcD}{{\mathcal D}}
\newcommand{\mcE}{{\mathcal E}}
\newcommand{\mcF}{\mathcal{F}}
\newcommand{\mcI}{{\mathcal I}}
\newcommand{\mcJ}{{\mathcal J}}
\newcommand{\mcK}{{\mathcal K}}
\newcommand{\mcL}{{\mathcal L}}
\newcommand{\mcM}{{\mathcal M}}
\newcommand{\mcN}{{\mathcal N}}
\newcommand{\mcO}{{\mathcal O}}
\newcommand{\mcP}{{\mathcal P}}
\newcommand{\mcR}{{\mathcal R}}
\newcommand{\mcQ}{{\mathcal Q}}
\newcommand{\mcS}{\mathcal{S}}
\newcommand{\mcT}{{\mathcal T}}
\newcommand{\mcU}{\mathcal{U}}

\newcommand{\ii}{\mathbf{i}}
\newcommand{\jj}{\mathbf{j}}
\newcommand{\pp}{\mathbf{p}}
\newcommand{\rr}{\mathbf{r}}
\newcommand{\mm}{\mathbf{m}}
\newcommand{\qq}{\mathbf{q}}
\newcommand{\UU}{\mathbf{U}}
\newcommand{\FF}{\mathbf{F}}
\newcommand{\aalpha}{\boldsymbol{\alpha}}
\newcommand{\rrho}{\boldsymbol{\rho}}
\newcommand{\ttheta}{\boldsymbol{y}}
\newcommand{\oone}{\boldsymbol{1}}

\newcommand{\Nset}{\mathbb{N}_0}
\newcommand{\cset}{{\mathbb C}}
\newcommand{\rset}{{\mathbb R}}
\newcommand{\nset}{{\mathbb N}}
\newcommand{\bbNset}{{\mathbb N}}
\newcommand{\qset}{{\mathbb Q}}
\newcommand{\pset}{{\mathbb P}}
\newcommand{\Pol}{\mathbb{P}}
\newcommand{\eset}[1]{{\mathbb E} \left[ #1 \right] }

\newcommand{\Grad}{\nabla}
\newcommand{\ssy}{\scriptscriptstyle}
\newcommand{\dist}{\operatorname{dist}}
\newcommand{\KL}{Karhunen--\Loeve }
\newcommand{\lv}{w}
\def\scrG{\mathscr{G}}
\newcommand{\Real}{\mathop{\text{\rm Re}}}
\newcommand{\Imag}{\mathop{\text{\rm Im}}}
\newcommand{\bno}{n}

\newcommand{\func}{u}

\DeclareMathOperator*{\esssup}{ess\,sup}
\DeclareMathOperator*{\essinf}{ess\,inf}

\newcommand{\sJ}[1]{
\begin{bmatrix*}[r]
  \bJ^{#1}_R   & -\bJ^{#1}_I  \\
  \bJ^{#1}_{I}  & \bJ^{#1}_R   \\
\end{bmatrix*}
}

\newcommand{\szv}[1]{
\begin{bmatrix}
  \bz^{#1}_{R} \\
  \bz^{#1}_{I} \\
\end{bmatrix}
}

\newcommand{\sfv}[1]{
\begin{bmatrix}
  \bbf^{#1}_{R} \\
  \bbf^{#1}_{I} \\
\end{bmatrix}
}

\newcommand{\sJflip}[1]{
\begin{bmatrix*}[r]
  -\bJ^{#1}_R   & \bJ^{#1}_I  \\
  \bJ^{#1}_{I}  & \bJ^{#1}_R   \\
\end{bmatrix*}
}

\newcommand{\szvflip}[1]{
\begin{bmatrix*}[r]
  -\bz^{#1}_{I} \\
  \bz^{#1}_{R} \\
\end{bmatrix*}
}

\newcommand{\sfvflip}[1]{
\begin{bmatrix*}[r]
  -\bbf^{#1}_{I} \\
  \bbf^{#1}_{I} \\
\end{bmatrix*}
}

\newcommand{\BallTaylor}{
\left(\bx_{0},
  \left[
    \begin{array}{c}
  \bqq \\
  \0
  \end{array}
  \right]
+ t
\left[\begin{array}{c}
  \bv_{R} \\
  \bv_{I}
  \end{array}
  \right]
  \right)
}

\def\sD{\mathcal{D}}
\def\sN{\mathcal{N}}
\def\sC{\mathcal{C}}
\def\R{\Bbb{R}}
\newcommand{\verteq}[0]{\begin{turn}{90} $=$\end{turn}}

\definecolor{darkgreen}{rgb}{0, 0.6, 0}
\definecolor{airforceblue}{rgb}{0.36, 0.54, 0.66}
\definecolor{applegreen}{rgb}{0.55, 0.71, 0.0}
\definecolor{asparagus}{rgb}{0.53, 0.66, 0.42}
\definecolor{cadetblue}{rgb}{0.37, 0.62, 0.63}
\definecolor{cambridgeblue}{rgb}{0.64, 0.76, 0.68}
\definecolor{olivine}{rgb}{0.6, 0.73, 0.45}
\definecolor{rufous}{rgb}{0.66, 0.11, 0.03}
\definecolor{sangria}{rgb}{0.57, 0.0, 0.04}
\definecolor{neworange}{rgb}{1, 0.64, 0}
\definecolor{flowblue}{rgb}{0.4471,    0.6235,    0.8118}
\definecolor{darkorange}{RGB}{255,140,0}
\definecolor{blueish}{rgb}{0.1176, 0.5647, 1.0000}
\definecolor{babyblue}{RGB}{153,204,255}
\definecolor{maroon}{RGB}{128,0,0}

\newcommand{\supess}{\mbox{ess} \operatornamewithlimits{sup}}
\newcommand*{\ackname}{Acknowledgements}

\newtheorem{theorem}{Theorem}[section]

\theoremstyle{remark}

\newtheorem{rem}{Remark}

\tikzstyle{block} = [draw, shade, drop shadow, rounded corners=1ex,
bottom color= 
green!40!,
minimum width=4cm]

\tikzstyle{newblock} = [draw, shade, drop shadow,rounded
  corners=1ex,top color=orange!70!, minimum width=4cm] 
  
\tikzstyle{obs} =
          [draw, shade, drop shadow,rounded corners=1ex,top color=orange!70,minimum width=2em]
          
          \tikzstyle{proc} = [draw,rectangle,rounded
            corners=1ex,fill=green!20!blue!60!, minimum width=2em]
          
          \tikzstyle{emptyblock} = [draw,minimum width=2em]
          \def\radius{.7mm}
          
          \tikzstyle{branch}=[fill,shape=circle,minimum size=3pt,inner
            sep=0pt] 
            
            \tikzstyle{vecArrow} = [thick, blue,
            decoration={markings,mark=at position 1 with
              {\arrow[semithick,blue]{open triangle 60}}}, double
            distance=1.4pt, shorten >= 5.5pt, preaction = {decorate},
            postaction = {draw,line width=1.4pt, white,shorten >=
              4.5pt}]

\makeatletter
\def\thm@space@setup{%
  \thm@preskip=0\topsep \thm@postskip=\thm@preskip
}
\makeatother

\newcommand{\argmax}{\operatornamewithlimits{argmax}}
\newcommand{\argmin}{\operatornamewithlimits{argmin}}
\DeclareMathOperator{\spn}{span}

\begin{document}
\title{deFOREST: Fusing Optical and Radar satellite data for Enhanced Sensing of Tree-cover loss
  \thanks{This article has been accepted for publication in IEEE Transactions on Geoscience and Remote Sensing. This is the author's version which has not been fully edited and
content may change prior to final publication. Citation information: 
J. E. Castrill\'{o}n-Cand\'{a}s, H. Gu, C. Meredith, Y. Li, X. Tang, P. Olofsson, and M. Kon,
\emph{deFOREST: Fusing Optical and Radar Satellite Data for Enhanced Sensing of Tree-Cover Loss},
\emph{IEEE Transactions on Geoscience and Remote Sensing}, vol. 64, Art. no. 4409213, 2026,
DOI: 10.1109/TGRS.2026.3689741.}
\thanks{© 2026 IEEE.  Personal use of this material is permitted.  Permission from IEEE must be obtained for all other uses, in any current or future media, including reprinting/republishing this material for advertising or promotional purposes, creating new collective works, for resale or redistribution to servers or lists, or reuse of any copyrighted component of this work in other works.}}

\author{Julio Enrique Castrill\'on-Cand\'as$^{1}$, Hanfeng Gu$^{2}$, Caleb Meredith$^{1}$, Yulin Li$^{1}$,  Xiaojing Tang$^{3}$, Pontus Olofsson$^{4}$, Mark Kon$^{1}$ 

\thanks{$^{1}$Department of Mathematics and Statistics, Boston University, Boston, USA. Emails: jcandas@bu,
cjmath@bu.edu, yulinli@bu.edu, mkon@bu.edu.}

\thanks{$^{2}$Department of Earth and Environment, Boston University, Boston, USA. Email: hanfengu@bu.edu.}

\thanks{$^{3}$College of Integrated Science \& Engineering, James Madison University,  Harrisonburg, VA, USA.  Email: tang3xx@jmu.edu  }
\thanks{$^{4}$ NASA Marshall Space Flight Center, Huntsville, AL, USA. Email: pontus.olofsson@nasa.gov}
}

\markboth{IEEE Transactions on Geoscience and Remote Sensing}{}




\maketitle


\begin{abstract} In this paper we develop a deforestation detection pipeline that incorporates optical and Synthetic Aperture Radar (SAR) data. A crucial component of the pipeline is the construction of anomaly maps of the optical data, which is done using the residual space of a discrete Karhunen-Lo\'{e}ve (KL) expansion. Anomalies are quantified using a concentration bound on the distribution of the residual components for the nominal state of the forest.  This bound does not require prior knowledge on the distribution of the data.  This is in contrast to statistical parametric methods that assume knowledge of the data distribution, an impractical assumption that is especially infeasible for high dimensional data such as ours. Once the optical anomaly maps are computed they are combined with SAR data, and the state of the forest is classified by using a Hidden Markov Model (HMM). We test our approach with Sentinel-1 (SAR) and Sentinel-2 (Optical) data on a $92\,km \times 92\,km$ region in the Amazon forest. The results show that both the hybrid optical-radar and optical only methods achieve high accuracy that is superior to the recent state-of-the-art hybrid method. Moreover, the hybrid method is significantly more robust in the case of sparse optical data that are common in highly cloudy regions.
\end{abstract}

\begin{IEEEkeywords}
Fusion, Discrete Karhunen-Lo\`{e}ve Expansions, Hidden Markov Models
\end{IEEEkeywords}


 \section{Introduction}

\IEEEPARstart{L}{and} use
and land cover changes caused by both natural and human
drivers have transformed the landscape globally \cite{winkler2021global}, 
and have significant impact on the surface energy balance,
hydrological cycle, and ecosystem services. Timely and accurate
monitoring of land use and land cover change provides crucial
information for the modeling of the Earth’s systems. Remote sensing
has been commonly used to map and monitor land use and land cover
change over large areas \cite{zhu2022remote}. Most of the past efforts
are retrospective, focusing on constructing a complete history of
changes during the past several decades (e.g. \cite{song2018global}).
While important, such products are often not updated frequently
enough to provide information on the most recent dynamics of land use
and land cover change. Certain events, such as illegal logging,
encroachment in protected areas, flooding, and other natural
disasters, require much faster responses. Analysis of massive data
sets and associated advances in Artificial Intelligence (AI) are
producing transformations in many aspects of society. Thanks to
the availability of vast remote sensing satellite datasets,
detection of tropical forest loss, in near
real-time, is now possible (e.g. \cite{tang2023near}). \po{Note that we use \textit{tree-cover loss} and \textit{forest loss} interchangeably throughout the paper; it is defined as complete canopy loss within a 10-m pixel with no immediate recovery. Further note that the lack of immediate recovery does not entail deforestation as the algorithm does not track the post-disturbance land cover. }

The density of cloud-free observations directly impacts the quality
and timeliness of a near real-time monitoring system
\cite{bullock2022timeliness}. This is problematic for regions
where the monitoring capability of optical sensors is hampered by the presence of clouds \cite{zhang2022global}. The amount of cloud
and cloud shadow missed by masking algorithms often results
in errors and so negatively affects the accuracy of monitoring. To
compensate for \po{such errors}, a monitoring algorithm would have to adapt to the
noise by increasing the number of consecutive observations of the
change signal for confirmation or adjusting the thresholds for change
detection -- \po{but} this would \po{in turn} affect the timeliness and accuracy of the
system. The use of Synthetic Aperture Radar (SAR) data (e.g.,
Sentinel-1 \cite{Torres2012}) can mitigate the data availability issue in cloudy
regions, as the SAR signal is not affected by clouds. Bullock et
al. \cite{bullock2022timeliness} and Reichie et al
\cite{reiche2021forest} have demonstrated the usefulness of Sentinel-1
data in monitoring \po{forest loss} in cloudy regions such as tropical
dry forests. However, SAR data is inherently noisy and is only useful
in tracking certain types of disturbances.

Combining data from optical and radar sensors is a logical way to
increase data density and improve the capacity for monitoring \po{forest loss} in near real-time. The abundance of freely-available
high-quality data collected by multiple remote sensing \po{missions} (e.g.,
Landsat, Sentinel-1, Sentinel-2 \cite{Drusch2012}, and NISAR), coupled
with advances in cloud computing technology and infrastructure,
offer a unique opportunity for monitoring 
using multi-sensor data fusion. \po{However, fusing optical and radar data is inherently complicated as the two sensing systems measure different signals.} Current data fusion approaches
for monitoring \po{the land surface} are often limited in
terms of geographic region, types of disturbance, and operational
readiness \cite{tang2023near,shang2022near,reiche2018improving}. 
\po{Hence, there is a need to develop multi-sensor approaches for monitoring land cover. While the approach presented in this paper is experimental and tested in a relatively small region, we aim to contribute to the development by introducing new advances in computational applied mathematics in remote sensing approaches.  The novel direction of the presented research is in the detection of anomalies in the combined optical and radar signal by interpreting the data as realizations of random vectors (or random fields) in a Bochner space \cite{Castrillon2025}  and constructing information function subspaces that are adapted to the nominal behavior. This approach involves tensor product representations, such as the Karhunen-Loève (KL)
expansion.} 


The KL expansion is strongly related to Principal
Component Analysis (PCA). PCA is widely used for building ML features by employing the principal components. However, most applications of PCA tend to ignore the probabilistic interpretation. In contrast, by using the KL
expansion of random fields (or random vectors for the discrete case), we conclude that it is not the principal components, but rather the residual eigenspace, that is important for detection and classification. This theory has been used to construct features for the classification of Alzheimer's disease with results that surpass state-of-the-art machine learning methods \cite{Castrillon2025}.


This approach is very different from the previous ones; KL expansions are in many senses the right tool for representing stochastic processes
and random fields, forming optimal tensor product representations. From its generality, large classes of processes and fields over complex geometrical domains can be represented with high accuracy \cite{Castrillon2025}. In contrast to current statistical approaches, from its core in functional analysis of tensor product expansions, our approach has many useful properties well suited to detection of hidden phenomena on complex domains.  In particular:
\begin{inparaenum}[i)]
\item Principled detection of anomalous global and local signals described as scalar
  or vector data \cite{Castrillon2022b}
\item Construction of non-parametric reliable hypothesis tests using strong concentration inequalities  conditioned only on covariance structure, with no other assumptions on distributions of data (important)
\item Filters that can process massive quantities of data with near-optimal performance.
\end{inparaenum}
Note that in \cite{Lakhina2004} a similar approach was developed using
the residual subspace of the principal components of PCA for the detection of network traffic anomalies. However, that was done in the context of PCA and not KL, thus a mathematical probabilistic rationale was not fully developed.

We tested our approach \hgt{for detecting forest loss} in the Amazon forest for a region of approximately $92\,km \times 91\,km$  and compared it to
the recent Fusion Near Real-Time (FNRT) algorithm \cite{Tang2023}, \hgt{the Global Land Analysis and Discovery (GLAD) Forest Alerts, and the RADD Forest Disturbance Alert. In this implementation, we define forest loss as a complete transfer of forest to non-forest within a 10-meter pixel with no immediate recovery. Sub-pixel canopy loss, such as selective logging, while still detectable by our algorithm after parameter tuning, is not a target for detection within the scope of this study.} 
In contrast to FNRT, our approach is highly robust and accurate for time frames that have sparse optical data, making it suitable for regions with persistent cloudy areas.

\section{Technical Approach and Methodology} \label{sec:tech}


We introduce a new approach for detecting subtle phenomena
in general datasets, including remote sensing data, by introducing a
novel mathematical framework. This is essential because current
advanced statistical methods often depend on assumptions about data
distributions that are either unrealistic or difficult to
validate. Our approach recognizes that, even when observations are
high-dimensional or diffuse, clear distinctions can emerge within the
appropriate stochastic function space. By constructing stochastic
tensor product maps, we can uncover
differences between phenomena. This is significant because previous
multimodality methods either assumed independence or imposed
artificial covariance structures. Our new theory leverages stochastic
functional analysis of tensor product representations,
utilizing the KL expansion. However, the mathematical 
presentation is simplified in this paper to the discrete case.



In \textbf{Figure} \ref{PR:Fig1} the pipeline for
detection of deforestation and cloud cover is shown. This pipeline
consists of the following modules:

\begin{itemize}
    
    \item \textbf{Training Dataset:} This data is used to build the anomaly 
    filter and machine learning features. It is the input into the KL Module.  
    The training data consist of the nominal state of the land cover,
    such as Enhanced Vegetation Index (EVI) measurements of the initial
    state of a forest.
    
    \item \textbf{Covariance Eigenstructure:} 
    A covariance matrix and the corresponding eigenpairs are
    constructed from the training dataset measurements. The
    training dataset is assumed to correspond to a number $M_T$ of time samples of the nominal
    state of the land cover.
    
    \item \textbf{Kahunen-Lo\`{e}ve:} A truncated KL expansion is constructed from the eigenpairs of the covariance matrix. However, only the eigenvalues and eigenvectors are needed.
    
    \item \textbf{Novel Optical Data:} From the novel testing dataset we can 
    now use the KL expansion to construct an anomaly map of the EVI data.

 \item \textbf{Anomaly Map:} From the eigenvectors the optical novel data is projected onto the truncated eigenspace and the residual map is constructed. The residual map effectively describes the anomaly intensity that can be used to detect changes in land cover. The anomaly corresponds to deviations from the initial state of the forest (training data). For example, if a particular pixel is sampled from an initial state, which is a forest, then the anomaly would be non-forest, while a reversal of these two roles would also be sought in the same way.

\item \textbf{Novel SAR Data:} \jc{The SAR dataset is filtered using a Bayesian approach both in space and time. This is described in detail in the supplementary material.}

\item \textbf{Finite State Machine} \jc{From the input data the state of the land cover can be detected. For the optical anomaly map  and the SAR data a Hidden Markov Model (HMM) with the Viterbi algorithm (See Chapter 12 for details in \cite{Ewens2005} and \cite{Kon2025slides}) is used to track the land cover changes in the forest. Note that the fusion of the optical and radar bands is performed by choosing an appropriate HMM model that incorporates the transition and emission probabilities of the optical and radar data.}

\end{itemize}

\begin{figure*}[htbp]
\centering
\begin{tikzpicture}[scale = 0.85, every node/.style={scale=0.65},>=latex']

    \node at (3,-3.8) {\includegraphics[scale = 0.17]{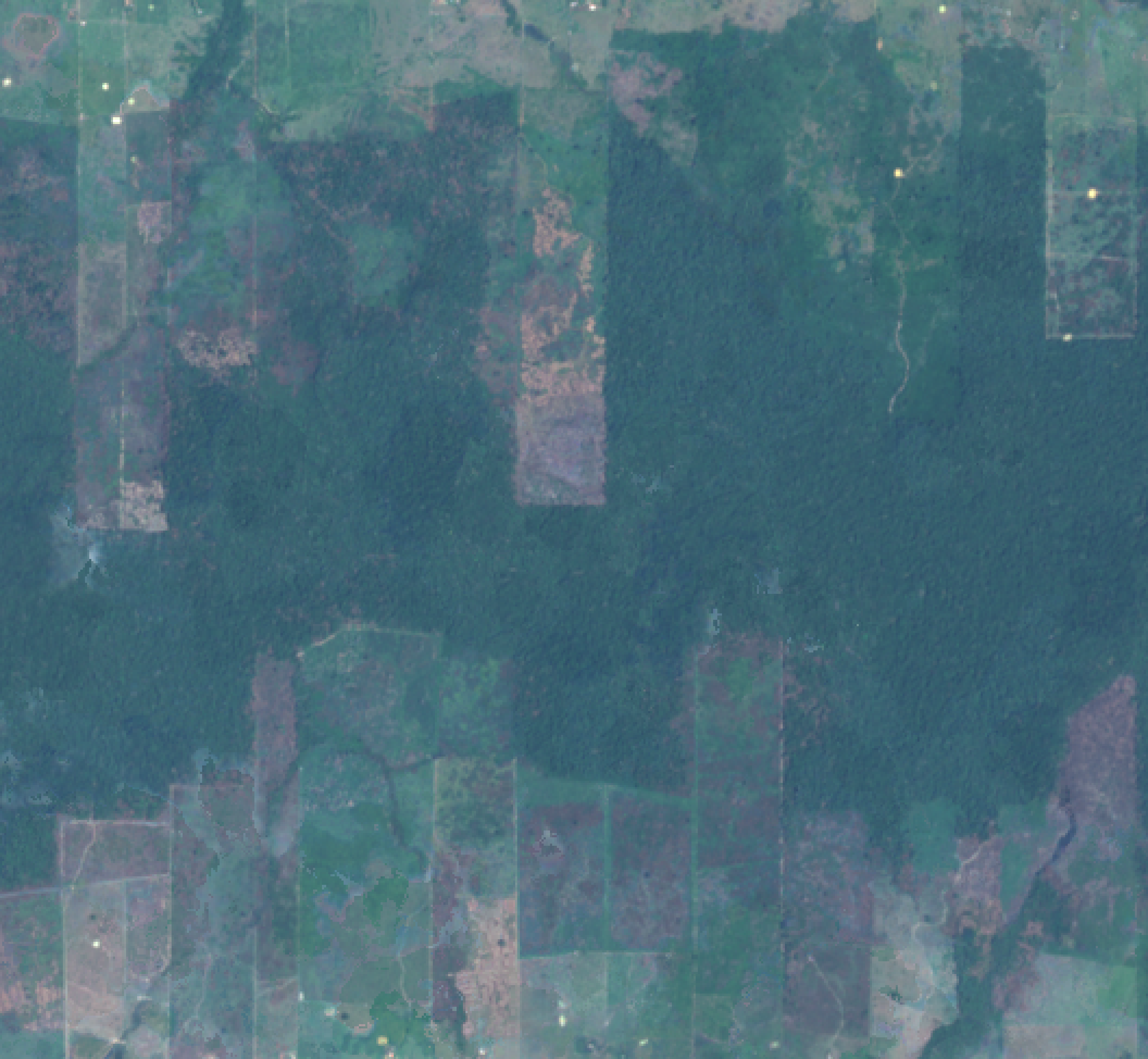}};     
    \node at (13,-2.4) {\includegraphics[scale = 0.25]{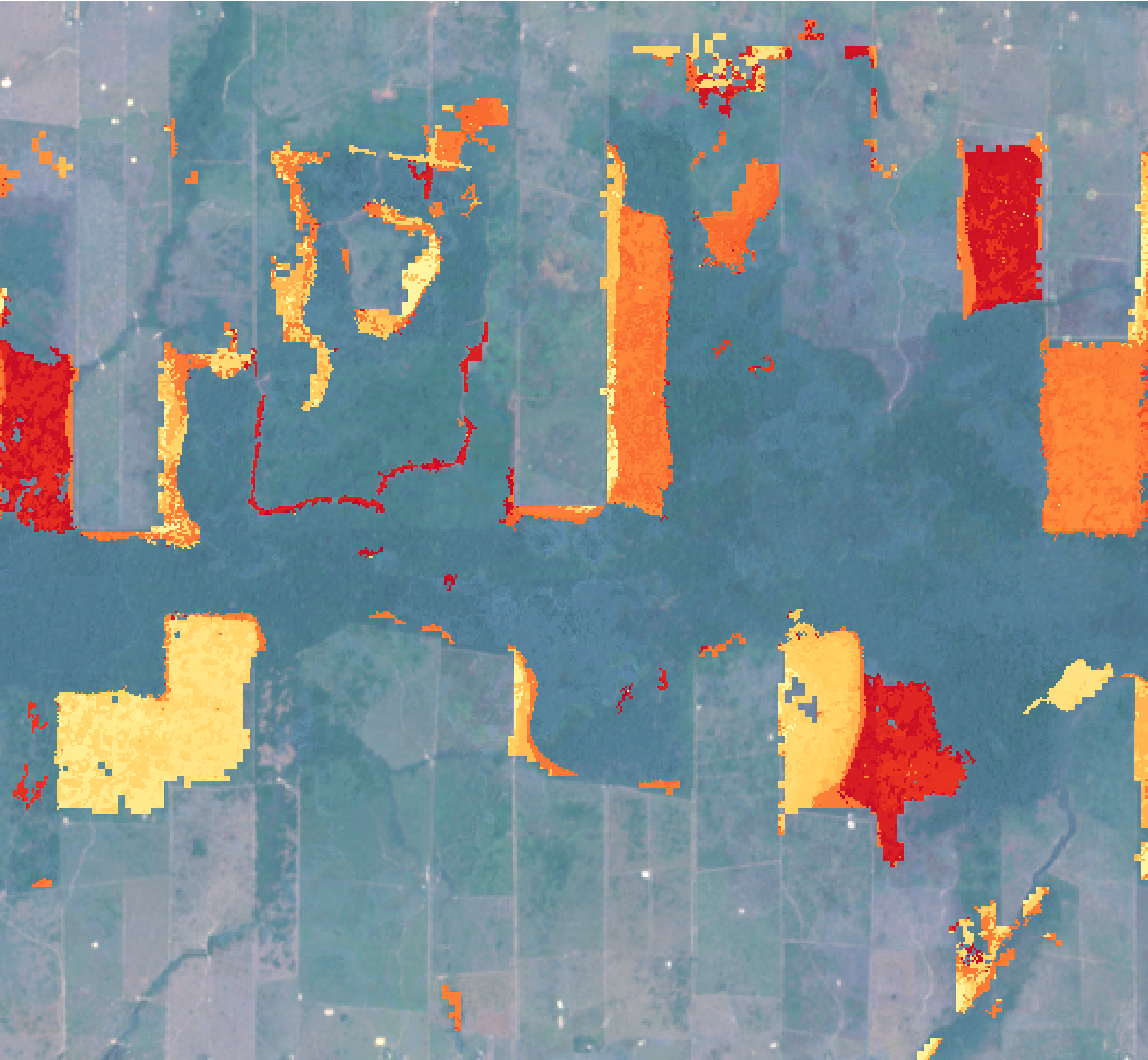}};

    \node at (3,-0.43) {\includegraphics[scale = 0.093]{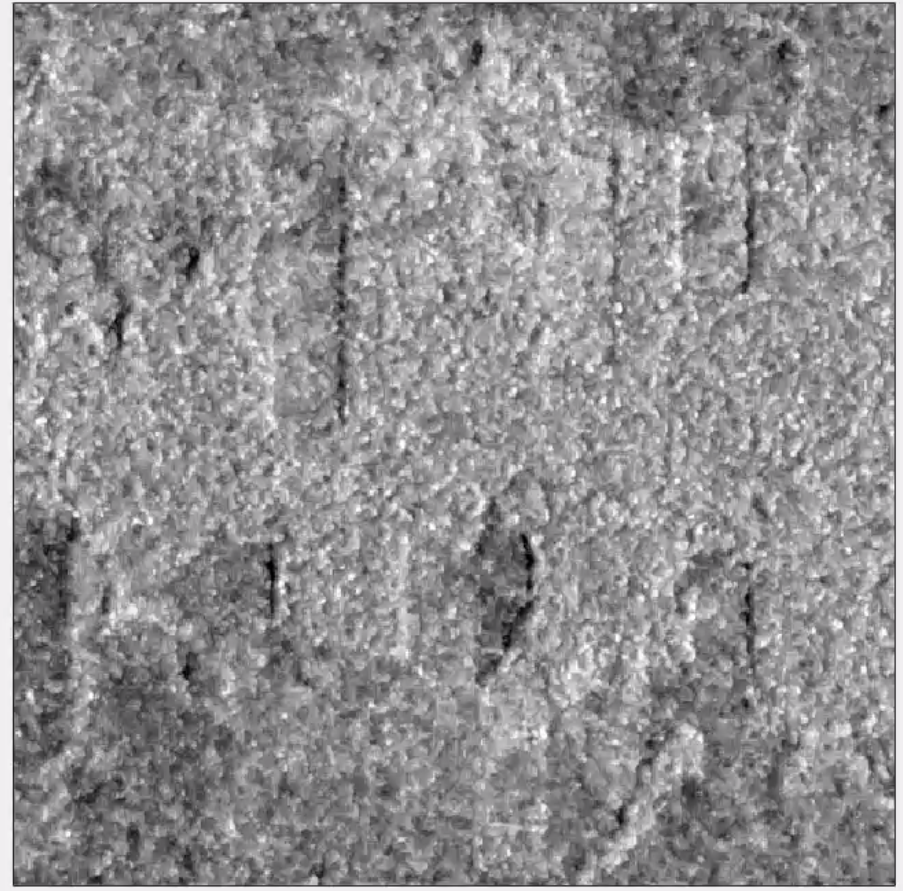}};

    \node at (3,2) {\Large Data};
    \node at (8,2) {\Large Filter};
    \node at (13,2) {\Large Classification};
           
    \draw[darkgreen!50,line width=0.5mm,dashed] (5.5,2) -- (5.5,-6);
    \draw[darkgreen!50,line width=0.5mm,dashed] (10.5,2) -- (10.5,-6);
     
    \node[block] at (3,-2.1) (sourcenew)
    { \begin{tabular}{c} 
     Novel Sentinel-2 \\
     Optical Data
       \end{tabular}};

        \node[block] at (3,1.25) (radar)
    { \begin{tabular}{c} 
     Novel Sentinel-1 \\
      SAR Data
    
       \end{tabular}};

       \node[block] at (3,-5.5) (source)
    { \begin{tabular}{c} 
     Sentinel-2 Optical \\ Training  Dataset
       \end{tabular}};
       
    \node[block] at (13,1.25) (measures)        
         {\begin{tabular}{c}
           Land Cover \\ State 
           \end{tabular}};
            
    \node[block] at (8,1.25) (filter)            
           {
           \begin{tabular}{c}
           Finite State \\ Machine
           \end{tabular} 
           };
           \draw[vecArrow] (filter.east) -- (measures.west);
           \draw[vecArrow] (radar.east) -- (filter.west);

            \node[block] at (8,-2.1) (spaces)            
           {\begin{tabular}{c} 
           Anomaly \\
           Map 
           \end{tabular}} ;
           \draw[vecArrow] (spaces.north) -- (filter.south);
         
            \node[block] at (8,-3.8) (expansion)            
           {\begin{tabular}{c} 
           Karhunen Lo\`{e}ve \\
           Expansion
            \end{tabular}} ;
           \draw[vecArrow] (expansion.north) -- (spaces.south);
            \node[block] at (8,-5.5) (spatio)            
           {\begin{tabular}{c} 
           Covariance 
           Eigenstructure \\
           $(\lambda_k,\phi_k)$           
           \end{tabular}};
           \draw[vecArrow] (spatio.north) -- (expansion.south);
           \draw[vecArrow] (source.east) -- (spatio.west);
           \draw[vecArrow] (sourcenew.east) -- (spaces.west);

\end{tikzpicture}
\caption{Schematic of monitoring land cover fusion pipeline for remote sensing data. This may include optical and radar data.}
\label{PR:Fig1}
\end{figure*}
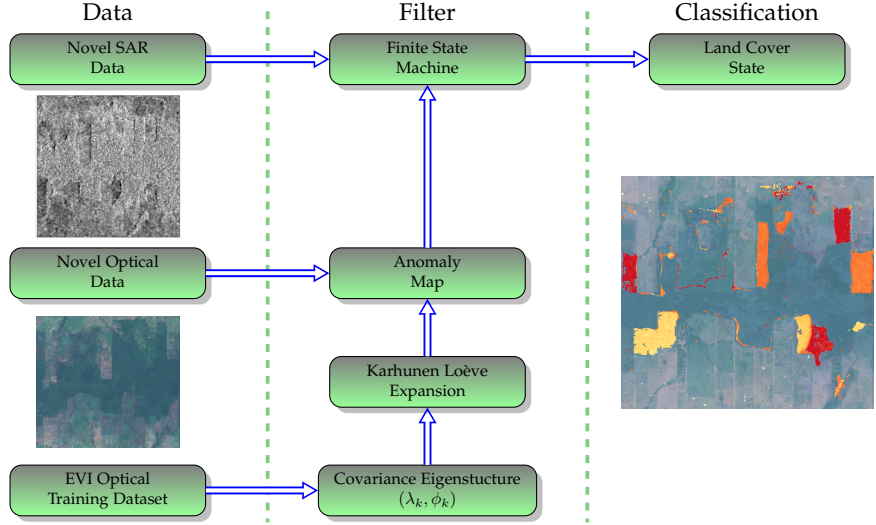

\subsection{Discrete KL expansions}

The KL expansion is a popular method for representing stochastic
processes and random fields. The KL expansion can be
used for a statistical approach to the detection of anomalies with the
interesting characteristic that the particular distribution of the data is not assumed or needed beforehand.  Due to its
simplicity, we shall describe the discrete KL expansion instead of the
continuous version. The
theorems contained in this paper can be proved from simplified
arguments in our publication
\cite{Castrillon2022b}.

Suppose that $\bv$ is a random vector in $\R^n$,  and $\bC:= \eset{(\bv
- \eset{\bv}) (\bv - \eset{\bv})^{T}}$.  The $i^{th}$ component of
$\bv$ corresponds to a sensor value (this can be extended to multiple
sensor values such as multispectral and radar
data \cite{Castrillon2022b}) in the spatial map. The theory developed
in this paper will be strongly based on the following result.

\begin{theorem}
Let $\bv(\omega) = [v_1(\omega), \dots, v_n(\omega)] \in L^{2}(\Omega;\R^{n})$
be a random vector and covariance matrix
$\bC := \eset{(\bv - \eset{\bv})(\bv - \eset{\bv})^{T}}$. Suppose that $\bC$
is a positive definite matrix with eigenpairs $(\lambda_{k},\bphi_{k})$ such
that for $k = 1,\dots,n$ 
\begin{equation}
\bC\bphi_k = \lambda_{k} \bphi_k,
\end{equation}
and $\lambda_1 \geq \dots \geq \lambda_n$
then there exists a set of zero-mean random variables $Y_1(\omega), \dots Y_{n}(\omega)$
such that 
\begin{equation}
\bv(\omega) = \eset{\bv(\omega)} +  \sum_{k = 1}^{n} \sqrt{\lambda_k} \bphi_k Y_k(\omega),
\label{KLE}
\end{equation}
where $\eset{Y_k(\omega)Y_l(\omega)} = \delta[l-k]$.
\end{theorem}

\begin{rem}
Note that the eigenvectors of the discrete KL expansion in equation
\eqref{KLE} exactly correspond to the principal components. Furthermore, the
eigenvalues indicate the level of variability of the signal.
\end{rem}

A crucial characteristic of the KL expansion is the optimality properties.  
Suppose that we form the truncated KL expansion i.e. for any $m \leq n$
\begin{equation}
\bv_{m} = \eset{\bv} +  \sum_{k = 1}^{m} \sqrt{\lambda_k} \bphi_k Y_k.
\end{equation}
It can be shown that such representation is optimal.
\begin{theorem}
Suppose $\eset{\bv(\omega)} = \0$,  $\bpsi_{1},\dots,\bpsi_{n}$ is an orthonormal basis of
$\R^{n}$ and let $\bQ^m$ be a projection of $\bv(\omega)$ onto
$\bpsi_{1},\dots,\bpsi_{m}$, then
\begin{equation}
    \begin{split}
    \eset{\|\bv(\omega) -  \bv_{m}(\omega)\|^{2}} 
    &= \sum_{k=m+1}^{n} \lambda_{k} \\
    &\leq \eset{
    \|\bv(\omega) 
    - \bQ^m \bv(\omega) \|^{2}
    }
    \end{split}.
\end{equation}
\end{theorem}

\subsection{Discrete Karhunen-Lo\`{e}ve expansions application to anomaly detection}
Due to the optimality properties of the KL expansion, a strong 
hypothesis test for presence of the anomaly can be formed.  Suppose that
$\bv \in \R^{n}$ is a random vector that describes the nominal state
of the land cover. Now, let $\bu \in \R^{n}$ be a realization of the
optical data. We want to form the hypothesis test for the observation
$\bu$ to test if it is from the nominal state of the land cover or
from the anomalous:
\medskip
\begin{center}
\begin{tabular}{l l l l}
     $\mbox{H}_0$  :& $\bu = \bv$ (No anomaly)  &$\mbox{H}_{A}$:& $\bu \neq \bv$ (Anomaly). 
\end{tabular}
\end{center}
\medskip
Suppose that $\bP^m$ is the projection of $\bv$ onto the eigenvectors $\phi_1,\dots,\phi_m$. We can 
then form the residual vector
\begin{equation}
\br  = \bv - \bv_{m} = \bv - \eset{\bv} -  \bP^m(\bv - \eset{\bv}) 
= \sum_{k=m+1}^{n}  \sqrt{\lambda_k} \bphi_k Y_k.
\end{equation}
Let $\alpha$ be the significance level then it can be shown (See Theorem
\ref{appendix:thm3} in the supplementary material) that the distribution of the null hypothesis $\mbox{H}_0$ satisfies the following bound
\begin{equation}
\bbP\left(|\br[i]| \geq \alpha^{-\frac{1}{2}}
\left(\sum_{k = m + 1}^{n}\lambda_k \bphi_k[i]^2 \right)^{\frac{1}{2}} 
\right) \leq \alpha.
\label{residual}
\end{equation}
From this concentration bound the probability for the null Hypothesis
can be computed. If for a given observation the null hypothesis $H_0$ is true
then $\bu = \bv$ and we can form the vector $\bneta := (\bu - \eset{\bv})
- \bP^m(\bu - \eset{\bv}) = \bv - \bv_m = \br$. From
equation \eqref{residual} the distribution of $|\bneta|$ will be
concentrated around zero if the eigenvalues decay sufficiently rapidly and $m$ is
sufficiently large. In contrast, if $H_{A}$ is true then $\bneta
:= (\bu - \eset{\bv}) - \bP^m(\bu - \eset{\bv}) \neq \bv - \bv_m
= \br$. Thus the distribution of $|\bneta|$ will in general not be
controlled by the bound in equation \eqref{residual} 
(See \textbf{Figure} \ref{Separation}).

By forming the vector $\bneta := (\bu - \eset{\bv})
- \bP^m(\bu - \eset{\bv})$ the class distinctions between nominal and
anomalous data are more clearly distinguished. This makes it easier to
train classifiers such as Hidden Markov Models (HMM) and Support
Vector Machines (SVM) \cite{Scholkopf1997}.


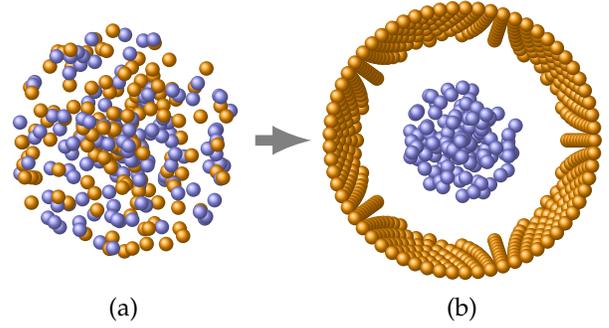
\begin{figure}[htbp]

\begin{center}
\begin{tikzpicture}
\begin{scope}[scale = 0.5]
    \tikzset{
  my ball/.style={
    ball color=#1,
    circle,
    minimum size=5.5pt,
    inner sep=0pt,
    outer sep=0pt,
    shading=ball,
  }
}
\pgfmathsetseed{2025}
\foreach \i in {1,...,150} {
  \pgfmathsetmacro\r{rand*1.5}
  \pgfmathsetmacro\a{rand*360}
  \pgfmathsetmacro\x{\r*cos(\a)}
  \pgfmathsetmacro\y{\r*sin(\a)}
  \node[my ball=blue!40!white,scale = 1] at (\x,\y) {};
}
\pgfmathsetseed{100}
\foreach \j in {0,...,9} {
  \pgfmathsetmacro\R{2.8 + 0.08*\j}
  \pgfmathsetmacro\N{int(20 + 5*\j)}
  \foreach \i in {0,...,99} {
    \ifnum\i<\N
      \pgfmathsetmacro\angle{360/\N*\i}
      \pgfmathsetmacro\x{\R  * cos(\angle)}
      \pgfmathsetmacro\y{\R  * sin(\angle)}
      \node[my ball=orange!80!yellow,scale = 1] at (\x,\y) {};
    \fi
  }
}
\node at (0,-4.5) {(b)};
\end{scope}

\begin{scope}[scale = 0.5]
    \tikzset{
  my ball/.style={
    ball color=#1,
    circle,
    minimum size=5.5pt,
    inner sep=0pt,
    outer sep=0pt,
    shading=ball,
  }
}
\pgfmathsetseed{2025}
\foreach \i in {1,...,150} {

  \pgfmathsetmacro\r{rand*3}
  \pgfmathsetmacro\a{rand*360}
  \pgfmathsetmacro\x{\r*cos(\a)}
  \pgfmathsetmacro\y{\r*sin(\a)}
  \node[my ball=blue!40!white,scale = 1] at (\x-9,\y) {};
  \pgfmathsetmacro\r{rand*3}
  \pgfmathsetmacro\a{rand*360}
  \pgfmathsetmacro\x{(\r *cos(\a)}
  \pgfmathsetmacro\y{(\r *sin(\a)}
  \node[my ball=orange!80!yellow,scale = 1] at (\x-9,\y) {};
}
\node at (-9,-4.5) {(a)};
   \coordinate (O) at (-5.5,0);
    \coordinate (P) at (-4,0);
    \draw[->, >=latex, gray, line width=4 pt] (O) -- (P);
\end{scope}

\end{tikzpicture}

\end{center}
\caption{Illustrative example of the separation capabilities of the KL
expansion by applying the transformation to the nominal and anomalous
data.  (a) The blue balls represent the nominal behavior such as the
starting state of the land cover and orange balls the signal anomaly
(changes in the land cover state).  These observations points are
mixed with each other, which makes it hard to build a decision
surface. (b) After forming the residual $\bneta := (\bu - \eset{\bv})
- \bP^m(\bu - \eset{\bv})$, the blue balls correspond to coefficients
$r_k$ that are subject to the null hypothesis $H_0$ (nominal
class). Thus from equation \eqref{residual} the coefficients are
centered around the origin with high probability. Conversely, under
the alternative hypothesis $H_A$ (signal anomaly) the coefficients
$\tilde r_k$ (orange balls) are likely not to concentrate around
zero. This makes it easier to build a separation surface for the two
classes.}
\label{Separation}
\end{figure}

\begin{rem}
It is important to note that for this hypothesis test no assumptions
are made about the distribution of the data, which is practically impossible to
estimate for high
dimensional and/or complex problems. One of the key weaknesses of many modern parametric
statistical methods is the assumption that the distribution is known
(i.e. Normal, Poisson, etc). For high dimensional complex data this
assumption is not reasonable. Furthermore estimating the distribution
for high dimensional data is also intractable since this problem suffers
from the curse of dimensionality, meaning that the amount of data needed explodes
exponentially with respect to the dimension. In contrast, the approach we introduce
here only requires knowledge of the covariance function, which is a
significantly easier problem.
\end{rem}

\begin{rem}
\cm{An implicit assumption here is that our random vectors contain no missing data, which is not the case for cloudy regions.  Section \ref{Supp:missing} in the supplemental material explains how the contributions of this missing data are left out, as well as ways in which the missing data can be filled in instead to improve performance.}
\end{rem}



\subsection{Anomaly detection and land cover classification using optical data}

\jc{For simplicity, we show the construction of the HMM for scalar optical data. The HMM model can be easily extended to the multi-band case by appropriately defining the emission probabilities of the observations.  Many of the details of the HMM and the Viterbi algorithm in this section can be found in \cite{Ewens2005}. Furthermore, see the lecture slides in \cite{Kon2025slides} for an excellent exposition.}

Suppose we have a set of discrete time points $\tau_0,\tau_1, \dots \tau_s \in
[0,S]$, and the corresponding observations of the optical and SAR 
sensors $\bv(\tau_0), \dots,$ $\bv(\tau_s)$. The time interval
$[0,S]$ corresponds to the nominal behavior of the land cover. For
example, this would correspond to a time period where the state of the
forest does not change much. From these samples the covariance matrix
$\bC$ is obtained. Now, suppose we have a set of discrete time points
$t_{0},\dots,t_f \in [S, T_{\textrm{final}}]$ and the corresponding
observations of the optical sensor
$\bu(t_{0}), \dots,$ $\bu(t_f)$. From the eigenvectors
$\phi_1,\dots,\phi_M$ of the matrix $\bC$, the projection matrix
$\bP^{m}$ is formed for a fixed truncation parameter $m$. The
observation vectors $\bu(t_{0}), \dots,$ $\bu(t_f)$ can now be
converted to the new features $\bneta(t_k) := (\bu(t_k) - \eset{\bv})
- \bP^m(\bu(t_k) - \eset{\bv})$ for $k = 0,\dots,f$.  

We now present an example of the behavior of the features
$\bneta(t_k)$ on a series of EVI calculated based on Sentinel-2 data. Each EVI
image consists of $150 \times 150$ pixels at the resolution of $10 m$.
In \textbf{Figure} \ref{PR:Fig2} we show the evolution of the forest
with respect to time. Notice that scattered trees were removed from
the forest but grew back over time. The last image corresponds to a
cloudy day, where the cloud removal algorithm has trouble detecting
the clouds, with only a subset of them removed (black areas).

\textbf{Scalar anomaly detection:}
From 71 Sentinel-2 EVI images starting from day 1 up to day 3200, the
covariance matrix is computed and the projection
operators $\bP^m$ are constructed from the eigenvectors. The projection operator $\bP^m$ is then
applied to each of the test frames starting from day 3300
(corresponding to frame number 1 on \textbf{Figure} \ref{PR:Fig2}) and
an anomaly features $\bneta$ are
constructed. In \textbf{Figure} \ref{PR:Fig3} the anomaly sequence for
the pixel corresponding to the red square in
\textbf{Figure} \ref{PR:Fig2} is shown. From the anomaly sequence in
\textbf{Figure} \ref{PR:Fig3} we see that
the \po{forest loss} occurs around day 3484 but reduces to the nominal
level by day 3704. This is due to the regrowth of leaves from adjacent
trees.

\begin{figure*}[ht]
\centering
\begin{tikzpicture}[scale = 0.98, every node/.style={scale=0.98},>=latex']
     \node at (0,0) {\includegraphics[scale = 0.36, trim = 5cm 8cm 5cm 6.5cm, 
         clip]{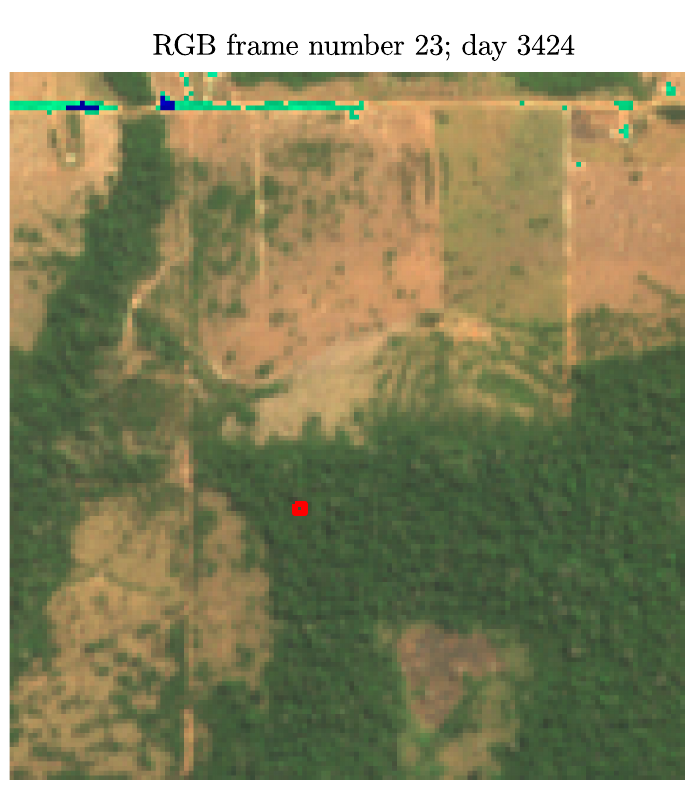}};
     \node at (4.1,0) {\includegraphics[scale = 0.36, trim = 5cm 8cm 5cm 6.5cm, 
         clip]{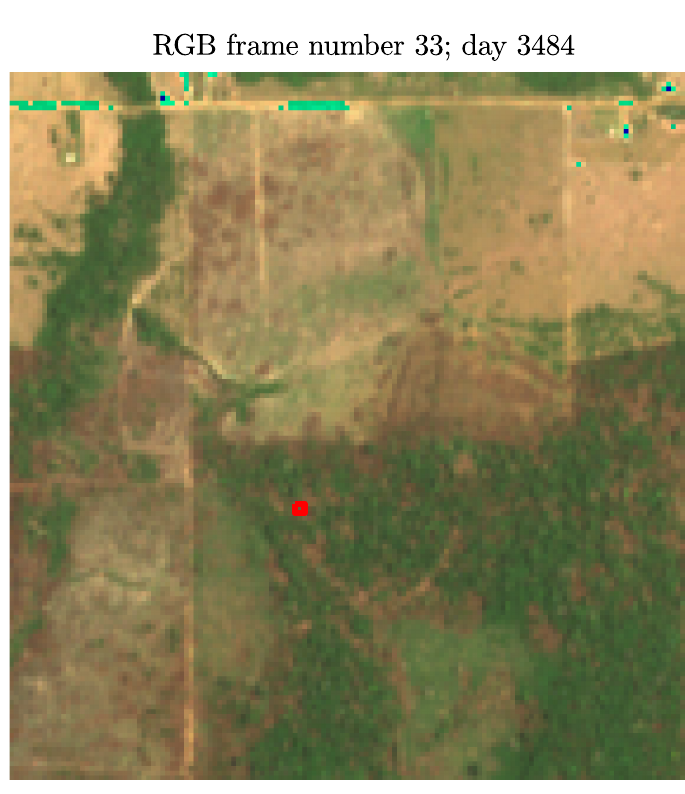}};
     \node at (8.25,0) {\includegraphics[scale = 0.35, trim = 5cm 8cm 5cm 6.5cm, 
         clip]{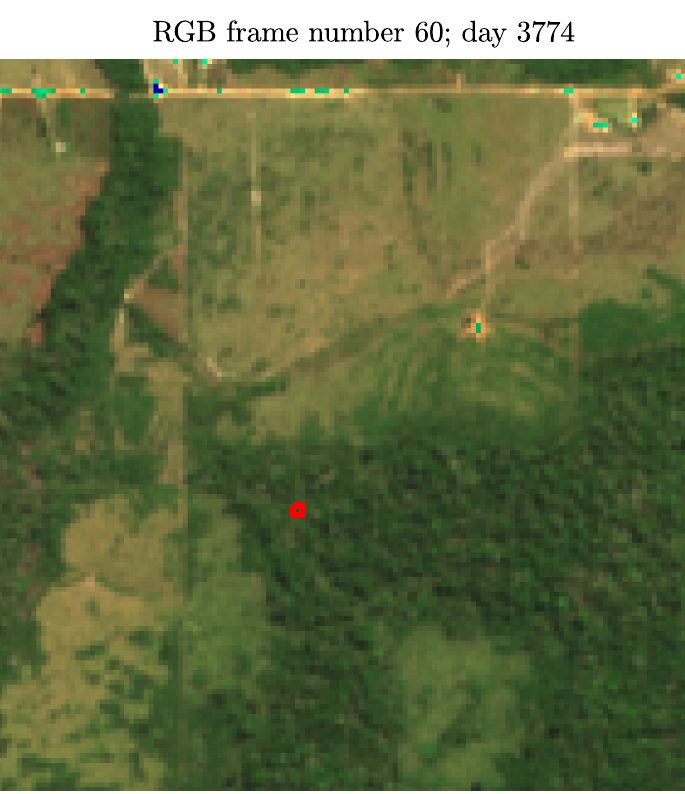}};
     \node at (12.35,0) {\includegraphics[scale = 0.35, trim = 5cm 8cm 5cm 6.5cm, 
         clip]{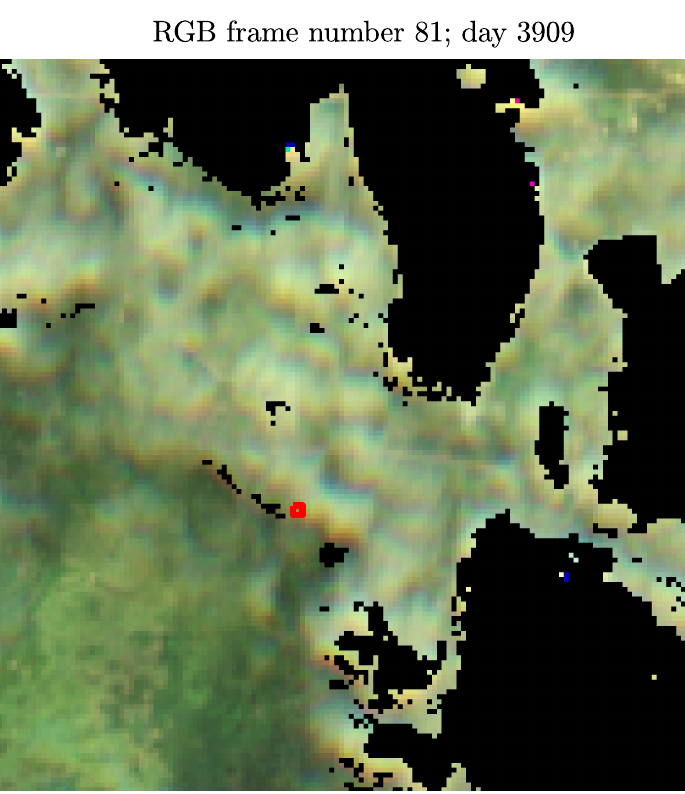}};
\end{tikzpicture}
\caption{\po{Forest loss in} Sentinel-2 data of
  the Brazilian Amazon. The data show a logging event and subsequent
  recovery of the forest.  The \po{anomaly detection} is applied with the goal of detecting
  the timing and location of the change. The feature information will
  be used to track the state of the forest. \po{The red box is the pixel for which the anomaly sequence is depicted in \textbf{Figure} \ref{PR:Fig3}.}}
\label{PR:Fig2}
\end{figure*}

\begin{figure}[htbp]
\centering
\begin{tikzpicture}[scale = 1.14,every node/.style={scale=1.14}]>=latex']
  \node at (0,0) {\includegraphics[scale = 0.45, trim = 2.20cm 14cm 2.25cm 7cm, clip]
                 {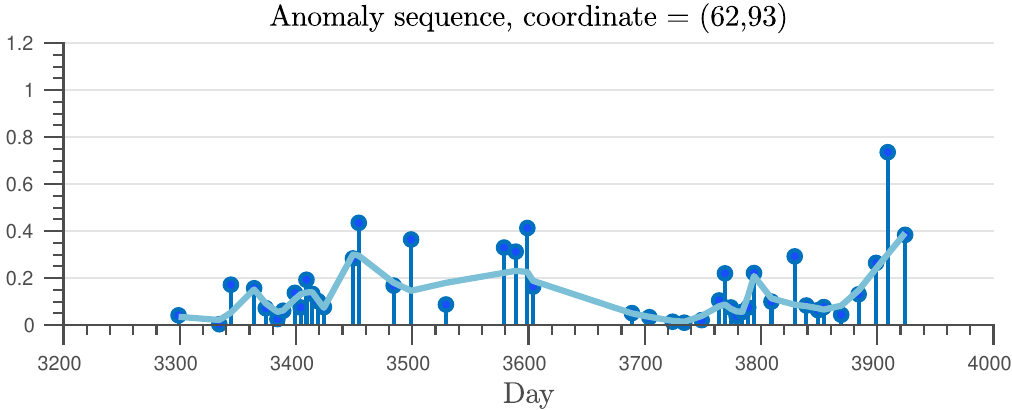}};
  \draw[color = red, line width = 1pt, fill=red] (2.90,0.4) circle (0.05cm);
  \draw[color = olive, line width = 1pt,fill=olive] (-0.72,-0.25) circle (0.05cm);
\end{tikzpicture}
\caption{Anomaly sequence for the red pixel
  in \textbf{Figure} \ref{PR:Fig2}. The projection operator $\bP^{m}$
  is applied spatially to each frame, with the anomaly quantified and
  plotted against time. A robust LOESS is performed on the sequence
  (blue line).  Forest loss is detected on day 3484, with
  the anomaly level increasing. The forest 
  recovers after the loss event as determined on day 3704. On day 3909 (red
  marker), a localized anomaly is caused by cloud screening (image
  for day 3909 in \textbf{Figure} \ref{PR:Fig2}).}
\label{PR:Fig3}
\end{figure}

\begin{rem}
It is important to note that although the anomaly sequence corresponds
to a single patch of land, the information contained in each anomaly
pixel has information of the surrounding land. This is due to the
covariance matrix in general being non-diagonal, containing
correlation terms among the pixels.
\end{rem}

\textbf{Finite state machine anomaly classification:}
From the temporal anomaly feature $\bneta(t_k)$ 
we can use these features for the detection of
evolving phenomena. For example, in \textbf{Figure}
\ref{PR:Fig3} on day 3909 we observe a sudden change in the anomaly sequence.
This implies that this is a spurious anomaly probably caused
by a cloud or a cloud shadow. Using information on the behavior of
the anomaly, we can classify the state of the land cover.

Let $\bgamma(t_{0}), \dots, \bgamma(t_f)$ be the underlying state of
the land cover of a single pixel at the discrete time sample $t_0,\dots,t_f$.  Using the observation features $\bneta(t_k)$ and possibly the data $\bu(t_k)$ for $t = t_0,\dots,t_f$ we can detect and classify the current state of the land cover using a Hidden Markov Model. We now show how the time-evolving features and can be used for detecting anomalies. Let $\Sigma:= \{ \gamma_1, \dots, \gamma_N \}$ be underlying state of the land cover e.g. \{ forest + no cloud, forest + cloud, bare ground + no cloud, bare ground + cloud\} and $P = \{p_{11}, \dots, p_{ij}, \dots, p_{NN}$ $\}$ a transition probability matrix. These are the probabilities that the land cover will change from one state to
another.

\jc{Given the anomaly sequence, we can form the visible state $\zeta(t_k) =
f(\bneta(t_k))$, where $f$ is the emission function. This function usually consists of a binary vector signal $\{0,1\}$ reflecting if $\bneta(t_k)$ are below or above a predefined threshold level.  Let $\pi(t_0)=\bbP(\gamma(t_0))$ be the initial probability distribution over the states. We have the Markov assumption that the probability depends only on the previous state: $\bbP(\gamma(t_k) \mid \gamma(t_{0}), \dots, \gamma(t_{k}))
= \bbP(\gamma(t_{k}) \mid \gamma(t_{k-1}))$, and the observations only
depend on the current state i.e. $\bbP(\zeta(t_{k}) \mid \gamma(t_{0}), \dots, \gamma(t_{f}), \zeta(t_{0}), \dots, \zeta(t_{f})) \bbP(\zeta(t_{k}) \mid \gamma(t_{k}))$. Now, given the observations, we want to estimate the most likely
sequence of the state of the forest}
\begin{equation}
\begin{split}
&(\bgamma^{*}(t_{0}), \dots, \bgamma^{*}(t_f)) 
= \argmax_{ \bgamma(t_{0}), \dots, \bgamma(t_f)}
\bbP(\zeta(t_{0}), \dots, \zeta(t_f) \\
&
\mid \bgamma(t_{0}), \dots, \bgamma(t_f) ) \\
&
=
\argmax_{ \bgamma(t_{0}), \dots, \bgamma(t_f)} \bbP(\bgamma(t_{0}), \dots, \bgamma(t_f) \mid \zeta(t_{0}), \dots, \zeta(t_f)) \\
&\bbP(\zeta(t_{0}), \dots, \zeta(t_f)).
\end{split}
\end{equation}











\begin{figure}[hbtp]
  \centering
\begin{tikzpicture}[xshift=-8cm,scale = 0.3,every node/.style={scale=0.3}]>=latex']
 \begin{scope}
    \path [
    mindmap,
    text = black,
    level 1 concept/.append style =
      {font=\Huge\bfseries, sibling angle=90, level distance=5.5cm, minimum size=4cm},
    level 2 concept/.append style =
      {font=\Huge\bfseries, sibling angle =90, level distance=5.5cm, minimum size=4cm}, 
    level 3 concept/.append style = 
    {font=\Huge\bfseries, sibling angle = 90,level distance=5.5cm, minimum size=4cm},
    level 4 concept/.append style = 
    {font=\Huge\bfseries, sibling angle = 90,level distance=5.5cm, minimum size=4cm},
    tex/.style     = {concept, ball color=orange!50!white,
      font=\Huge\bfseries},
    editors/.style = {concept, ball color=orange!50!white},
    systems/.style = {concept, ball color=babyblue!100!white}
    ]
        node [tex, minimum size=3cm, concept color=lightgray] at (0,0) {$\bgamma(t_{0})$} [clockwise from=90]
    child[concept color=lightgray, nodes={systems}]{
      node [text width=3cm,text centered]{$\zeta(t_{0})$}}
    child[concept color=lightgray, nodes={editors}]{
      node [text width=3cm,text centered]{$\bgamma(t_{1})$} 
           child[concept color=lightgray, nodes={systems}]{
           node [text width=3cm,text centered]{$\zeta(t_{1})$}}
           child[concept color=lightgray, nodes={editors}]{
             node [text width=3cm,text centered]{$\dots$} [clockwise from=90]
             child[concept color=lightgray, nodes={systems}]{
               node [text width=3cm,text centered]{$\dots$}}
               child[concept color=lightgray, nodes={editors}]{
             node [text width=3cm,text centered]{$\bgamma(t_k)$} [clockwise from=90]
             child[concept color=lightgray, nodes={systems}]{
               node [text width=3cm,text centered]{$\zeta(t_k)$}}
             }
           }
      };
    \end{scope}
 \end{tikzpicture}
 
  \begin{tikzpicture}
  \node at (0,0) {\includegraphics[scale = 1, trim = 0cm 0cm 5.1cm 0cm,
      clip]{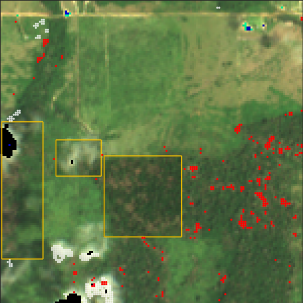}};

  \node at (4,0.5) {\includegraphics[scale = 1, trim = 5.5cm 1.23cm 2.2cm 0.18cm,
       clip]{./Images/plot_version7_knn_compressed.pdf}};

  \node at (4.85,-0.3) {\includegraphics[scale = 1, trim = 8.2cm 1.85cm 0cm 1.2cm,
      clip]{./Images/plot_version7_knn_compressed.pdf}};

  \node at (4.3,-2) {\includegraphics[scale = 1, trim = 6.5cm 0cm 1cm 4.25cm,
      clip]{./Images/plot_version7_knn_compressed.pdf}};
 \end{tikzpicture}

\caption{
  The state of the forest is tracked with optical observations.  This is achieved by
  applying a Hidden Markov Model to the $i^{th}$ pixel with the
  observation sequence $\zeta(t_k)=\bneta(i,t_k)$. The state of the
  land cover for the $i^{th}$ pixel is \{forest, cloud or shadow,
  loss\}.  A Hidden Markov Model and the Viterbi algorithm are
  used to classify the state of the forest $\gamma(t_k)$ at time
  $t_k$. The red pixels classify trees that have been cleared. The
  white pixels correspond to clouds or dark shadows. The black points
  correspond to a well-known cloud masking algorithm, which could not
  detect the light clouds and dark shadows. These are particularly
  difficult to detect.}  
  \label{Fig:Evolving} 
  \end{figure}
  
This optimization is, however, too expensive, since we would need to consider
all the possible state trajectories. In practice we use the Viterbi
algorithm to reduce the computational complexity. Let
$\bgamma^{\#}(t_{0}), \dots, \bgamma^{\#}(t_f)$ be likely sequence
given by the Viterbi algorithm. We can now classify the 
anomaly of the sequence. Given that we assume that the initial state
is a forest we are looking for persistent anomalies (bare ground + no
cloud). Thus from the sequence
$\bgamma^{\#}(t_{0}), \dots, \bgamma^{\#}(t_f)$ we are looking for
subsequences of bare ground + no cloud that are persistent. A
persistent parameter is defined in the code: Frames To Classify (FTC). The \po{forest loss}
(bare ground if the pixel started as forest)
is classified as positive at the location of the first
subsequence of length FTC. We classify this as \po{forest loss} at the end of the first
subsequence.

In \textbf{Figure} \ref{Fig:Evolving} we show how the time-evolving
 features can be used for classification of the land cover with
 $\Sigma$ := \{forest, cloud or shadow, loss\} in the Amazon
 forest using the HMM on each pixel separately. Suppose we only have
 observational data $\bu(t_0),\dots,\bu(t_f)$  consisting of EVI
 optical measurements from Sentinel 2. From the training data
 $\bv(\tau_0), \dots, \bv(\tau_s)$ we construct the projection
 operator $\bP^{m}$ and construct the anomaly sequence
 $\bneta(t_0),\dots,\bneta(t_f)$.

Now, let the emission function $f$ be a function such
 that output is 1 if the value of $\bneta(t)_k$ is greater than a
 threshold value.  Using the HMM and the Viterbi
 algorithm we obtain the likely sequence
 $\bgamma^{\#}(t_{0}), \dots, \bgamma^{\#}(t_f)$ and classify it.  The
 red pixels are classified as trees that have been cleared. The white pixels
 correspond to clouds or shadow. The black points correspond to results from an existing well-known cloud masking algorithm \cite{Skakun2022}. 
 Notice that the
 cloud masking algorithm was not able to detect light clouds or
 shadows. Our approach detected \po{forest loss} and simultaneously
 distinguished it from clouds and shadows.  

  \cm{The above-described method requires choosing application specific parameters, namely the transition and emission probabilities as well as the thresholds for mapping the data to binary vectors.  For the transition probabilities we first use a cloud detection algorithm to get the approximate percent of pixels covered by clouds.  Since observations are days apart we assume that cloud occurrences are independent in time, which means that, for example, a 5 percent average cloud coverage for the entire area gives us a roughly 5 percent probability of transitioning to a cloudy state from any state.  Combining this with the fact that transitions from forest to deforestation or vice versa are rare, happening 0-1 times for almost all pixels, we have that the probabilities of forest to forest, cloudy forest to forest, bare ground to bare ground, and cloudy bare ground to bare ground depend on the average cloud cover and based on our data should be close to 1.  For example, the probability of forest to forest is approximately 1 minus the average cloud cover minus the probability of forest to bare ground, with the latter a very small value that must be estimated.  After that is selected, forest to cloudy forest vs forest to cloudy bare ground can be split up based on average forest cover for the region using a forest mask.  Transition probabilities from the other states are similar.  Emission probabilities are harder to estimate and are thus calibrated by hand using the small region shown in \textbf{Figure} \ref{results:zoom}, which is about 0.3 percent of the entire region shown in \textbf{Figure} \ref{results:hybridmap}. Once these probabilities are selected, the HMM is run on this small region using all combinations of reasonable values for the threshold(s) and FTC values, and the datemaps are compared by hand to pick the best parameters.}

\subsection{Hybrid optical and SAR fusion land cover tracking}

\jc{The HMM model is now applied to the optical data from Sentinel-2 (EVI) and the SAR data from Sentinel-1. The sequence $\bu(t_0),\bu(t_1), \dots$ now consists of optical and radar data.} 
Three separate scenarios are tested: a) Optical-only using the anomaly sequence, b) SAR-only, and c) Hybrid method with optical (anomaly sequence) and  SAR data. The results will show that the hybrid approach is significantly better than the single-sensor approaches.

We test the tracking algorithm for detecting deforestation from March 26, 2020 to December 31, 2022, with data 
from both sensors. 
 This data will be split into two groups:

\begin{itemize}          
\item The training data $\bv(\tau_0), \dots, \bv(\tau_s)$ will consist of 71 Sentinel-2 EVI measurements from Sentinel-2 between  December 17, 2018 and March 21, 2020. These
measurements are used to construct the projection matrix $\bP^{m}$ and are then applied to the optical EVI sequence $\bu(t^o_0),\dots,\bu(t^o_f)$, and the anomaly sequence
$\bneta(t^o_0),\dots,\bneta(t^o_f)$ is obtained. The test data $\bu(t^o_0),$ $\dots,\bu(t^o_f)$ consist of 161 time samples between March 26, 2020 and December 26, 2022. Note that there is a changed
notation from $t_k$ to $t^o_k$ to indicate that this time sample
consists only of optical data.
\item The second group consists only of Sentinel-1 SAR measurements
$\bu(t^r_0),\dots,\bu(t^r_g)$. The full set of SAR observations consists of 234 samples 
between January 4, 2017 and December 28, 2022.
\cm{However, since Sentinel-1 was launched earlier than Sentinel-2 we will always have SAR data for the time span corresponding to the optical training data as well as the time span before that going back to the start of Sentinel-1 observation.  If we assume that the latter set of SAR data has a nontrivial impact on performance, possibly positive or negative, then the results from including this set of data would not account for how performance may differ for earlier/later validation data sets with different amounts of pre optical SAR data.  Because of this, we start the use of SAR data on December 25, 2018, the first day after the optical data is available, and assume others truncate their SAR data likewise.  After this truncation the second group contains 178 SAR observations.}
Since these measurements are noisy, 
a spatio-temporal Bayesian filter is applied. For simplicity we will refer to $\bu(t^r_0),\dots,\bu(t^r_g)$ as the
filtered data from the Bayesian method \cm{truncated with the start date of December 28, 2018}. Data measurements that are numerically low
indicate presence of bare ground (possibly with small amounts of grass). If the measurements are high this indicates
backscattering, and a structure such as a tree or human
construction is located at that pixel.
\end{itemize}
Given the optical anomaly sequence and the radar measurement data
$\bu(t^r_k)$ we can form the optical-radar state $\zeta(t_k) =
f(\bneta(t^o_k), \bu(t^r_k))$, where $f$ is the emission function. 

    In \textbf{Figure} \ref{results:zoom} the tracking of \po{forest loss} is shown for all three methods.  Due to clouds, such tracking is difficult. \jct{However, the hybrid method, which combines optical and SAR data, captures many forest-loss events.} 
The hybrid approach, which combines optical and SAR data, proved to be effective. In the supplement video, we demonstrate time-evolution tracking of the forest from SAR and optical satellite data as trees are removed. However, these results are for a small area ($5120 \, m \times 5120 \, m$, $512 \times 512$ pixels).  Notice that there is a delay of about 10 frames before the detection is confirmed. In the results section, we perform validation tests for \po{the study} area and compare the optical-only, SAR, and hybrid methods. There are significant advantages in using the hybrid method.





\section{Experiment and Discussion}
\subsection{Study Area }
\hg{To evaluate the performance of \po{deFOREST}, the optical, SAR and fusion algorithms were applied to detect \po{forest loss} in the Amazon rainforest. The study area is $92 \,km$ by $92 \,km$ over the Jacund\'{a} National Forest in Brazil, at the southern boundary of the Amazon forest (\textbf{Figure} \ref{Fig:Amazons}). Previous studies have shown that the humid tropics, such as in West Africa and Southeast Asia, have very few optical satellite observations because these areas have persistent cloud cover and shadows \cite{zhang2022global}. In addition to heavy precipitation, satellite passes in the tropics do not overlap, whereas most high latitudes are in the overlapping zone of multiple orbits. This study area is selected for two reasons: 1) it represents the climate and land cover of the humid tropics where the availability of optical remote sensing data is limited, and 2) as it is located at the southern edge of the humid tropics, it has more clear Sentinel-2 and Landsat observations than the area closer to the equator. While the fusion algorithm works \po{in regions with low data availability} like Gabon or Indonesia, the limited amount of clear optical remote sensing images in \po{such regions} makes it very difficult to collect reference data and validate the performance of the algorithm. \po{Our study area, on the other hand, has sufficient  Sentinel-2 data} to validate the result. In the following analysis, optical images are randomly removed from the test dataset to simulate  regions with fewer available clear observations. 
}

\begin{figure*}
    \centering
    \includegraphics[width=0.75\linewidth]{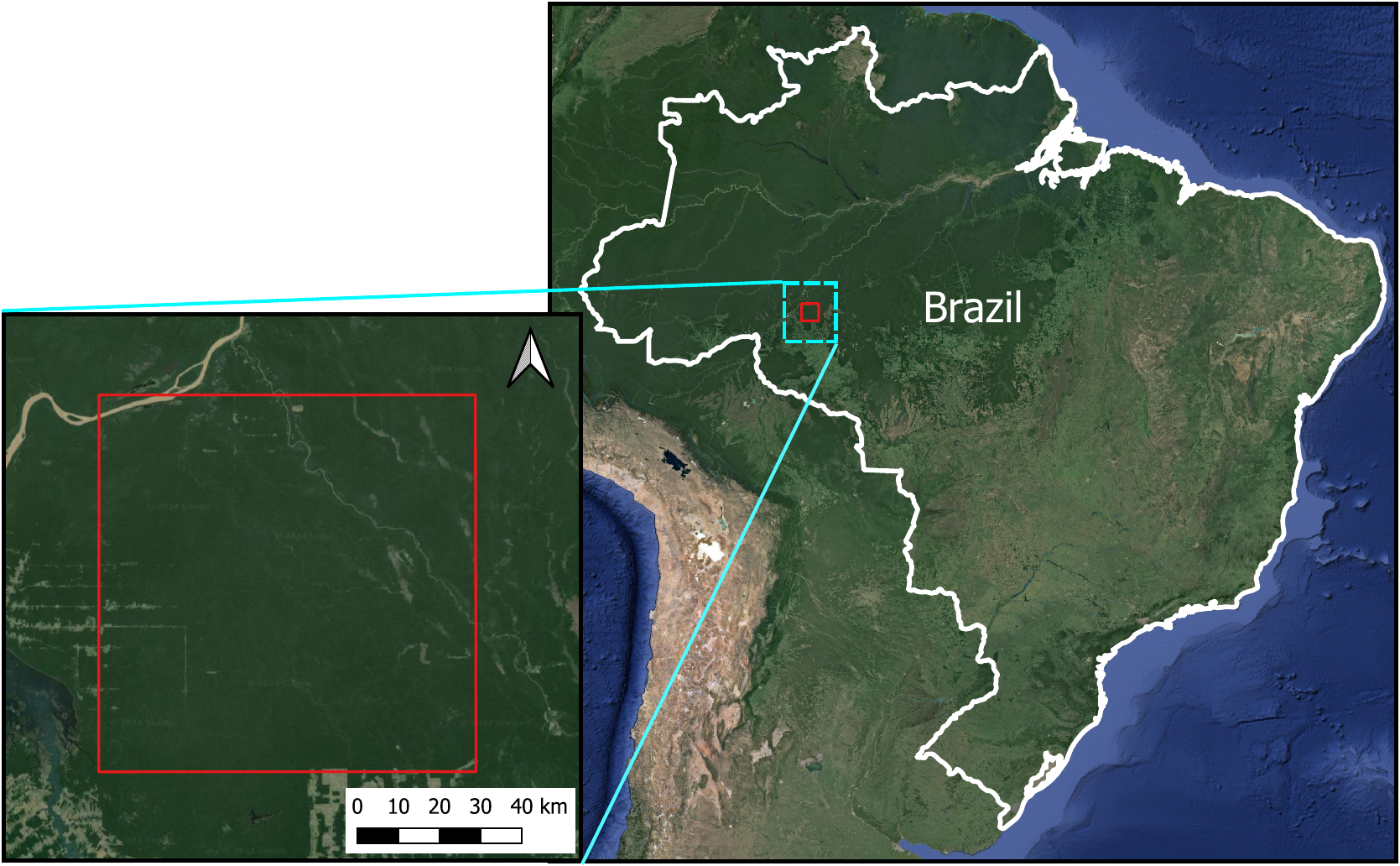}
    \caption{Amazonian forest in Brazil test area.}
    \label{Fig:Amazons}
\end{figure*}

\subsection{Implementation}
\hg{
As explained in the method section, the EVI time series computed from the Sentinel-2 surface reflectance data and the VV and VH time series of the Sentinel-1 SAR observations are used to detect \po{forest loss} in the study area between March 26, 2020 and December 31, 2022. The Sentinel-2 QA band and the Sentinel-2 cloud probability layer created by LightGBM are applied to prescreen the clouds and shadows from the Sentinel-2 data \cite{Skakun2022}. Radiometric slope correction and lee-sigma speckle filtering were applied to preprocess Sentinel-1 images \cite{Lee2009} \cite{Vollrath2020}. After preprocessing, the Sentinel-2 only algorithm, the Sentinel-1 only algorithm, and the hybrid algorithm using both data streams are applied to generate three separate maps of deforestation in the study area using all the available observations. } For the hybrid method the FTC parameter is set to 10, the optical anomaly threshold to 0.9 and the SAR threshold to -5.5, for the optical-only method the FTC is set to 6 and the optical anomaly threshold to 0.9, and for the SAR method the FTC is set to 5 and the SAR threshold to -5.5.

\hg{
To simulate the regions with fewer available optical observations, we ran the Sentinel-2 and hybrid algorithms with images randomly removed from the monitoring period. 
}\cm{The training data was left the same, meaning that the quantity of the optical data is changing, but not the quality of the anomaly data.  This makes the results favorable for optical only, which is more sensitive to the quality of the anomaly data.}
 
\cm{Results from this can be seen in \textbf{Figure} \ref{results:accuracy} and \ref{results:producers}, as well as in the Supplemental section in \textbf{Figure} \ref{results:BA} and \ref{results:F1}.  Along the x-axis we have the number of optical days included, from 1 to the 161 days left after 71 were used for training.  For each number of optical days, 100 sets of optical days of that length were randomly selected, and the hybrid and optical-only algorithms were run using those same sets to make the results comparable.  Additionally, the SAR-only algorithm was run once to give the horizontal dashed line.  Across all numbers of optical days the hybrid and optical-only algorithms used the previously described full set of SAR data starting on December 28, 2018, since cloud cover or other constraints on Sentinel-2 data should not affect Sentinel-1 data.}

\cm{The amount of available data does not affect the choices of SAR and optical thresholds; however, it does affect the FTC, since a fixed FTC represents a longer time span to confirm a detection as the amount of available data decreases.  Because of this, a variable FTC is needed.  For the hybrid method we linearly interpolate between the manually selected SAR FTC of 5 and full data hybrid FTC of 10 using $\lceil10(\frac{n}{161})+4(1-\frac{n}{161})\rceil$ for $n$ optical days.  For the optical-only algorithm, the initial assumption was that the FTC should be directly proportional to the number of optical days.  Based on the results in \textbf{Figure} \ref{SS:ftc} this appears to be reasonable, so we used an FTC of $\lceil 6(\frac{n}{161})\rceil$. 
The performance difference between variable and fixed FTC can be seen in \ref{SS:fixed}.  Since the FTC is an integer we cannot adjust it continuously, which is the reason for jumps in the various metrics.}

\hg{
The accuracy assessment of the Sentinel-2 only, Sentinel-1 only, and the hybrid algorithms followed the Good Practices introduced by Olofsson et al. \cite{OLOFSSON2014}.    \po{Because the mapped area of forest loss is less than 10 percent of the study area, we used a stratified random sampling approach to ensure sufficient sample representation in areas of forest loss. A total of 1000 sample units were selected; 700 units were allocated in areas mapped as stable in all three change maps; 130 units were allocated in the areas mapped as forest loss by all three maps; and 100 units were allocated in areas of forest loss mapped by either the optical-only or radar-only. Finally, 70 units were allocated in the stable area of the hybrid map that shows up as deforestation in either the optical-only or radar-only map. Eight trained researchers interpreted the land surface at sample locations using Sentinel-2 images, Landsat time series, and high-resolution images in Google Earth. The sample units labeled by one researcher were verified by another researcher to ensure the quality of the validation data. 
The overall accuracy, and user’s and producer’s metrics of the maps were estimated from the sample data.}  These measures are shown in detail in Section \ref{deForest:Results}. In addition, balanced accuracy, F1 score, and user's and producer's stable metrics can be found in Section \ref{Supp:additional} of the supplement material.
}

\begin{rem}
 \jct{deFOREST is implemented in MATLAB \cite{Matlab2025}. 
 The implementation is extensive; a public version can be downloaded from GitHub  \cite{GitHub2025} (https://github.com/jcandas/deFOREST).}   
\end{rem}

\subsection{Results}
\label{deForest:Results}
\hg{
The result of each algorithm is a map of deforestation in the study area. The  map generated by the hybrid algorithm is shown in \textbf{Figure} \ref{results:hybridmap}. A zoomed-in comparison of the SAR-only, optical-only and hybrid results is shown in \textbf{Figure} \ref{results:zoom}. The SAR-only map captures most of the \po{loss events without many false positives (commission errors). However, it omits some forest loss sites, such as the logging trail on the left of the figure. In addition, the SAR-only algorithm omits parts of the loss event at the top right corner. The optical-only map  captures all the loss events in this subregion but also introduces scattered false positive detections (commission errors) }in the middle of the forest. The hybrid algorithm  captures the logging trail on the left, and it looks cleaner than the SAR-only map without many of the scattered false positives. \hgt{We compared our result with two established forest alert datasets, GLAD-S2 (optical only) and RADD (SAR-only) alert, in \textbf{Figure} \ref{results:zoom} as well \cite{GLAD_2016}\cite{RADD_2021}. GLAD-S2 (\textbf{Figure} \ref{results:zoom}(f)) is similar to the optical-only and hybrid deFOREST result, but it misses the forest loss over the logging trail in the middle-left of the area while detecting the subtle disturbances around it. GLAD-S2 also does not fully cover some of the forest-loss in the top-right and bottom-right sections. RADD alert (\textbf{Figure} \ref{results:zoom}(g)) is close to our SAR-only result (\textbf{Figure} \ref{results:zoom}(c)). Both of them missed some areas of forest loss, including the disturbances at the left side of the area and at the top-right corner. }
}

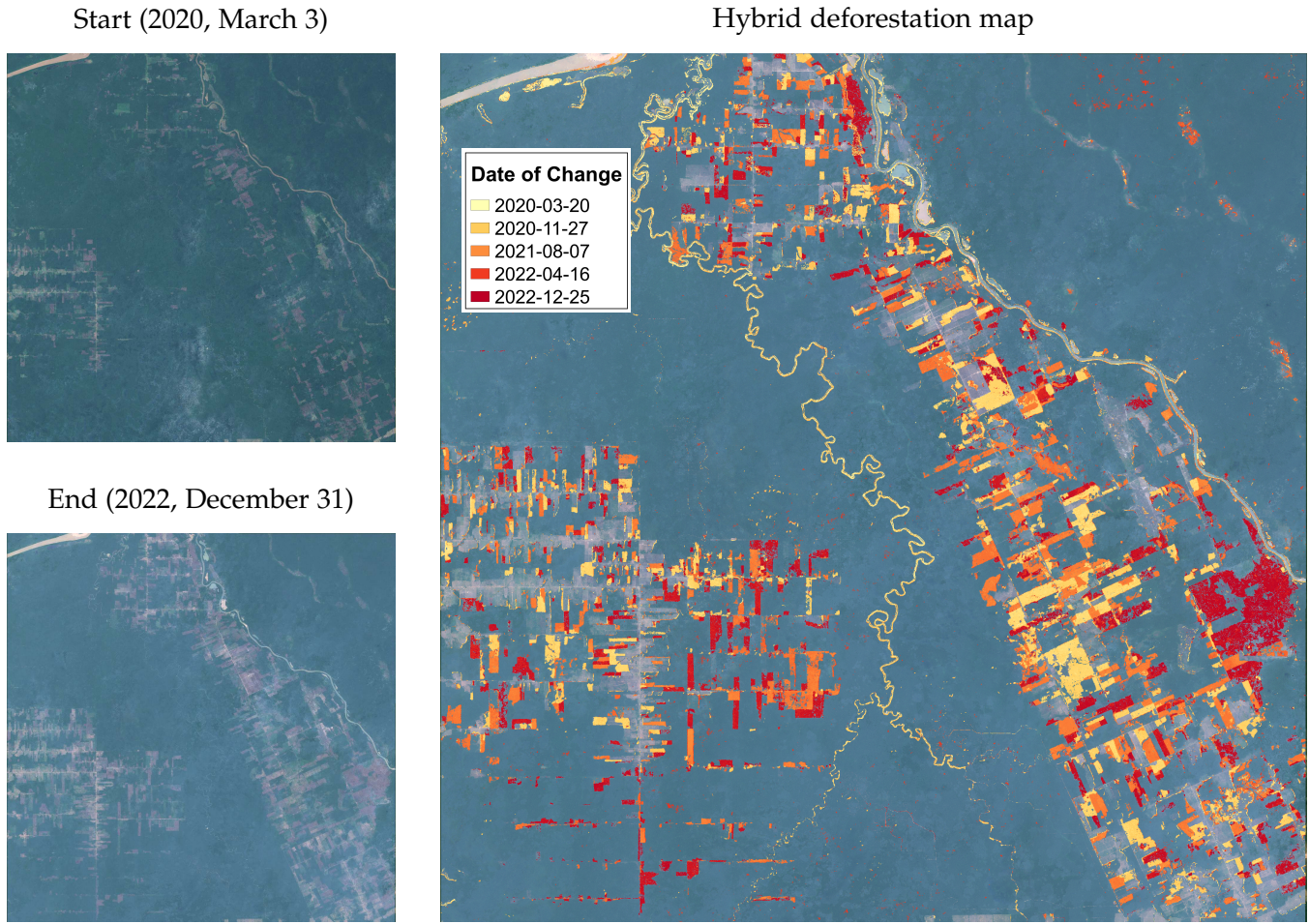
\begin{figure*}[htpb]
    \centering
    \begin{tikzpicture}[scale = 1.2, every node/.style={scale=1.2}]



        \node at (-2.7,0) {\includegraphics[scale = 0.20]{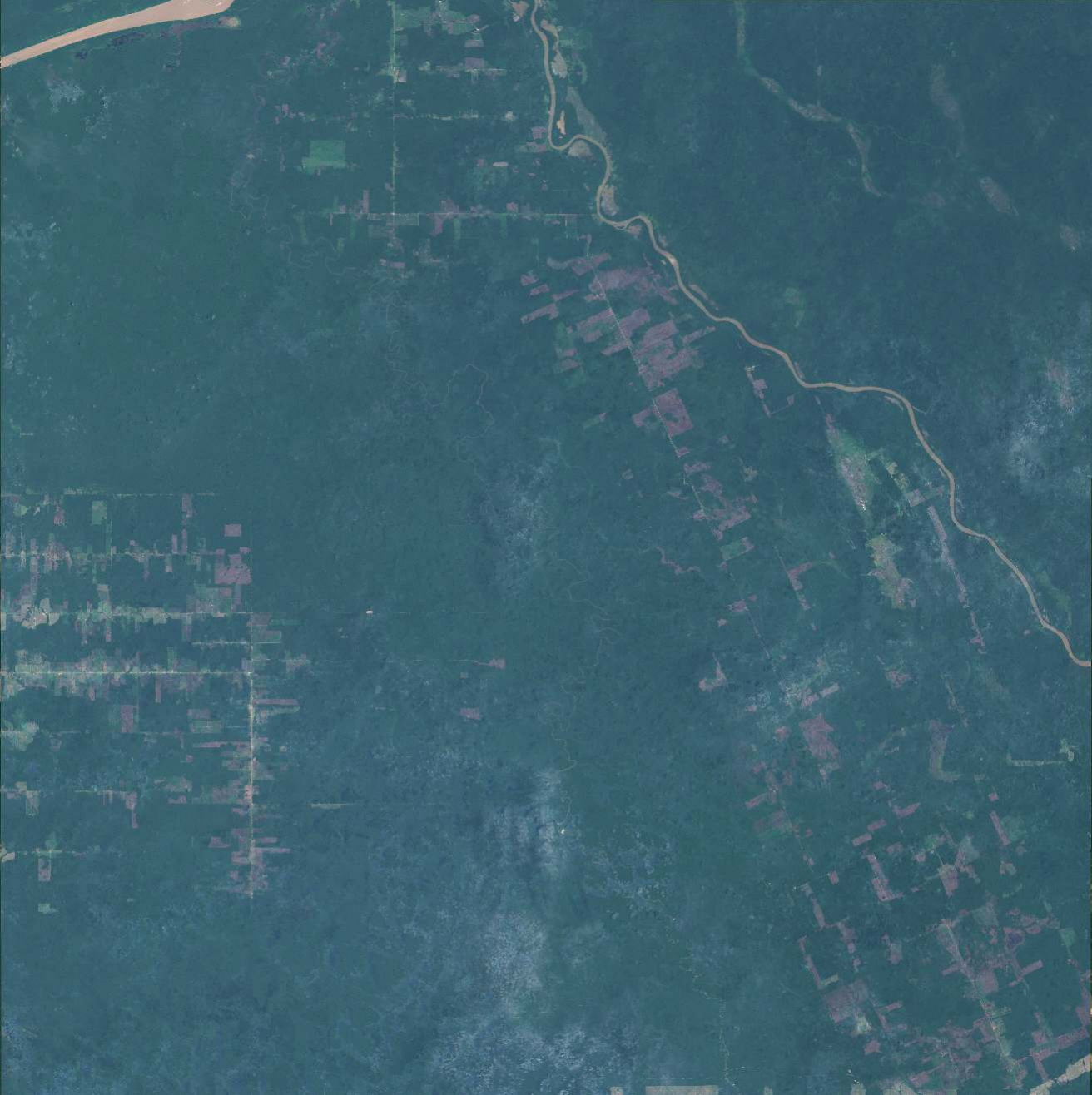}};

        \node at (-2.7,-5.5) {\includegraphics[scale = 0.20]{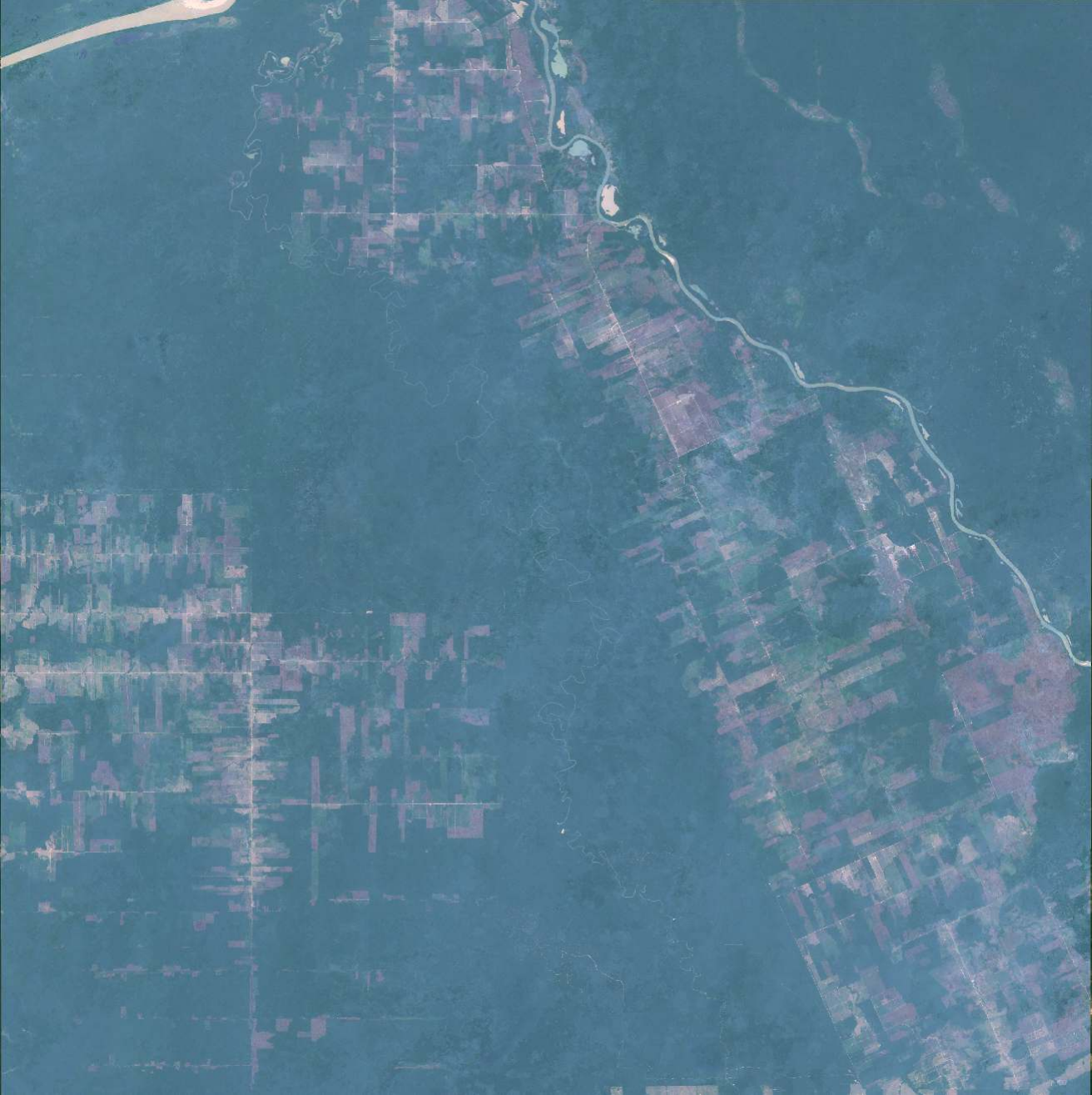}};

        \node at (5,-2.75) {\includegraphics[scale = 0.583]{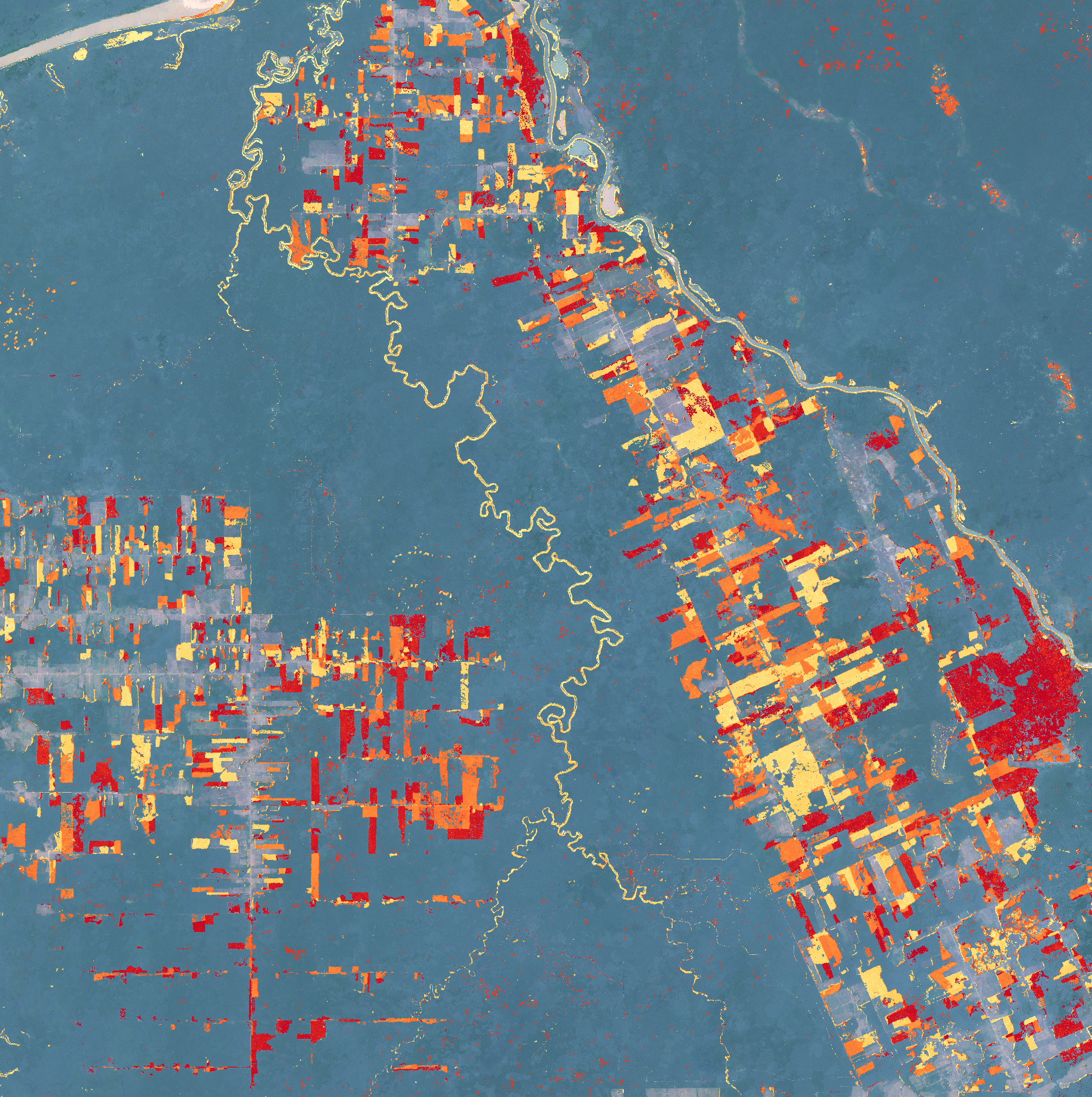}};


        \node at (-2.7,2.6) {\small Start (2020, March 3)};
        \node at (-2.7,-2.9) {\small End (2022, December 31)};
        \node at (5,2.6) {\small Hybrid deforestation map};

        \node at (1.25,0.2) {\includegraphics[scale = 0.3]{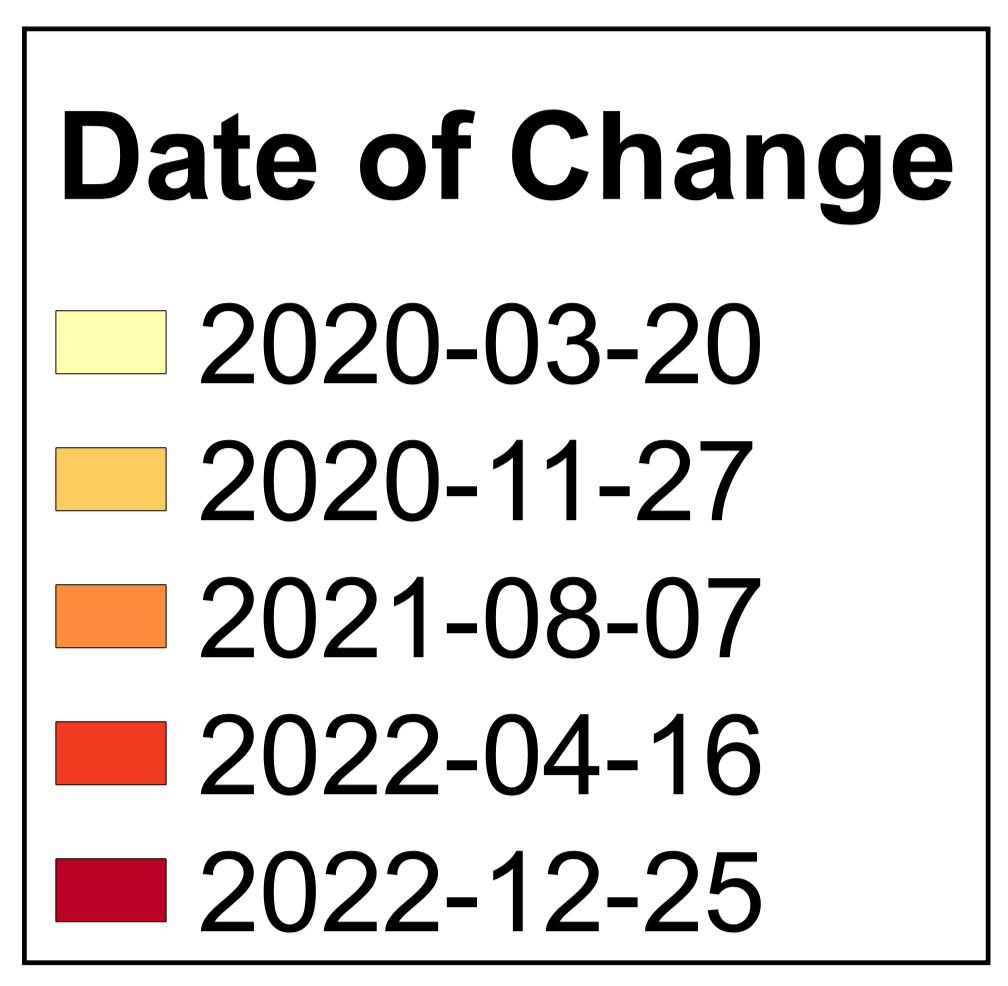}};

    \end{tikzpicture}  
    \caption{a) The median Sentinel-2 image composite for the beginning of the study period. b) The median Sentinel-2 image composite for the end of the study period. c) The map of deforestation from the hybrid of SAR and optical data of the entire study area. Note that the legend covers part of the area. However, there is no \po{forest loss} detection in that area. See \textbf{Figure} \ref{results:pd} in the supplement.}
    \label{results:hybridmap}
\end{figure*}

\begin{figure*}[t]
    \centering
\begin{tabular}{c c}
(a) & (b) \\
\includegraphics[scale = 0.29]{Images/Hybrid/Figure8-1clear.png}     
& 
\includegraphics[scale = 0.29]{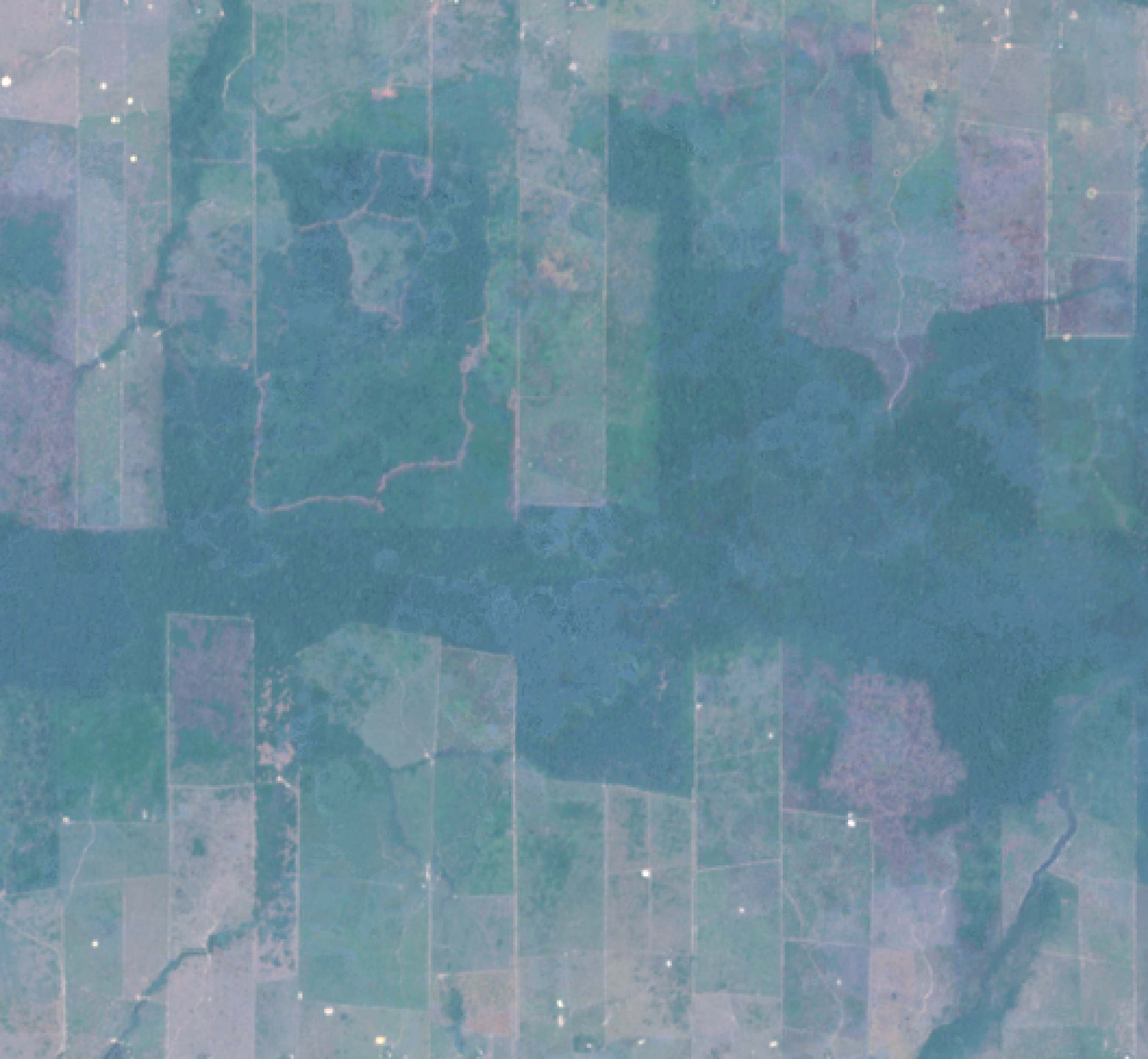} \end{tabular}
\bigskip
\begin{tabular}{c c c}
(c) & (d) & (e) \\
\includegraphics[scale = 0.26]{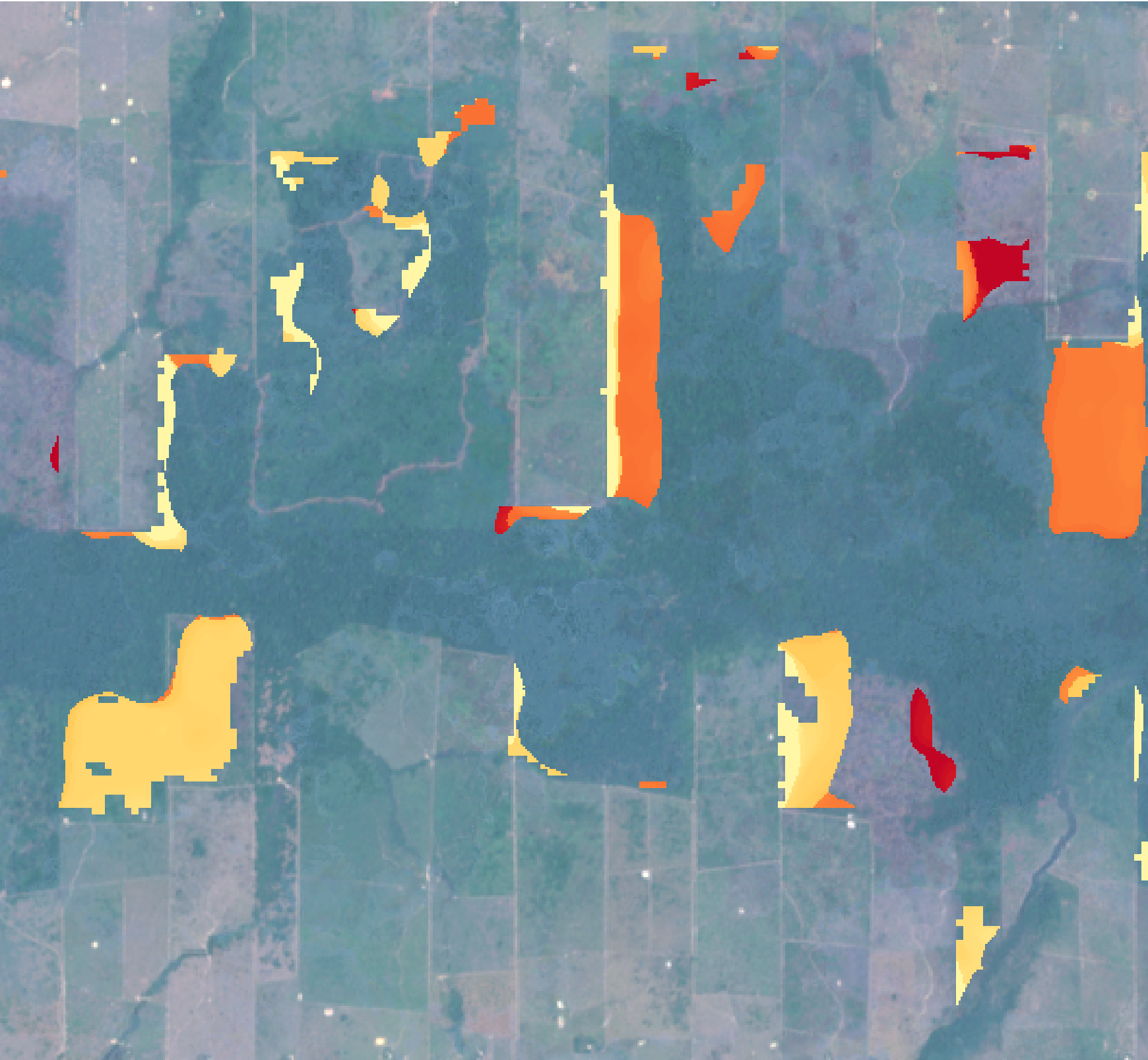}     &
\includegraphics[scale = 0.26]{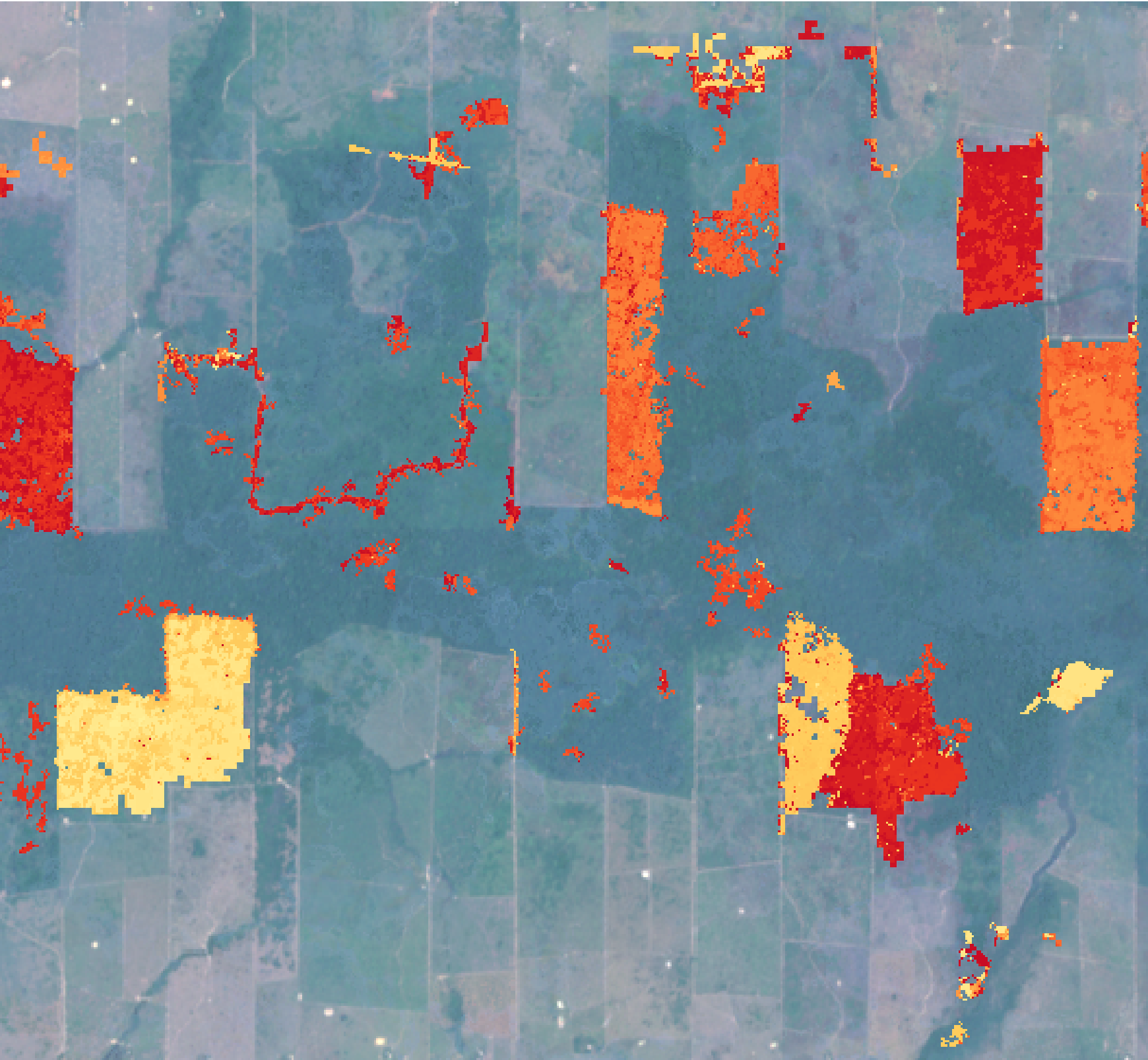}   &
\includegraphics[scale = 0.26]{Images/Hybrid-New/Figure8e.png}    
\end{tabular}
\bigskip
\begin{tabular}{c c}
(f) & (g) \\
\includegraphics[scale = 0.349]{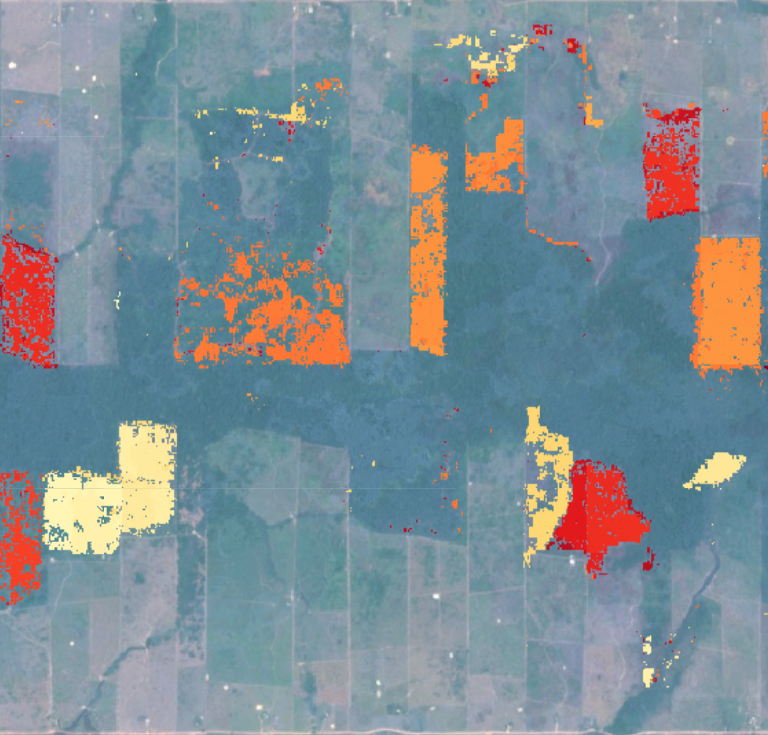}     
& 
\includegraphics[scale = 0.349]{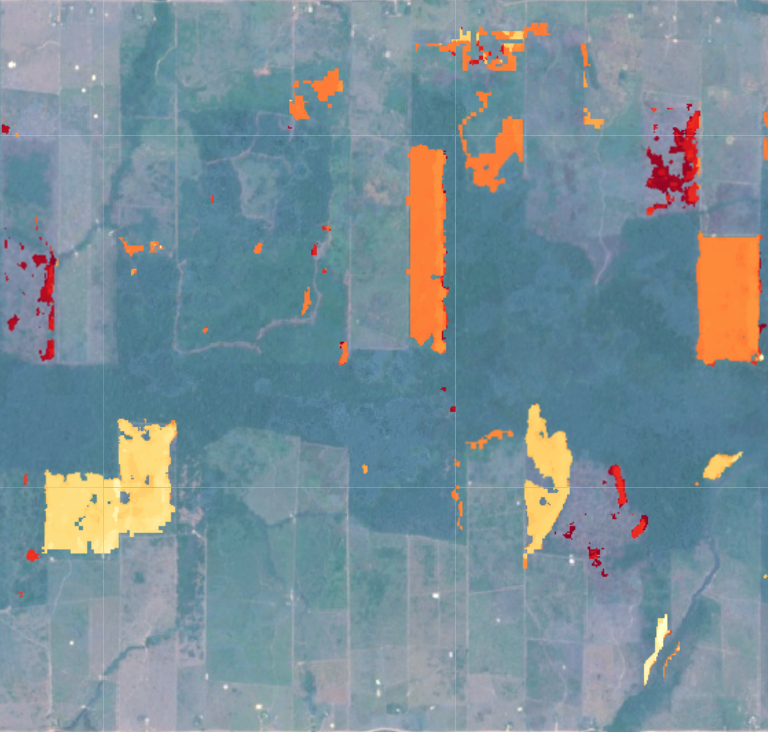}    
\end{tabular}
    \caption{Comparison of the map of \po{forest loss} from SAR-only, optical-only, the hybrid of SAR and optical data \hgt{of deFOREST, GLAD-S2 optical-only, and RADD radar-only }in a subset of the study area. a) The median image composite of the Sentinel-2 data at the beginning of the study period. b) The median image composite of the Sentinel-2 data at the end of the study period. c) The loss map from deFOREST Sentinel-1 (SAR) only. d) The loss map from deFOREST Sentinel-2 (optical) only. e) The loss map from the hybrid deFOREST algorithm. \hgt{f) The loss map from high confidence GLAD-S2 alert g) The loss map from RADD alert.}}
    \label{results:zoom}
\end{figure*}

\hg{
The overall accuracy and the user’s and producer’s accuracies of each map are presented in \textbf{Table} \ref{results:table}. The results of the optical-only and the hybrid algorithm have higher overall accuracy than the SAR-only result since the producer’s accuracy of the radar-only result is lower than the other two results. In other words, the SAR \po{map omits more forest loss} than the other two outputs. The overall accuracy of the optical-only result is almost as high as the hybrid result. However, the producer’s accuracy of the optical-only algorithm is slightly lower than the hybrid algorithm, which means the optical-only result omits more forest loss.
}

\hg{
\po{In an attempt to simulate conditions in data-scarce cloudy regions, accuracies were estimated for the maps constructed using the }optical-only and hybrid algorithms after removing some of the optical images\po{; the accuracies} are presented in \textbf{Figure} \ref{results:accuracy} and \textbf{Figure} \ref{results:producers}. \po{As expected,} the overall accuracies of both maps decrease with more optical images being removed from the dataset\po{, especially the} optical-only map. As the number of available optical images approaches zero, the accuracy of the hybrid result is increasingly closer to the SAR-only result. The producer’s accuracies of both hybrid and optical-only results decrease \po{at the same rate} until fewer than 50 optical images are available within the monitoring period. The user’s accuracy of the optical-only map decreases steadily with fewer optical images, while the user’s accuracy of the hybrid algorithm remains stable.}  

\cm{We tested the performance of deFOREST with fewer training data. 
\textbf{Table} \ref{results:table} shows that reducing  the number of training days from 71 to 35 by leaving out every other time slice does not significantly drop the overall accuracies of the optical-only and hybrid results. This shows that deFOREST can perform well with fewer cloud-free observations in the training period, so the algorithm can monitor forest loss in the area with persistent cloud cover.}

\jc{We compared our results with the recently developed FNRT algorithm. In \textbf{Table} \ref{results:table} it is shown that for the same training period of 71 days\po{, deFOREST yields substantially higher accuracies}. The training period must be increased to 130 days to obtain comparable results. However, this requires significantly more data and can present a problem in highly cloudy regions. In contrast, the training period for the deFOREST method can be reduced to 35 days and comparable results are obtained. \hgt{We also tested the performance of GLAD-S2 and RADD alerts in this area. The overall accuracy of GLAD-S2 is slightly lower than FNRT and deFOREST. However, GLAD-S2 has a higher user's accuracy, which implies fewer false positive detections, and a much lower producer's accuracy caused by the higher omission error. The RADD alert has a slightly better performance compared to our SAR-only result while having lower overall accuracy than the optical-only and hybrid deFOREST results. The advantage of our SAR-only algorithm is that our HMM-based algorithm does not require training data. In contrast, the RADD alert requires two years of Sentinel-1 imagery for the training period.}}

\cmt{The second plot in \textbf{Figure} \ref{results:accuracy} shows the variance in the overall accuracy for Hybrid and Optical only, given different numbers of optical days used.  This demonstrates the stabilizing effect of the radar data in the Hybrid method.  When the number of optical days is close to 161, all sets of optical days have a high degree of overlap, since we only have access to 161 days' worth of data, so the inputs and thus the performance vary little for both methods.  As fewer optical days are used there is less overlap in the sets of optical data, so the variance for Optical only increases.  While the variance for the Hybrid method also starts to increase, it increases more slowly since the inclusion of the radar data makes it less dependent on the optical data.  Additionally, because the radar data are fixed, the optical data become a progressively smaller portion of the joint input for the Hybrid method.  Eventually the dilution of the optical data by the radar data overcomes the additional variance from the optical data, causing the Hybrid variance to peak and then decrease.}

\cmt{To estimate the impact of  the KL expansion, we ran the hybrid and optical-only algorithms using the unprocessed EVI data in place of the optical anomaly data; the results can be seen in Section \ref{unprocessed}
of the supplemental material.  According to our metrics (see \textbf{Table} \ref{results:utable}) the results appear better than when using the anomaly data. However, a visual inspection of the results reveals large regions of false positives corresponding to bodies of water (see \textbf{Figure} \ref{results:hybrid_comp}).  These regions are not well represented in our reference sample data (see \textbf{Figure} \ref{results:valp}). Since the reference data were sampled based on the anomaly-data results, which produced far fewer detections in these regions.
Because of this, the accuracy metrics may not properly reflect the performance when using the unprocessed optical data.}

\begin{table*}
    \caption{The overall accuracy, user's and producer's accuracy of forest loss, and computation time of each result. Note that the timings for FNRT ($*$) are 368 Google Earth Engine EECU-hours, which corresponds $\approx$ 3 or 4 wall-clock hours. \hgt{As both GLAD-S2 and RADD alert are established datasets, we do not know the computational time for these two datasets ($\ddagger$).} \hgt{The RADD alert uses all Sentinel-1 observations in 2017 and 2018 as training data, which corresponds to 60 training days in our study region ($\dagger$).} deFOREST with optical and radar performs best and is robust to decreases in anomaly data quality, as can be seen from the results using 35 instead of 71 training days. FNRT requires 130 training days to get comparable results to deFOREST using 35, and performs poorly using 71 training days. \hgt{The algorithm of GLAD-S2 requires two years of Sentinel-2 data as training, which correspond to 82 training days in our study area ($\dagger$).} This suggests that deFOREST is well suited for regions where optical data are sparse.}
    \label{results:table}
        \begin{tabular}{>{\centering\arraybackslash}p{4.5cm}
    c c c c c}
    \toprule
\rowcolor{olive!40}  Algorithm (Data) & Training Days & Overall Acc. & User Acc.  & Producer Acc. & Computational Time (h) \\
\midrule
\rowcolor{blue!20}  FNRT (Optical + Radar) & 71 & 0.260 & 0.260 &  1.000 &  $368^{*}$ \\
FNRT (Optical + Radar) & 130 & 0.935 & 0.892 &  0.707 & $368^{*}$ \\
\rowcolor{blue!20}  GLAD-S2 (Optical) & $82^{\dagger}$ & 0.916 & 0.938 &  0.629 &  NA$^{\ddagger}$ \\
RADD (Radar) & $60^{\dagger}$ & 0.881 & 0.972 &  0.538 & NA$^{\ddagger}$ \\
\rowcolor{blue!20} HMM  (Radar)              & NA   & 0.875 & 0.867 & 0.528 & 8.86\\
deFOREST (Optical) & 35     & 0.949 & 0.827 & 0.805 & 5.35\\
\rowcolor{blue!20} deFOREST (Optical) & 71  & 0.949 & 0.821 & 0.810 &  6.11 \\
deFOREST (Optical + Radar) & 35     & 0.954 & 0.813 & 0.844 & 12.84\\
\rowcolor{blue!20}  deFOREST (Optical + Radar) & 71     & 0.958 & 0.818 & 0.868 & 13.06\\
    \end{tabular}
\end{table*}

\subsection{Discussion}

\hg{
The lower user’s accuracy of the radar-only \po{map shows that using radar data alone tends to underestimate the area of forest loss, which is likely due to the presence of dense herbaceous understory vegetation with radar backscattering coefficients in C-band similar to that of  the tree canopy. This implies that unless the understory is removed following logging,  the decrease in C-band backscattering  might not be sufficient for detection of the forest loss. In addition, the higher level of noise present in the radar data complicates the detection of a change point in the monitoring period. However, as evidenced by the user’s accuracy of the radar-only result of 93 percent, the radar-based analysis is unlikely to introduce false detections of forest loss (i.e., unlikely to introduce commission errors). Given the 88 percent overall accuracy, the result demonstrates that the SAR-only algorithm is reliable for detecting tropical forest loss }if optical data are unavailable. 
}

\hg{
With 71 training days, the optical-only and hybrid approaches achieve overall accuracies of 94.9\% and 95.8\%, respectively, demonstrating that the anomaly detection algorithm developed in this study effectively detects \po{forest loss}. 
The optical-only algorithm has a user’s accuracy of 80 percent, which is lower than the \po{hybrid and radar-only }results. The \po{main reason for the lower user's accuracy is false detections introduced by }missed clouds and cloud shadows in consecutive observations in the time series. \po{While the} anomaly detection algorithm has a tolerance for clouds or shadows missed by the data pre-processing, \po{the presence of consecutive contaminated observations complicates the algorithm's ability to distinguish the contamination  from a true anomaly on the land surface}. Therefore, because consecutive cloudy observations is not uncommon in the humid tropics, the final result of the optical-only algorithm has more false positive detections than the other two algorithms.
}
\begin{figure*}[!htb]
\centering 
  \includegraphics[scale = 0.46]{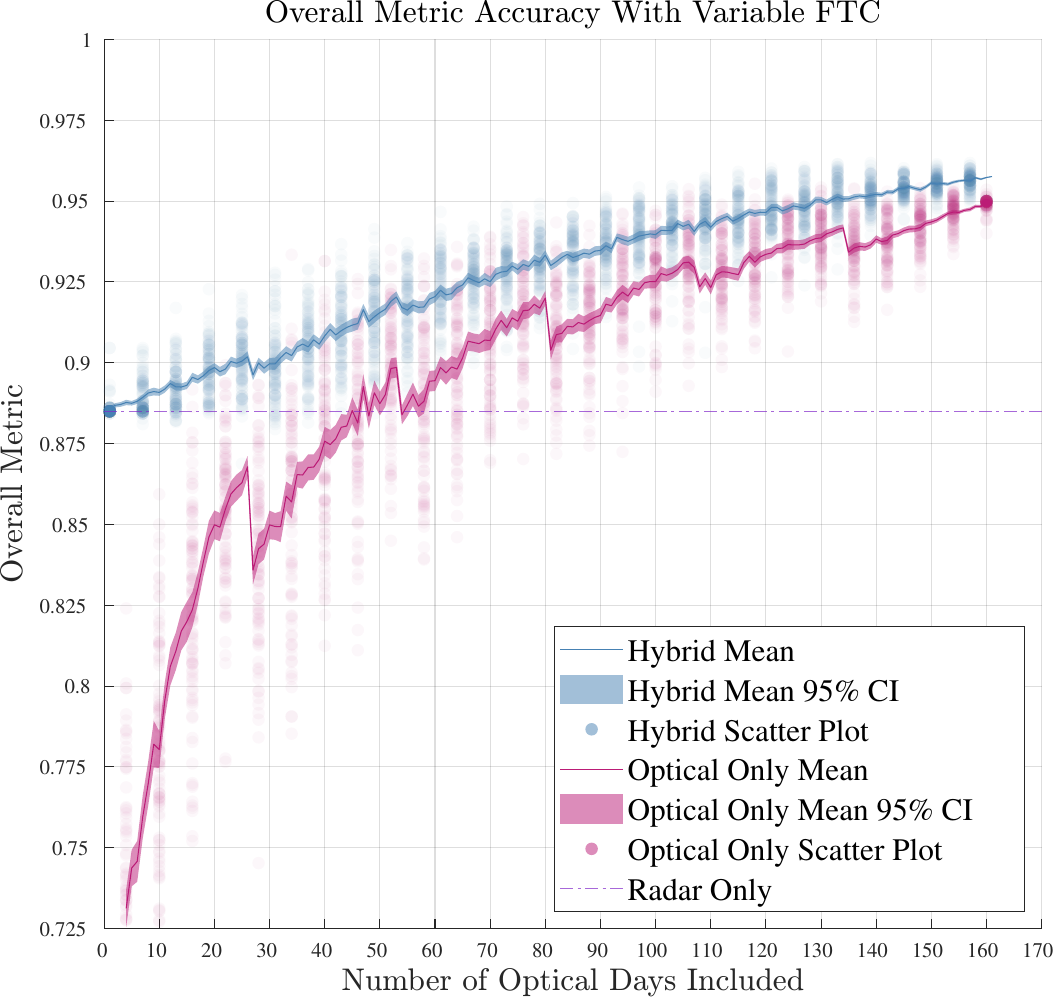}
  \includegraphics[scale = 0.46]{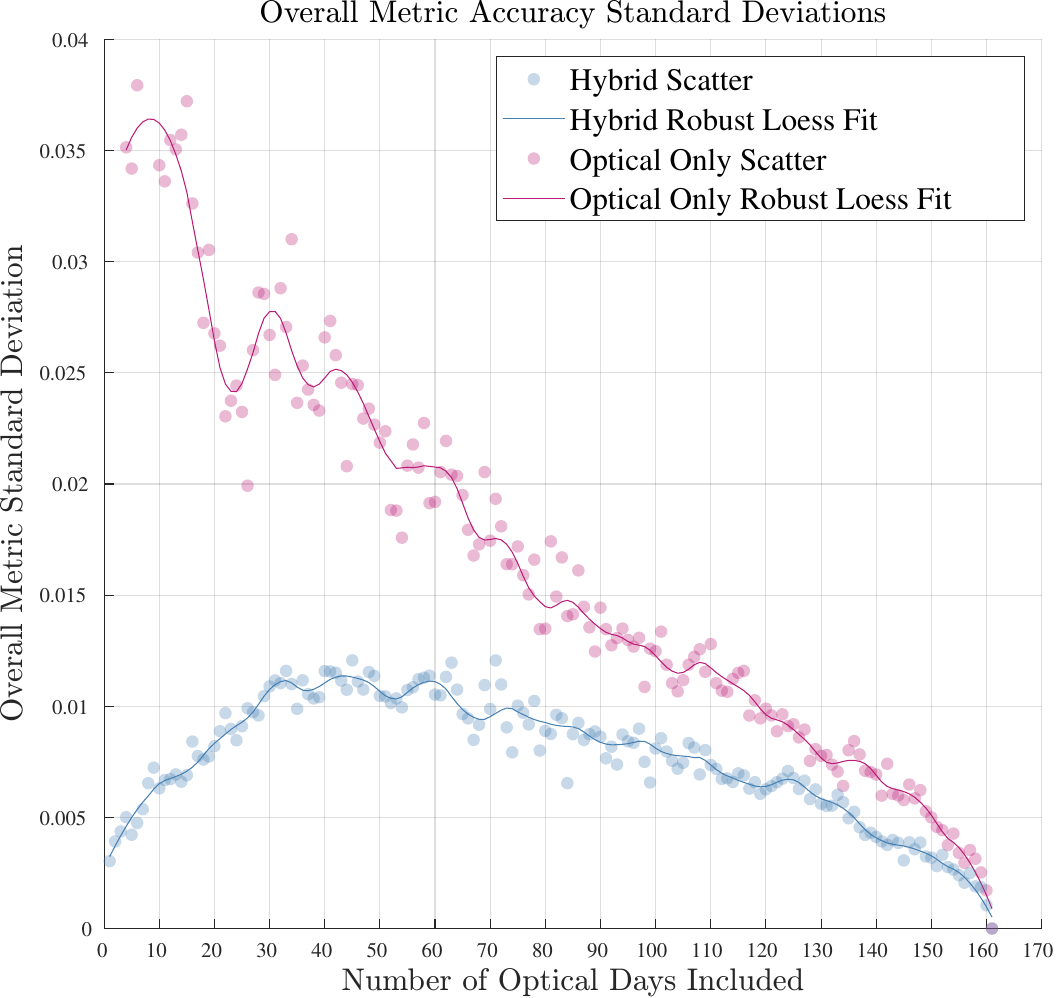}
    \caption{Overall metric accuracy and variance using variable FTC for hybrid and optical-only.  Hybrid is far more robust against decreases in the quantity of optical data, limiting to the Radar only performance as the number of optical days goes to zero.}
    \label{results:accuracy}
\end{figure*}

\begin{figure*}[!t]
\centering 
  \includegraphics[scale = 0.46]{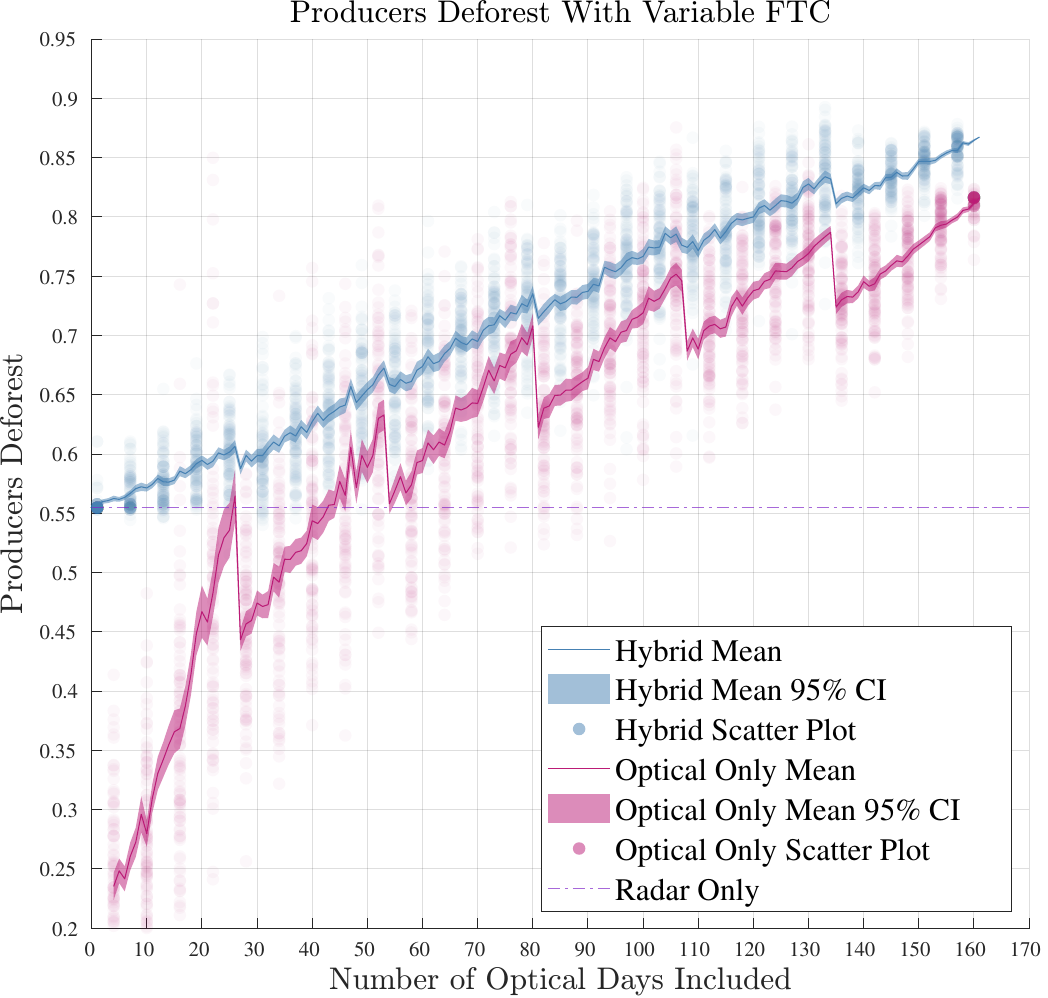}
     \includegraphics[scale = 0.46]{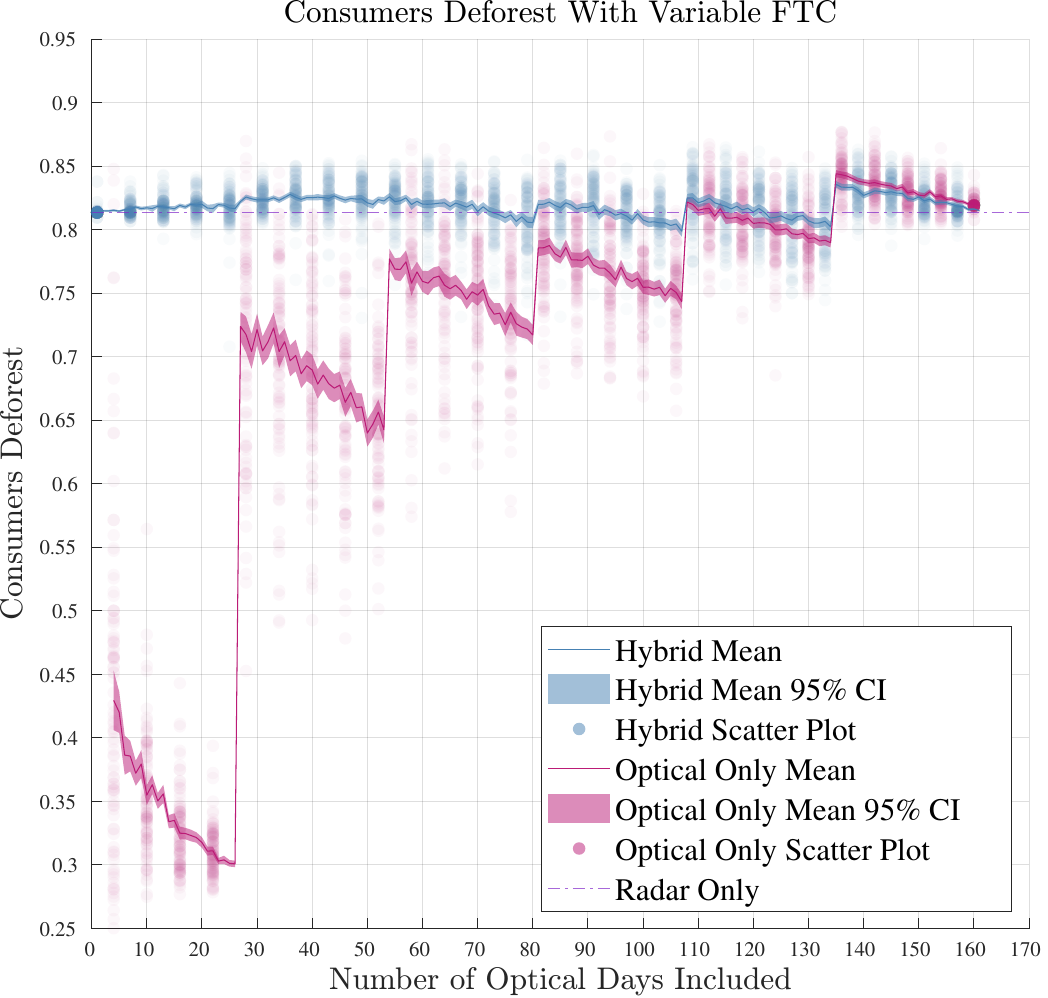}
    \caption{Left: Producers Deforest.  Right: Consumers Deforest.  Variable FTC used for hybrid and optical-only.}
    \label{results:producers}
\end{figure*}

\hg{
While multi-sensor fusion has been attempted to solve multiple types of land change problems, very few studies have analyzed the necessity of using a multi-sensor fusion method over single-sensor or optical-only algorithms. Using multiple data streams to solve a problem that can also be solved by just using one data source is a waste of computational resources. The similar accuracy of the optical-only and the hybrid results demonstrates that with enough clear observations, optical data alone can accurately detect \po{tropical forest loss}. However, our experiment that removes parts of the optical dataset also shows that the hybrid method is superior to the optical-only algorithm with fewer optical images available. The experiment here is based on the assumption that we can establish a solid benchmark state before the monitoring period. Only images from the monitoring period are removed in the experiment, while the normal state of the original land cover is still built from all the optical images in the training period. Therefore, the experiment is set as an optimal scenario for the performance of the optical-only method. The accuracy of the optical-only result will be lower if fewer observations are available to establish the benchmark state of the time series. As shown in \textbf{Figure} \ref{results:accuracy}, when more than 70 images are used in the 33-month monitoring period, or about 25 images a year, the performance of the optical-only result is close to the hybrid result. When fewer than 50 Sentinel-2 images are available \jc{during the monitoring period}, the accuracy of the optical-only algorithm is significantly lower than the hybrid algorithm. The sudden drop in both producer’s and user’s accuracy of the optical-only result when fewer than 20 images are included in the dataset shows that the algorithm based on Sentinel-2 only cannot effectively detect forest loss if fewer than 6 images are available each year. Previous studies have shown that fewer Sentinel-2 clear observations are available in central Africa, most of the Amazon Basin, and Southeast Asia than in the test area of this study because of the more frequent presence of cloud cover \cite{Sudmanns2020}. Therefore, the hybrid algorithm has the potential to better detect \po{forest loss}, or generally land-cover changes, in these regions than the current algorithms based on Landsat or Sentinel-2 data. 
}

\hg{
The high accuracy of the hybrid results demonstrate that the proposed methods can effectively detect forest disturbances in cloudy tropical rainforests with Sentinel-1 and Sentinel-2 data. \po{The presented method has potential to make contributions to various applications including early detection of illegal logging and to estimates of terrestrial carbon emissions.} Current estimates of the carbon emissions from land use/land cover change are uncertain compared to other carbon fluxes, \po{in part because of a lack of precise estimates from the forest sector }\cite{Friedlingstein2022}. Restrained by the low optical data density, the detection of \po{forest loss} in the humid tropics is more difficult than in other parts of the world \cite{Hansen_2016}. At the same time, the carbon densities of the tropical forests are some of the highest among all forest biomes \cite{Pan_science}. Therefore, lack of accurate maps \po{ and precise estimates} of forest \po{disturbance} in the tropics greatly influences the global estimation of terrestrial carbon fluxes. \po{Compared to the methods tested in this study,} the data fusion algorithm \po{presented} in this paper detects forest loss in data-scarce environments \po{faster and more accurately}. Future applications of deFOREST in tropical Africa and Southeast Asia will hopefully improve the estimates of \po{forest loss and associated carbon emissions.  Further, with the reliance on open and free data in combination with efficient computing,} the algorithm could help tropical countries to improve the mapping and estimation of \po{their forest resources}. 
}


\section*{Acknowledgments} 
This work has been funded in part by the National Science Foundation under grant number 2319011 \po{and by the National Aeronautics and Space Administration (NASA) through the NASA Carbon Monitoring System (grant number 80NSSC20K0151)}. We thank the reviewers for their insightful comments.  We are also grateful to Russell Goebel for his valuable contributions to this project.

\begin{IEEEbiography}
[{\includegraphics[width=1in,clip,keepaspectratio]{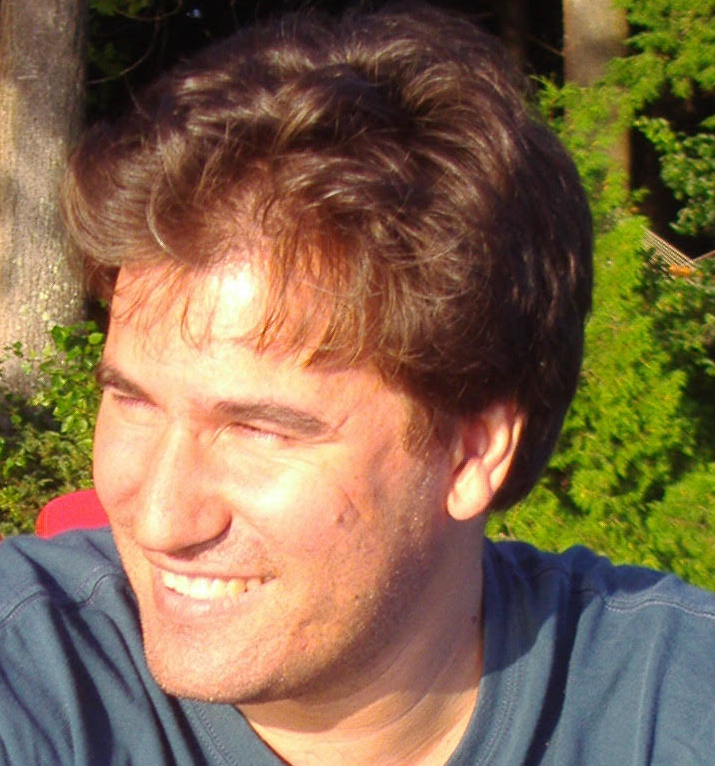}}]{Julio Enrique Castrill\'on Cand\'as} received the MS and Ph.D. degrees in electrical engineering and computer science from the Massachusetts Institute of Technology (MIT), Cambridge. He is currently a faculty member in the department of Mathematics and Statistics at Boston University. His area of expertise is in Uncertainty Quantification (PDEs, non-linear stochastic networks), large scale computational statistics, functional data analysis and statistical machine learning.
\end{IEEEbiography} 

\begin{IEEEbiography}[{\includegraphics[trim=0cm 0cm 0cm 4cm,clip,width=1in,keepaspectratio]{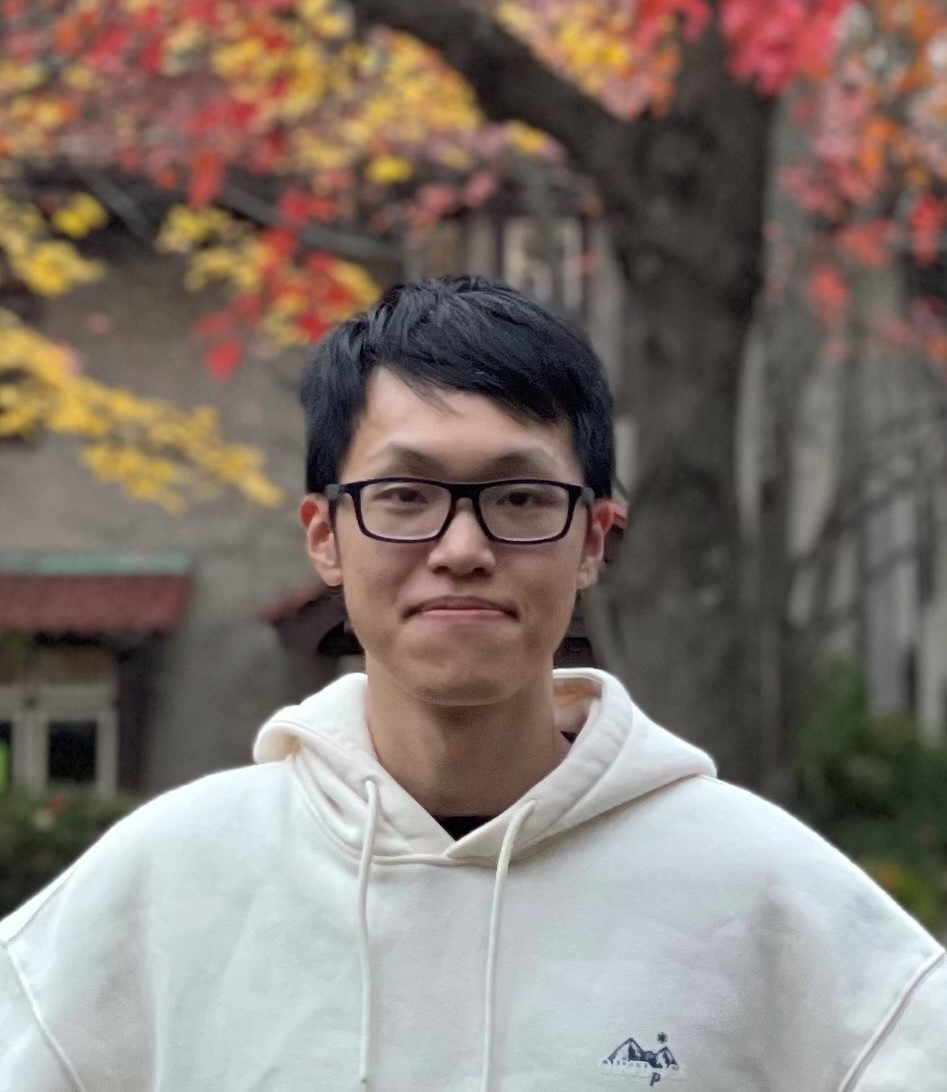}}]{Hanfeng Gu} received his Ph.D. degree in Earth and Environment from Boston University. His research focuses on monitoring land-use and land cover changes in tropical forests, coastal wetlands and urban environments using time series of multispectral and Radar remote sensing data.
\end{IEEEbiography} 

\begin{IEEEbiography}[{\includegraphics[width=1in, clip,keepaspectratio]{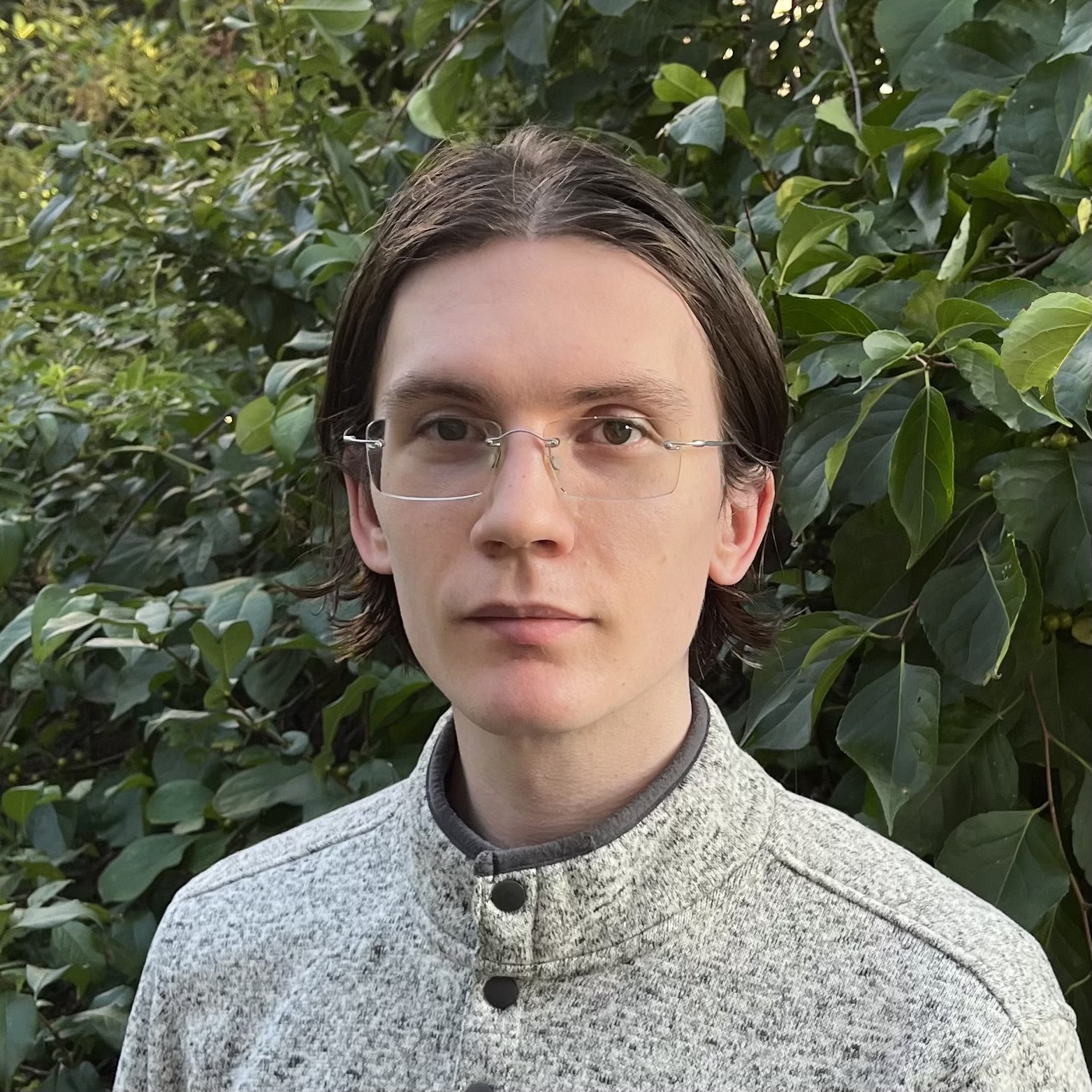}}]{Caleb Julius Meredith}received his BS degree in Mathematics from the University of Massachusetts Amherst (UMass Amherst).  He is currently a Math PhD student in the department of Mathematics and Statistics at Boston University.  He is most experienced with PDEs, numerical analysis, and deep learning.  
\end{IEEEbiography} 


\begin{IEEEbiography}[{
\includegraphics[width=1in, trim=1cm 0cm 1.5cm 0cm, clip,keepaspectratio]{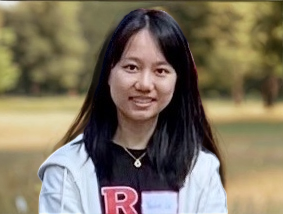}}]{Yulin Li} received the M.A. degree in Statistics from Boston University, USA, and a dual B.S. \textbackslash B.Eng. degree in Electrical Engineering from the University of Illinois at Urbana-Champaign, USA, and Zhejiang University, China. She is currently pursuing the Ph.D. degree in Statistics at Rutgers University, New Brunswick, USA. Her research interests include machine learning fairness, transfer learning theory, high-dimensional and functional data analysis, and modeling of multi-modal and non-i.i.d. data.
\end{IEEEbiography}

\begin{IEEEbiography}[{\includegraphics[width=1in,clip,keepaspectratio]{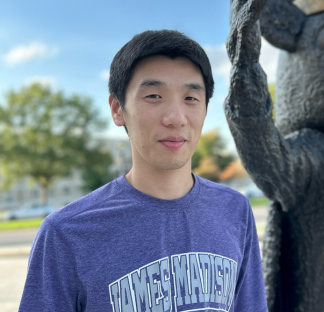}}]{Xiaojing Tang}  received his M.A. in Environmental Remote Sensing and GIS and Ph.D. in Geography from Boston University. He is currently an Assistant Professor in the School of Integrated Sciences at James Madison University. His research focuses on monitoring land changes using time series analysis of remote sensing data. He is a member of the NASA SERVIR Applied Science Team and the NASA LCLUC Science Team.
\end{IEEEbiography}

\begin{IEEEbiography}[{\includegraphics[width=1in,clip,keepaspectratio]{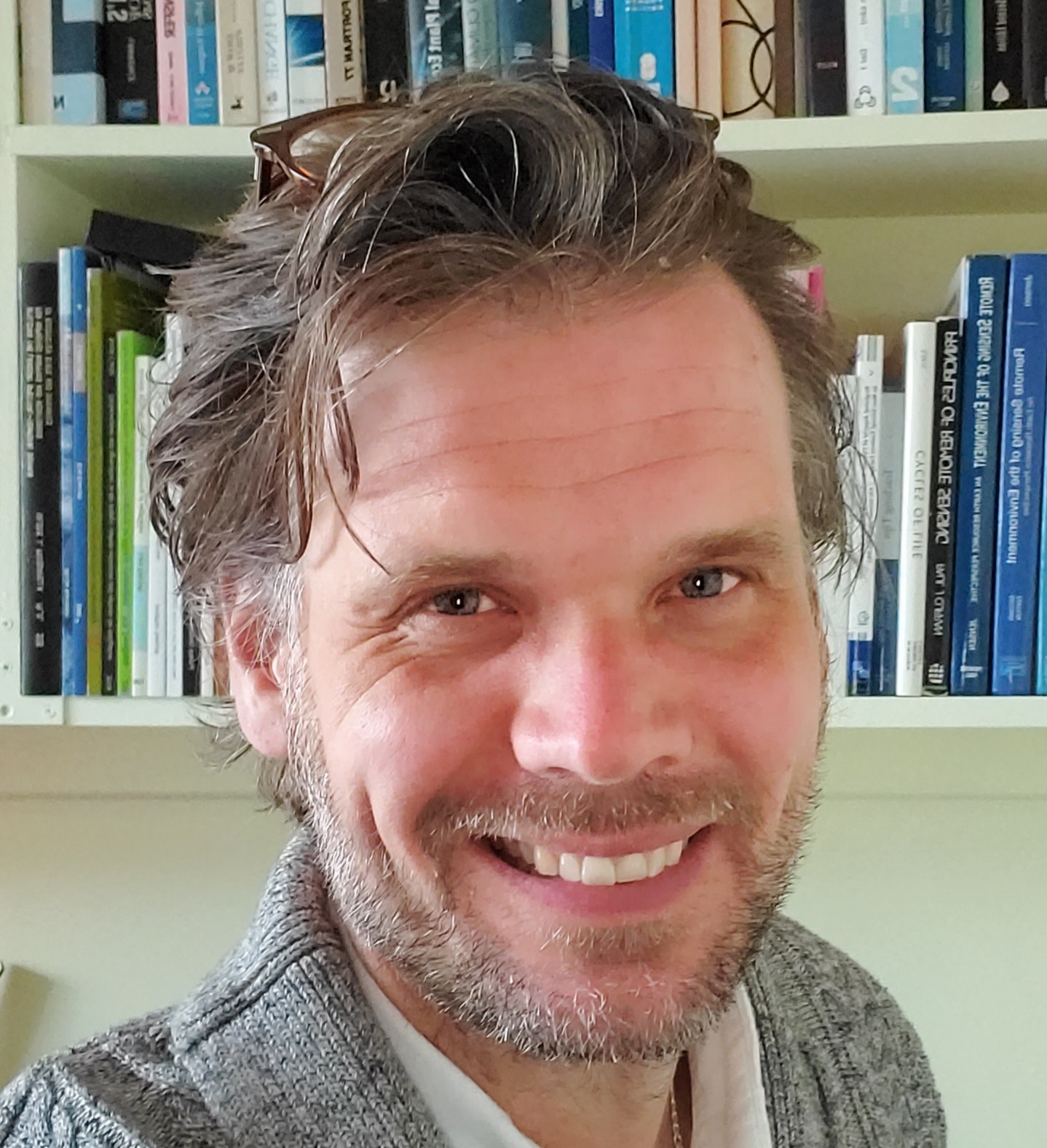}}]{Pontus Olofsson} is a research scientist at NASA Marshall Space Flight Center. Before joining NASA, he was in the department of Earth and Environment at Boston University. Olofsson has degrees in physical geography and mathematical statistics from Lund University, Sweden. He received his PhD in Physical Geography from Lund University in 2007. Olofsson’s primary research interests revolve around the mining of archives of satellite data to further the understanding of how Earth is changing and the impact of change on people and environment. He has a special interest in how to combine remote sensing data and traditional sampling techniques to gain knowledge of environmental change. 
\end{IEEEbiography}

\begin{IEEEbiography}[{\includegraphics[width=1in,clip,keepaspectratio]{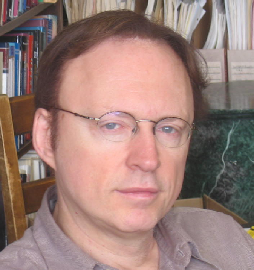}}]{Mark Kon} obtained Bachelor’s degrees in Mathematics,
Physics, and Psychology from Cornell
University, and a PhD in Mathematics from MIT. He
is a professor of Mathematics and Statistics
at Boston University. He is affiliated with the
Quantum Information Group, the Bioinformatics
Program and the Computational Neuroscience
Program. He has had appointments at Columbia
University as Assistant and Associate Professor
(Computer Science, Mathematics), as well as at
Harvard and at MIT. He has published approximately
100 articles in mathematical physics, mathematics and statistics, 
computational biology, and computational neuroscience, including two
books. His recent research and applications interests involve quantum
probability and information, statistics, machine learning, computational
biology, computational neuroscience, and complexity.
\end{IEEEbiography}

 \bibliographystyle{IEEEtran}
\bibliography{References/juliorefs,References/KonRefs,References/SuchiRefs,References/PontusRefs,References/citations,References/multilevel,References/changedetectionreferences,References/Xiaojing-Refs}

\clearpage
\section*{Supplementary Material}


\renewcommand{\thesection}{S.\Roman{section}} 
\renewcommand{\thesubsection}{\thesection.\Alph{subsection}}

\renewcommand{\thefigure}{S\arabic{figure}}
\renewcommand{\thetable}{S\arabic{table}}

\renewcommand{\theequation}{S\arabic{equation}}

\setcounter{section}{0}
\setcounter{figure}{0}

\bigskip


\section{Discrete KL Expansion}
\label{discrete}
The following discrete KL expansion  was developed by Trajan Murphy during our 
discussions.  This exposition is mathematically rigorous.
Let $(\Omega, \mcF, \bbP)$ let be a complete probability space, with the set of
events $\Omega$, the associated sigma algebra $\mcF$ and the probability measure
$\bbP$.
\begin{theorem}
Let $\bv(\omega) = [v_1(\omega), \dots, v_n(\omega)] \in L^{2}(\Omega;\R^{n})$
be a random vector and covariance matrix
$\bC := \eset{(\bv - \eset{\bv})(\bv - \eset{\bv})^{T}}$. Suppose that $\bC$
is a positive definite matrix with eigenpairs $(\lambda_{k},\bphi_{k})$ such
that for $k = 1,\dots,n$ 
\[
\bC\bphi_k = \lambda_{k} \bphi_k,
\]
and $\lambda_1 \geq \dots \geq \lambda_n$
then there exists a set of zero-mean random variable $Y_1(\omega), \dots Y_{n}(\omega)$
such that 
\[
\bv(\omega) = \eset{\bv(\omega)} +  \sum_{k = 1}^{n} \sqrt{\lambda_k} \bphi_k Y_k(\omega),
\]
where $\eset{Y_k(\omega)Y_l(\omega)} = \delta[l-k]$.
\label{appendix:thm1}
\end{theorem}

\begin{proof} Let $\bw(\omega) = \bv(\omega) - \eset{\bv(\omega)}$, then 
$\eset{\bw(\omega)} = 0$.   
     Since $\bC$ is a positive definite matrix, then $\{\bphi_1,\dots, \bphi_n\}$
    are an orthonormal basis for $\R^n$.
    Let $P:\R^n \rightarrow \R^n$ be the orthogonal projection onto
    $\{\bphi_1,\dots, \bphi_n\}$, then
    \[
    P\bw(\omega) = \sum_{k=1}^{n} (\bw(\omega)^T \bphi_k) 
    \bphi_k
    \]
    and $\bw(\omega) = P \bw(\omega)$. 
    
    For $k = 1,\dots,n$ let $Z_k(\omega) := \bw(\omega)^T \bphi_k$ and thus $\eset{Z_k}=0$.
    Let $l,k = 1,\dots,n$, then
    \[
    \begin{split}
    \eset{Z_k(\omega) Z_l(\omega)} 
    &= \eset{ \bw(\omega)^T \bphi_k \bw(\omega)^T \bphi_l}       \\ 
    &= \eset{ \bphi_k^T \bw(\omega) \bw(\omega)^T \bphi_l}       \\ 
    &= \bphi_k^T \eset{\bw(\omega) \bw(\omega)^T} \bphi_l       \\ 
    &= \bphi_k^T \bC \bphi_l        
    = \lambda_{l} \bphi_k^T \bphi_l   
    = \lambda_{l} \delta[k-l].       \\ 
    \end{split}
    \]
    Now, for $k = 1,\dots,n$ let $Y_k(\omega) = \frac{Z_k(\omega)}{ \sqrt{\lambda_k}}$. The result
    follows.
\end{proof}


A crucial characteristic of the KL expansion is the optimality properties.  
Suppose that we form the truncated KL expansion i.e. for any $m \leq n$
\[
\bv_{m}(\omega) = \eset{\bv(\omega)} +  \sum_{k = 1}^{m} \sqrt{\lambda_k} \bphi_k Y_k(\omega).
\]

\begin{theorem}
Suppose $\bpsi_{1},\dots,\bpsi_{n}$ is an orthonormal basis of $\R^{n}$ and
let $\bQ^m$ be a projection of $\bv(\omega)$ onto 
$\bpsi_{1},\dots,\bpsi_{m}$, then 
\[
    \begin{split}
    \eset{\|\bv(\omega) -  \bv_{m}(\omega)\|^{2}} 
    &= \sum_{k=m+1}^{n} \lambda_{k} \\ 
    &\leq \eset{
    \|\bv(\omega) 
    - \bQ^m \bv(\omega) \|^{2}
    }
    \end{split}
\]
\label{appendix:thm2}
\end{theorem}
\vspace{-1cm}
\begin{proof}
We first have that
    \[
      \bv(\omega) -  \bv_{m}(\omega)  = \sum_{k=m+1}^{n}  \sqrt{\lambda_k} \bphi_k Y_k(\omega),
    \]
Using the orthonormality properties of $\{\bphi_1,\dots,\bphi_{n}\}$ we have that

    \[
    \begin{split}
      & \| \bv(\omega) -  \bv_{m}(\omega)\|^2 = \\ 
      &= 
      \left( \sum_{k=m+1}^{n}  \sqrt{\lambda_k} \bphi_k^{T} Y_k(\omega) \right) 
      \left( \sum_{l=m+1}^{n}  \sqrt{\lambda_l} \bphi_l  Y_l(\omega) \right) \\
      &= \sum_{k=m+1}^{n}   \sum_{l=m+1}^{n}  
      \sqrt{\lambda_k} \sqrt{\lambda_l} 
      \bphi_k^{T} \bphi_l 
      Y_k(\omega) Y_l(\omega) \\
      &= \sum_{k=m+1}^{n}  
      \lambda_k 
      Y^2_k(\omega)
      \\
    \end{split}
    \]
    From the unit variance of the random variables $Y_1(\omega),\dots,Y_n(\omega)$ we 
    have that
    \[
    \eset{\|\bv(\omega) -  \bv_{m}(\omega)\|^{2}} = \sum_{k=m+1}^{n} \lambda_{k}.
    \]

Let  $\tilde \bv = \bQ^m \bv(\omega) = \sum_{k=1}^{m} G_k(\omega) \bpsi_k$
for some set of projection coefficients $G_1(\omega),\dots,G_m(\omega)$
Let $P_{\bpsi}:\R^n \rightarrow \R^m$ be the orthogonal projection of $\R^n$ onto the basis $\{\bpsi_1,\dots,\bpsi_m\}$. Since $P_{\bpsi}$ is the orthogonal projection then a.s.
\[
    \|\bv(\omega) -  \tilde \bv (\omega)\|^{2} \geq  \|\bv(\omega) -  P_{\bpsi}\bv(\omega)\|^{2}
\]
and thus
\[
    \eset{\|\bv(\omega) -  \tilde \bv (\omega)\|^{2}} \geq  \eset{\|\bv(\omega) -  P_{\bpsi}\bv(\omega)\|^{2}}.
\]
Now,
    \[
    \begin{split}
      \| \bv(\omega) -  P_{\bpsi}\bv(\omega)\|^2  
      &= 
      \left\| \sum_{k=m+1}^{n} (\bv(\omega)^T  \bpsi_k) \bpsi_k \right\|^2 \\
      &= \sum_{k=m+1}^{n} (\bv(\omega)^T  \bpsi_k)^2
      \\
      &= \sum_{k=m+1}^{n} \bpsi_k^T \bv(\omega) 
      \bv(\omega)^T 
      \bpsi_k
    \end{split}
    \]
and thus
    \[
    \begin{split}
      \eset{\| \bv(\omega) -  P_{\bpsi}\bv(\omega)\|^2}    
     &= \sum_{k=m+1}^{n} \bpsi_k^T \eset{\bv(\omega) 
      \bv(\omega)^T }
      \bpsi_k \\
     &= \sum_{k=m+1}^{n} \bpsi_k^T \bC
      \bpsi_k
    \end{split}
    \]
We now solve for the following constrained optimization problem:    
    \[
    \argmin_{\{\bpsi_{m+1},\dots,\bpsi_n\}}
     \sum_{k=m+1}^{n} \bpsi_k^T \bC
      \bpsi_k.
    \]
We can solved for this problem using an inductive argument.    
It is not hard to see that
 \[
    \argmin_{\bpsi_n \in \R^n}
     \bpsi_n^T \bC
      \bpsi_n = \lambda_{n}
    \]
and this is achieved by letting $\bpsi_n = \bphi_n$. Now, the next
choice of vector $\bpsi_{n-1}$ has to be such that the following 
optimization problem is solved
 \[
    \argmin_{\{\bpsi_{n-1} \in \R^n|\,\, \bpsi_{n-1} \perp \spn \bphi_n \}}
     \bpsi_{n-1}^T  \bC
      \bpsi_{n-1}
    \]
    The solution to this optimization is $\bpsi_{n-1} = \bphi_{n-1}$.
In general for $k<n$ we have that
 \[
    \argmin_{\{\bpsi_k \in \R^n|\,\, \bpsi_{k} \perp \spn\{ \bphi_{k+1},\dots,\bphi_n \} \}}
     \bpsi_{k}^T  \bC
      \bpsi_{k} = \lambda_k
    \]
$\bpsi_k = \bphi_k$. The results follows.
\end{proof}

Suppose that we form the residual vector
\[
\br(\omega) = \bv(\omega) -  \bv_{m}(\omega)  = \sum_{k=m+1}^{n}  \sqrt{\lambda_k} \bphi_k Y_k(\omega),
\]
where the $i_{th}$ entry of the residual vector corresponds to a pixel in the data.  
Let $\alpha$ be the significance level then it can be show that the distribution 
of the null hypothesis $\mbox{H}_0$ satisfies a concentration bound.

\begin{rem}
It is important
to note that for this hypothesis test no assumptions are made from the distribution of the data,
which for high dimensional problems it is practically impossible to obtain. The concentration of
the bound depends on the decay of the eigenvalues $\lambda_k$ and the truncation parameter $m$.
However, it is clear that if we choose $m=n$ then the residual is exactly zero. The parameter
$m$ has to be calibrated such that most of the signal for the nominal behavior is captured 
by the basis $\{\bphi_1,\dots,\bphi_m\}$.
\end{rem}

The following
theorem shows how this bound is obtained.

\begin{theorem} 
Suppose that we form the residual vector
\[
\br(\omega) = \bv(\omega) -  \bv_{m}(\omega)  = \sum_{k=m+1}^{n}  \sqrt{\lambda_k} \bphi_k Y_k(\omega),
\]
and let $\alpha \in (0,1)$ be the significance level then 
\[
\bbP\left(|\br[i]| \geq \alpha^{-\frac{1}{2}}
\left(\sum_{k = m + 1}^{n}\lambda_k \bphi_k[i]^2 \right)^{\frac{1}{2}} 
\right) \leq \alpha.
\]
\label{appendix:thm3}
\end{theorem}
\begin{proof}

Since $\eset{ Y_k(\omega) Y_l(\omega)} = \delta[l-k]$ then
\[
\begin{split}
\eset{\br[i]^2} &= \sum_{k=m+1}^{n} \sum_{l=m+1}^{n} 
\sqrt{\lambda_k} \sqrt{\lambda_l} 
\bphi_k[i] \bphi_l[i]
\eset{Y_k(\omega) Y_l(\omega)} \\
&= 
\sum_{k=m+1}^{n} \lambda_k \bphi_k[i]^2 
\end{split}
\]
The result follows from the Chebyshev inequality.
\end{proof}

\section{Bayesian spatio-temporal SAR filter}

The following provides details for the spatio-temporal Bayesian filtering used to clean the Sentinel-1 SAR data during preprocessing (See \textbf{Figure} \ref{SS:Fig1} for and example of the Baesian filter on radar data).  Let $\{Y_n\}_{n=1}^T$ for $Y_n\in \mathbb{R}^{n_1n_2}$ denote a set of $T$ flattened  $n_1\times n_2$ measurements indexed in order by time, and $\{X_n\}_{n=1}^T$ for $X_n\in \mathbb{R}^{n_1n_2}$ denote the corresponding true unknown values.  We make the standard assumption that $Y_n|X_n\sim\mathcal{N}(X_n,\sigma_1^2\textrm{I})$, i.e. that our measurements contain some amount of uncorrelated noise, and two assumptions on the spatial and temporal distributions of our true values.  For the former we want our data to be smooth in the sense that our second derivatives are not too large, so we assume that $\Delta X_n\sim \mathcal{N}(0,\sigma_2^2\textrm{I})$.  However, given that we have discrete grids/vectors of observations this is replaced by $DX_n\sim \mathcal{N}(0,\sigma_2^2 \textrm{I})$, where $D\in \mathbb{R}^{n_1n_2\times n_1 n_2}$ is the 2D discrete Laplacian matrix using a 7 point scheme and Neumann boundary conditions.  Finally, for the latter we assume that $X_{n}|X_{n-1}\sim \mathcal{N}(X_{n-1},\sigma_3^2\textrm{I})$, i.e. that values cannot vary too much with time.  Using these priors we can construct and maximize our log-likelihood function with respect to $X_n$, the result of which is used in place of $Y_n$ as our filtered data.  We have that 
$$\rho(X_n|Y_n,X_{n-1})=\frac{\rho(Y_n|X_n)\rho(X_{n-1}|X_n)\rho(X_n)}{\rho(Y_n,X_{n-1})}$$
so our log-likelihood function up to a constant is given by
$$\ln{(\rho(Y_n|X_n)\rho(X_{n-1}|X_n)\rho(X_n))}=$$$$-\frac{1}{2\sigma_1^2}(Y_n-X_n)^T(Y_n-X_n)$$$$-\frac{1}{2\sigma_3^2}(X_{n-1}-X_n)^T(X_{n-1}-X_n)-\frac{1}{2\sigma_2^2}(DX_n)^TDX_n$$  Differentiating with respect to $X_n$ and setting equal to zero gets us that 
$$X_n=[(\frac{1}{\sigma_1^2}+\frac{1}{\sigma_3^2})\textrm{I}+\frac{1}{\sigma_2^2}D^TD]^{-1}(\frac{1}{\sigma_1^2}Y_n+\frac{1}{\sigma_3^2}X_{n-1})$$
If we drop our temporal prior we can filter $X_0$ via a similar procedure with $$X_0=[\frac{1}{\sigma_1^2}\textrm{I}+\frac{1}{\sigma_2^2}D^TD]^{-1}\frac{1}{\sigma_1^2}Y_0$$
In practice all values are scaled by $\sigma_1^2$ so that only two parameters need to be tuned.   Additionally, due to large sizes the
matrix (indirect) inverses are approximated via the Preconditioned Conjugate Gradients Method using modified incomplete Cholesky factorization to compute the preconditioner factors.  

\begin{figure}[H]
\centering
{\includegraphics[width = 8.5cm]{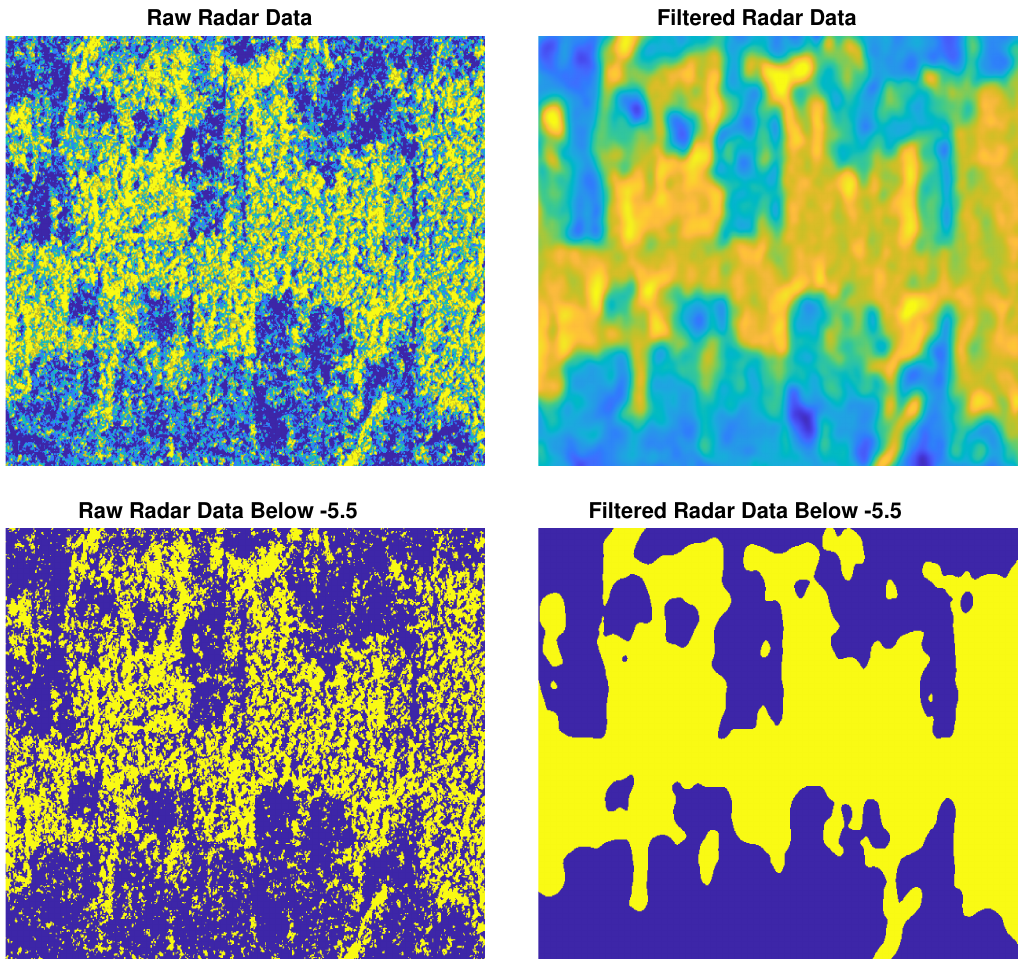}}
    \caption{Top row: Raw radar data and Filtered radar data for day 234 (December 28, 2022), data range -8 to -4 for both plots.  Bottom row:  Portion of the region below (in purple) and above (in yellow) the -5.5 threshold for both the Raw and Filtered radar data for day 234.}
  \label{SS:Fig1}
\end{figure}

\section{Variable FTC}
Prior to validation we picked the optimal threshold and FTC parameters for hybrid, radar, and optical only.  Since hybrid converges to radar only as the number of optical days goes to zero, we can linearly interpolate between their FTC values to get a reasonable value for any number of optical days between 0 and 161, with the endpoints corresponding to the radar only and hybrid data sets used for calibration.  For optical only we don't have parameters selected for the left endpoint, since there is no data to process there.  However, it seems reasonable to assume that the FTC should be proportional to the number of optical days; if we halve the data density, then we should only require half as many days of deforestation in a row to get a confirmation.  To test this, for each number of optical days we randomly selected 30 sets of days of that length, and for each found the FTC that maximized the Overall Metric.  \textbf{Figure} \ref{SS:ftc} shows the mean of those 30 values as well as a linear regression without an intercept fit to the optical data sets of length 20 to 161.  The regression doesn't include the first 19 data points because performance seriously degrades in this region, making the results much noisier, and because the minimum FTC of 1 means that a linear trend cannot continue indefinitely.   As you can see, there is a strong linear relationship here, with $R^2=0.9686$.  This indicates that it is reasonable to set the FTC roughly proportional to the number of optical days.

\begin{figure}[hbt]
\centering
  \includegraphics[scale = 0.45]{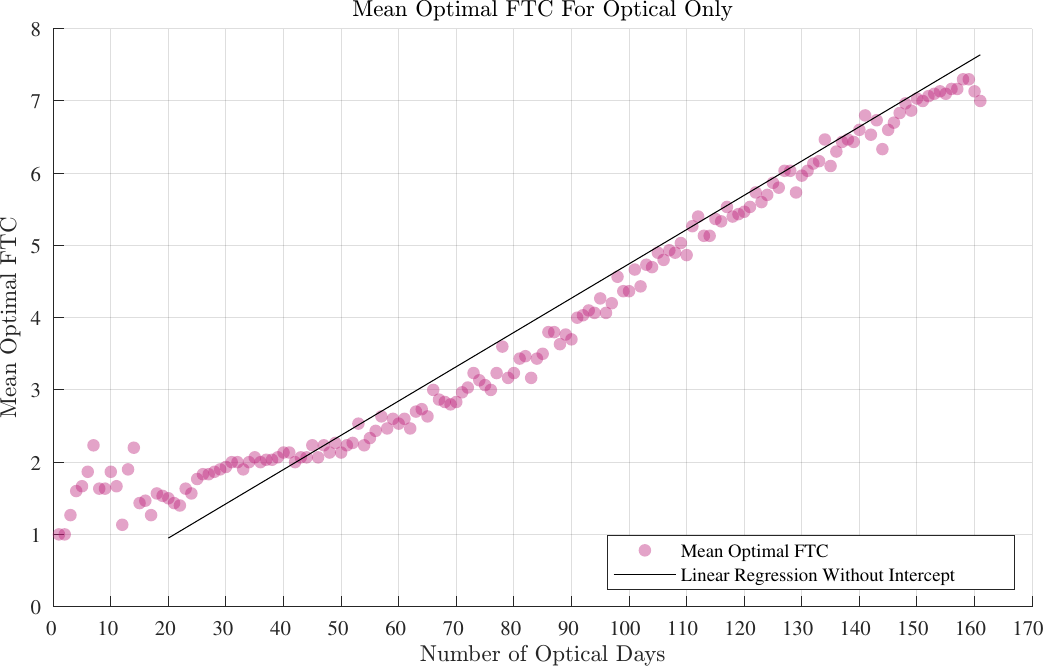}
  \caption{Mean Optimal FTC at each number of optical days, averaged over 30 samples each.}
     \label{SS:ftc}
\end{figure}

We should note that in the results we did not use these mean values or the regression line for selecting FTC values; this graph only provided confirmation of the linear relationship previously assumed.  The difference in performance between variable and fixed FTC can be seen in \textbf{Figure} \ref{SS:fixed}.

\section{Unprocessed Optical Data}
\label{unprocessed}

A natural question for this method is whether the anomaly values have better class separation than the unprocessed optical data.  To test this we hand picked new threshold and FTC parameters for the unprocessed optical data and recomputed our metrics, the results from which can be seen in \textbf{Table} \ref{results:utable}.  Since we don't need to process our optical data we added back in the data corresponding to the period used for constructing the KL expansion.  However, this didn't change any of the metrics, so our computation times are for running without that unneeded extra data. 

\begin{figure}[hbt]
\centering
  \includegraphics[scale = 0.45]{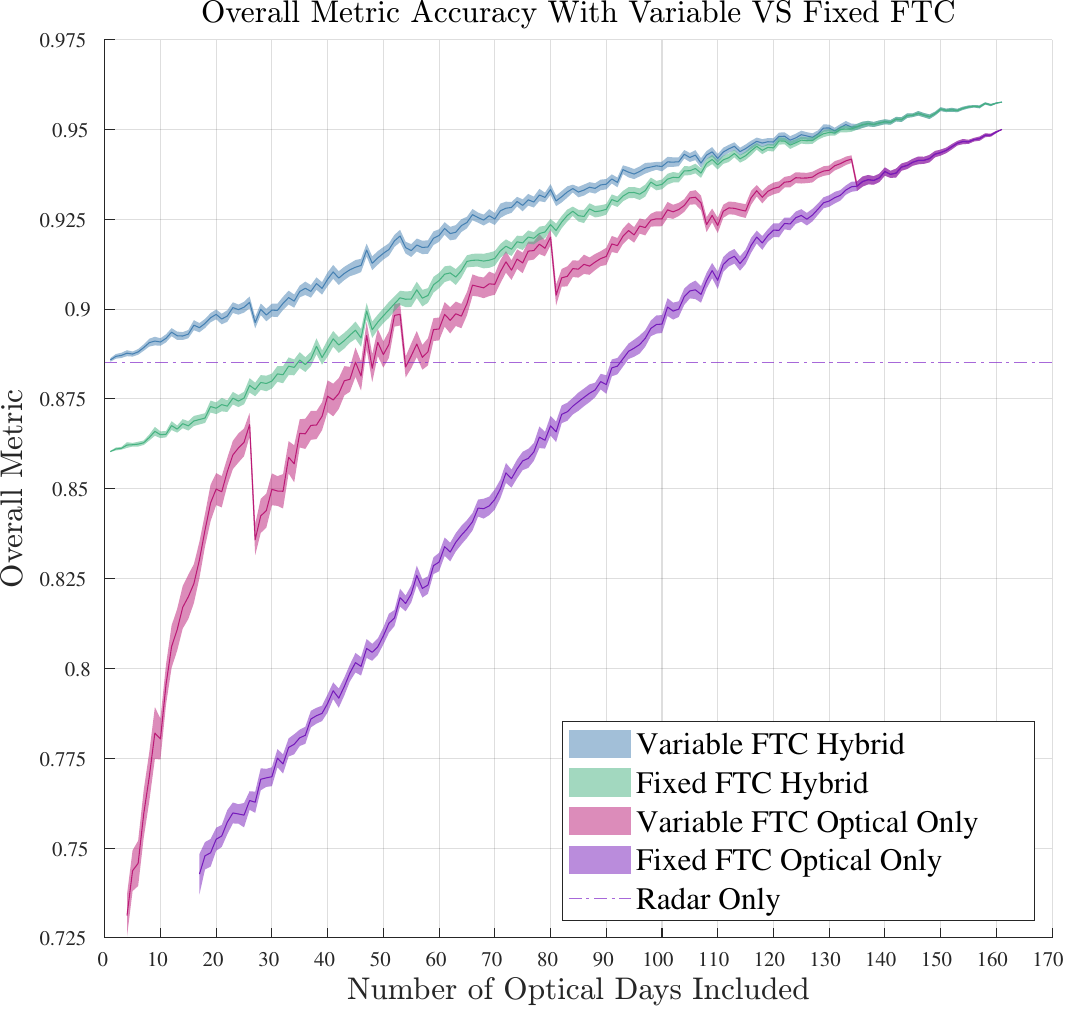}
  \caption{Comparison of Overall Metric for Variable vs Fixed FTC.}
     \label{SS:fixed}
\end{figure}

\textbf{Table} \ref{results:utable} appears to show us that the results for the unprocessed optical data are better, with comparable users accuracy and much better producers accuracy.  However, we can see in \textbf{Figure} \ref{results:hybrid_comp}---which shows the hybrid results using optical anomaly data on the top and unprocessed optical data on the bottom---that using the latter causes false positives in the marshy land at the top right.  This is because wet forest and bare ground both have low EVI values, and are therefore classified together.  The KL expansion is able to better separate these states, although some detections persist due to the radar data, which is also affected by water.  \textbf{Figure} \ref{results:valp} shows us why our metrics don't capture this improvement: our validation points are neither dense enough or concentrated enough to represent these regions.  

\begin{figure}[htb]
\centering 
\includegraphics[width=8.25cm,height=8.25cm]{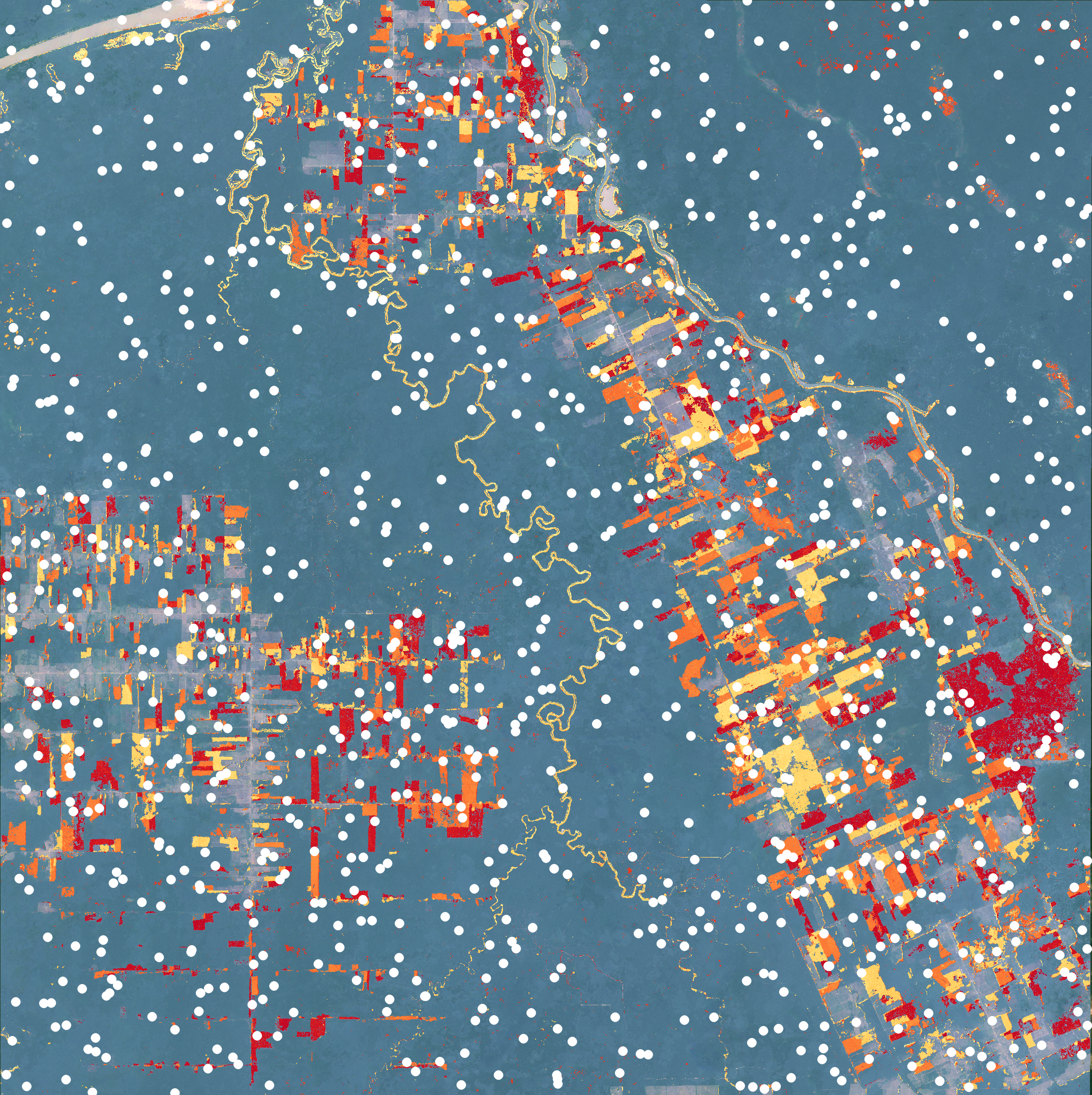}
\caption{Hybrid results using optical anomaly data overlaid with validation point locations.}
    \label{results:valp}
\end{figure}

\begin{table*}[!b]
    \caption{Algorithm accuracy results using optical anomaly data (deForest) vs unprocessed optical data (HMM)}
    \label{results:utable}
    \centering
        \begin{tabular}{>{\centering\arraybackslash}p{4.5cm} 
    c c c c c}
    \toprule
\rowcolor{olive!40}  Algorithm (Data) & Training Days & Overall Acc. & User Acc.  & Producer Acc. & Computational Time (h) \\
\midrule
\rowcolor{blue!20} deForest (Optical) & 71  & 0.931 & 0.823 & 0.718 &  13.95 \\
deForest (Optical + Radar) & 71     & 0.942 & 0.878 & 0.745 & 49.47\\
\rowcolor{blue!20} HMM  (Raw Optical)& NA   & 0.9668 & 0.8565 & 0.9105 & 9.82\\
HMM (Raw Optical + Radar)& NA   & 0.9614 & 0.8391 & 0.889 & 45.34\\
    \end{tabular}
\end{table*}

\begin{figure}[H]
\centering 
\begin{tikzpicture}[scale = 1, every node/.style={scale=1}]
\begin{scope}
\node at (0,0){\includegraphics[width=8.2cm,height=8.2cm]{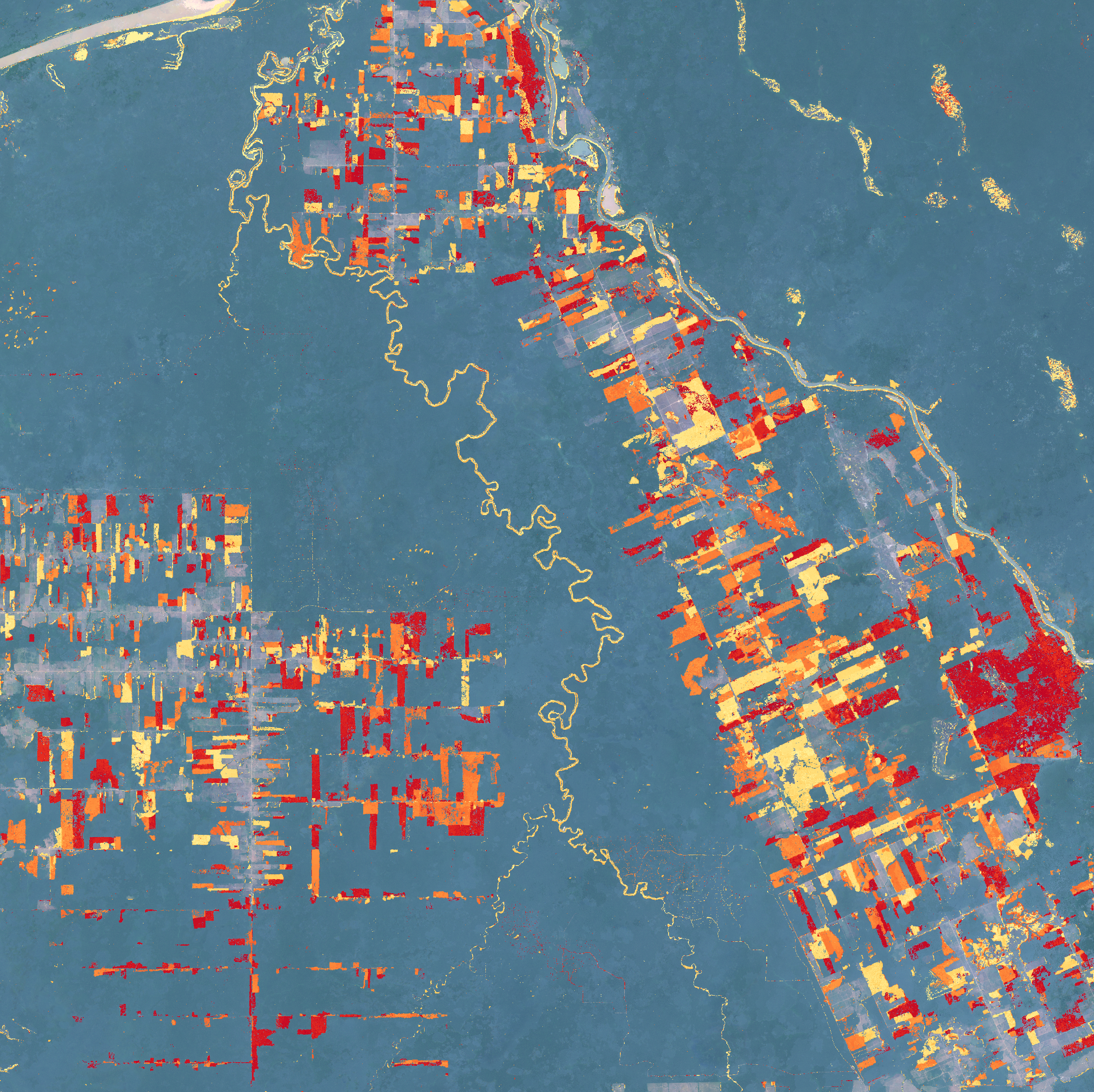}};
\node at (0,8.2){\includegraphics[width=8.2cm,height=8.2cm]{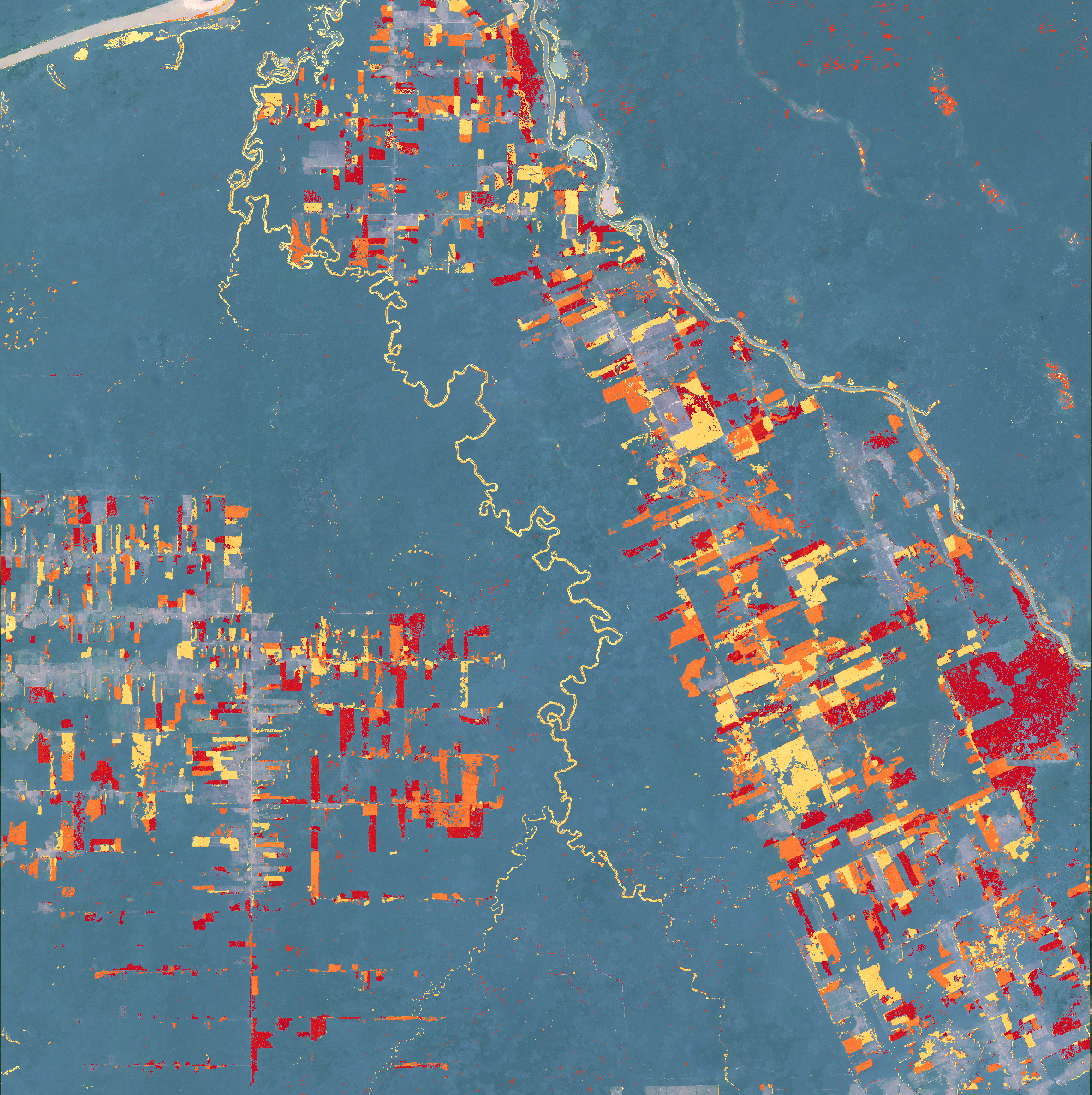}};
\draw[draw = black] (1.3,10.775) rectangle ++(1.3,1.52);
\draw[draw = blue] (1.75,9.8) rectangle ++(0.25,0.4);
\draw[draw = blue] (2.775,10.7) rectangle ++(0.275,0.2);
\draw[draw = blue] (2.8,11) rectangle ++(0.4,0.86);
\draw[draw = blue] (3.225,10.65) rectangle ++(0.36,0.36);
\draw[draw = blue] (3.825,10.4) rectangle ++(0.25,0.25);
\draw[draw = blue] (3.65,9.175) rectangle ++(0.4,0.49);

\draw[draw = black] (1.3,2.575) rectangle ++(1.3,1.52);
\draw[draw = blue] (1.75,1.6) rectangle ++(0.25,0.4);
\draw[draw = blue] (2.775,2.5) rectangle ++(0.275,0.2);
\draw[draw = blue] (2.8,2.8) rectangle ++(0.4,0.86);
\draw[draw = blue] (3.225,2.45) rectangle ++(0.36,0.36);
\draw[draw = blue] (3.825,2.2) rectangle ++(0.25,0.25);
\draw[draw = blue] (3.65,0.975) rectangle ++(0.4,0.49);
\end{scope}
\end{tikzpicture}
\caption{Top: Hybrid results using optical anomaly data.  Bottom: Hybrid results using unprocessed optical data.  The anomaly data separates wet forest and deforested land better than the unprocessed optical data, since both have low EVI values.  This can be seen in the marshy regions at the top right, where most of the detections in the black rectangle are removed and the detections in the blue regions are reduced when using the optical anomaly data.}
    \label{results:hybrid_comp}
\end{figure}

Part of the reason for this is that the stratified sampling used to pick the validation points was based on the results using the optical anomaly data, so those marshy regions were far less likely to be selected.  While it would be possible to resample the validation points with those regions in mind, from a more conceptual point of view we can see that the anomaly data is still preferable.  Although the unprocessed optical data has good class separation between dry bare ground and dry forest, and the regions of false positives are relatively small for this example, using the unprocessed data would result in significantly worse performance for regions with heavy rainfall/flooding or mostly marshy land, in which case the anomaly data would clearly be the superior choice.

\section{Missing Data}
\label{Supp:missing}
The descriptions of the KL expansion in \textbf{Section} 2 and \textbf{Section} \ref{discrete} assume that there is no missing data in our observations, which is not the case.  In fact, about two thirds of the data is removed by the cloud mask.  This presents a problem both for computing our projections and approximating the covariance matrix.  

For the former the covariance matrix is approximated for a given time slice using only the pixels where there is data in that time slice, meaning that if the time slice has all but three pixels covered by cloud then the training data will be restricted to those three pixels, and the covariance matrix approximation will be 3 by 3.  The projection is then applied to the vector containing just those three pixels.  This means that the KL expansion must be computed for each time slice.

For the latter we start by restricting the training data to the pixels where there is data in the time slice being projected, as just described.  This almost certainly doesn't restrict the training data to pixels where training data isn't missing, so the contributions of those pixels must be removed.  First, the time slice means for the training data are calculated using just the pixels with data.  Second, pixels with missing data are set to zero.  Third, the means are subtracted from the data, and the data is multiplied by its transpose.  This procedure means that for the covariance matrix element corresponding to the mean of the products of pixels $m$ and $n$ across the time slices, if at least one of pixels $m$ or $n$ has missing data for a given time slice, then that time slice is effectively left out.  Normally the product of the data with its transpose would be divided by the number of time slices in the training data to get the mean, but instead each element is divided by the number of time slices not left out for that element due to missing data.

The problem with this method is that different elements in the covariance matrix approximation can be the mean from very different numbers of time slices, making the approximation error variance not uniform.  Instead, we consider a number of ways of filling in this missing data, and compare the performance against the baseline just described.  Since it would be prohibitively time consuming to pick new threshold values and the FTC by hand, and we want to isolate the improvement to the anomaly values, we optimize the overall metric for the optical only results, checking all combinations of the optical threshold and FTC, from 0.2 to 1 by 0.05 and 2 to 10 by 1 respectively.        

The first result is for Cube Mean, which replaces missing data with the mean of the non-missing values in a cube centered at the missing data.  The cube is truncated so as to not extend out of the training data, and if there are no non-missing values in the cube then the missing data is replaced with zero.  The best performance increase is 0.63 percent for a cube with side length 5.

\begin{figure}[htb]
\centering 
  \includegraphics[scale = 0.43]{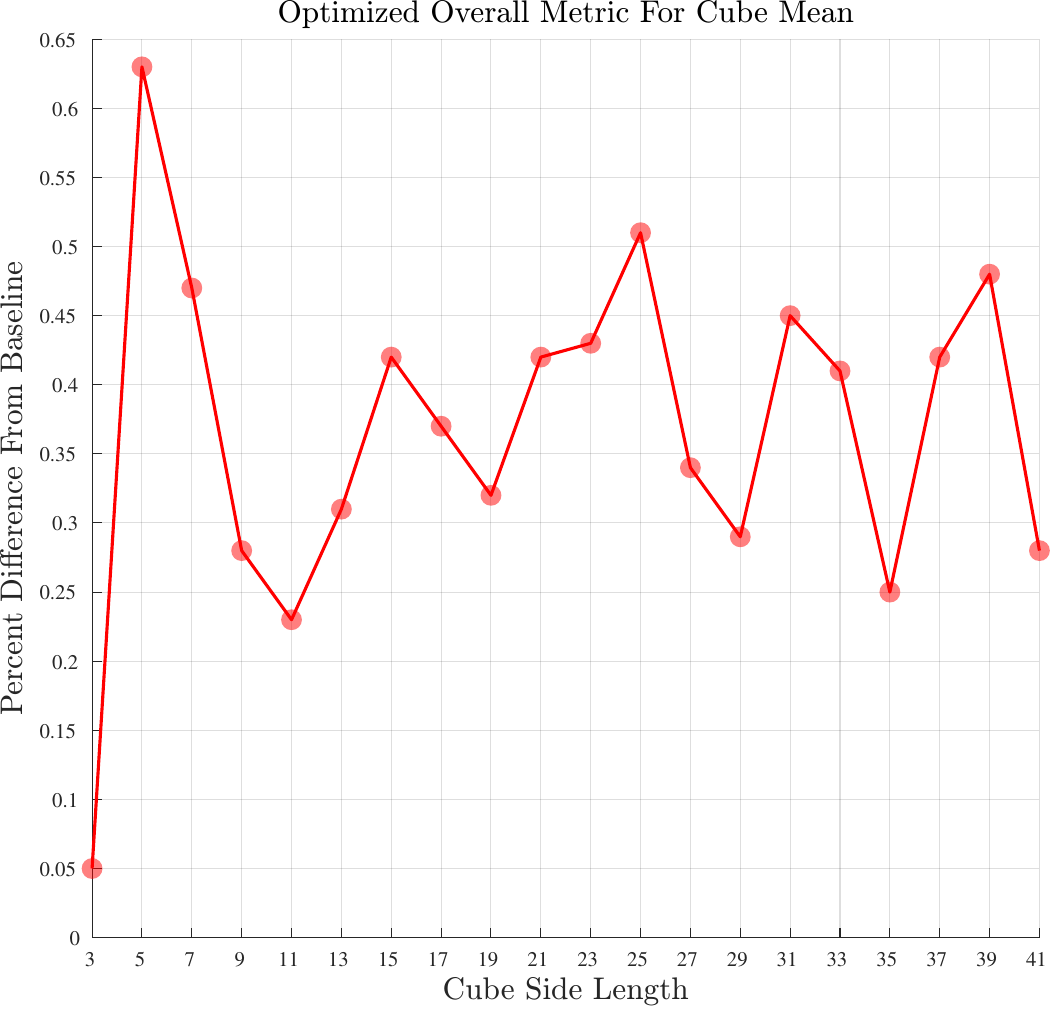}
    \caption{Percent difference in optimized overall metric for optical only results using Cube Mean, which replaces missing data with the mean of the values in a cube centered at the missing data.  If there are no values in the cube then the missing data is replaced with zero.}
    \label{results:cube}
\end{figure}

We also tried a number of nearest neighbor methods, where the missing data is replaced using either the mean of some number of neighbors or some other value, such as with zero for the cube mean.  We also include the results from filling all missing data using these values.  In the same order as the legend we have that:
\begin{itemize}

\item The baseline, Fill NaN, leaves missing data as missing. 
\item Fill 0 replaces all missing data with zero.  
\item Fill Global Mean replaces all missing data with the mean of all non-missing data.  
\item Fill Slice Mean replaces missing data with the mean of all non-missing data in its time slice. 
\item Fill Cube5 Mean replaces missing data with the mean in a cube with side length 5 centered at the missing data.  This is the optimal choice from \textbf{Figure} \ref{results:cube}.
\item Time replaces missing data with the mean of the k nearest neighbors in time, or 0 if there are no neighbors.
\item Space Fill 0 replaces missing data with the mean of k nearest neighbors in a cube of side length 3, or 0 if there are no neighbors.  Since there are large regions of missing data covered by clouds, it is likely that a square in space would need to be very large to contain any non-missing data, however it is also likely that a pixel covered by cloud won't be in the next time slice, so instead of just using a square in space we also extend up and down in time by one.
\item Space Fill NaN is Space Fill 0, except it leaves missing data as missing instead of replacing with 0.
\item Space Fill Time 5-Nearest is Space Fill 0 with the mean of the 5 nearest neighbors in time instead of 0.
\item Space Fill Global Mean is Space Fill 0 with the global mean instead of 0.
\item Space Fill Slice Mean is Space Fill 0 with the time slice mean for a given pixel instead of 0.
\end{itemize}


\begin{figure}[H]
\centering 
  \includegraphics[scale = 0.43]{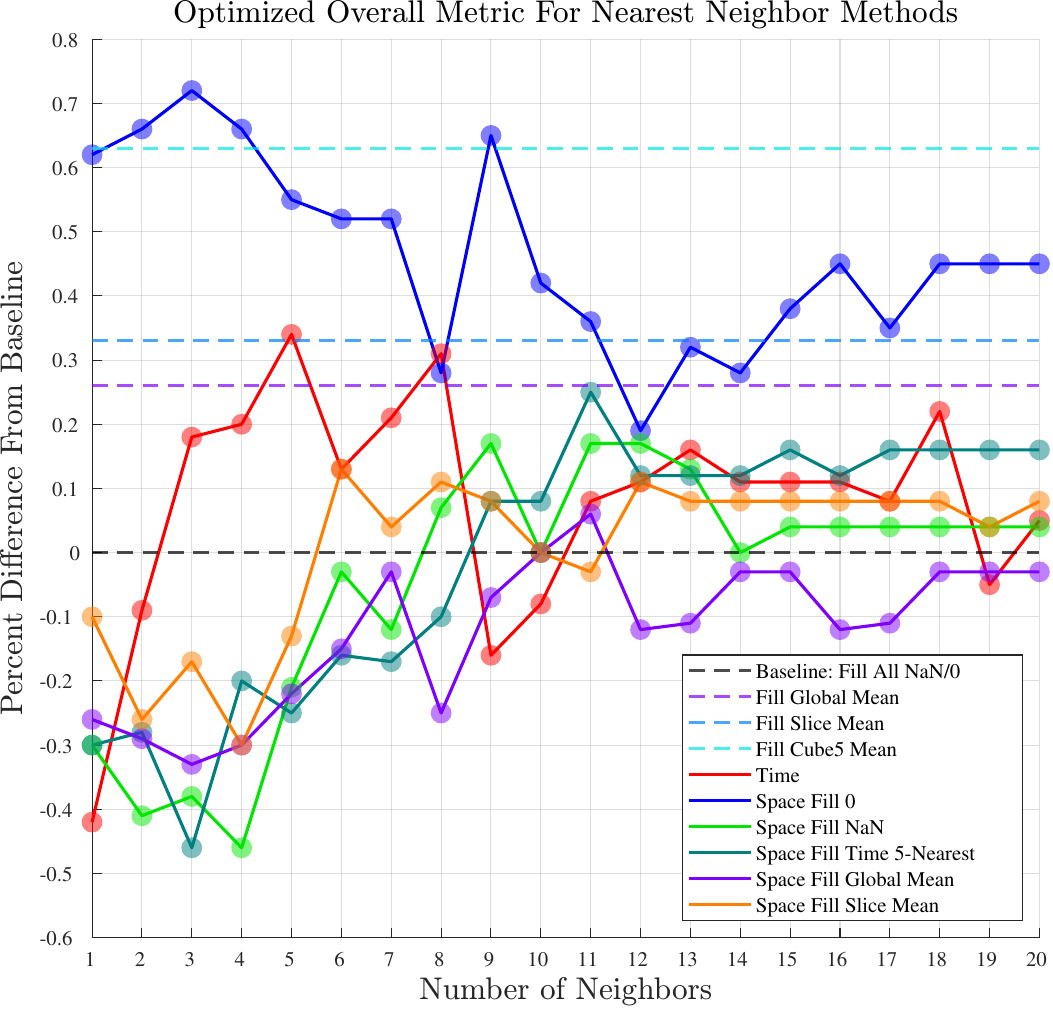}
    \caption{Percent difference in optimized overall metric for optical only results using various nearest neighbor methods, which are described above.}
    \label{results:nn}
\end{figure}

\begin{figure}[H]
\centering 
  \includegraphics[scale = 0.43]{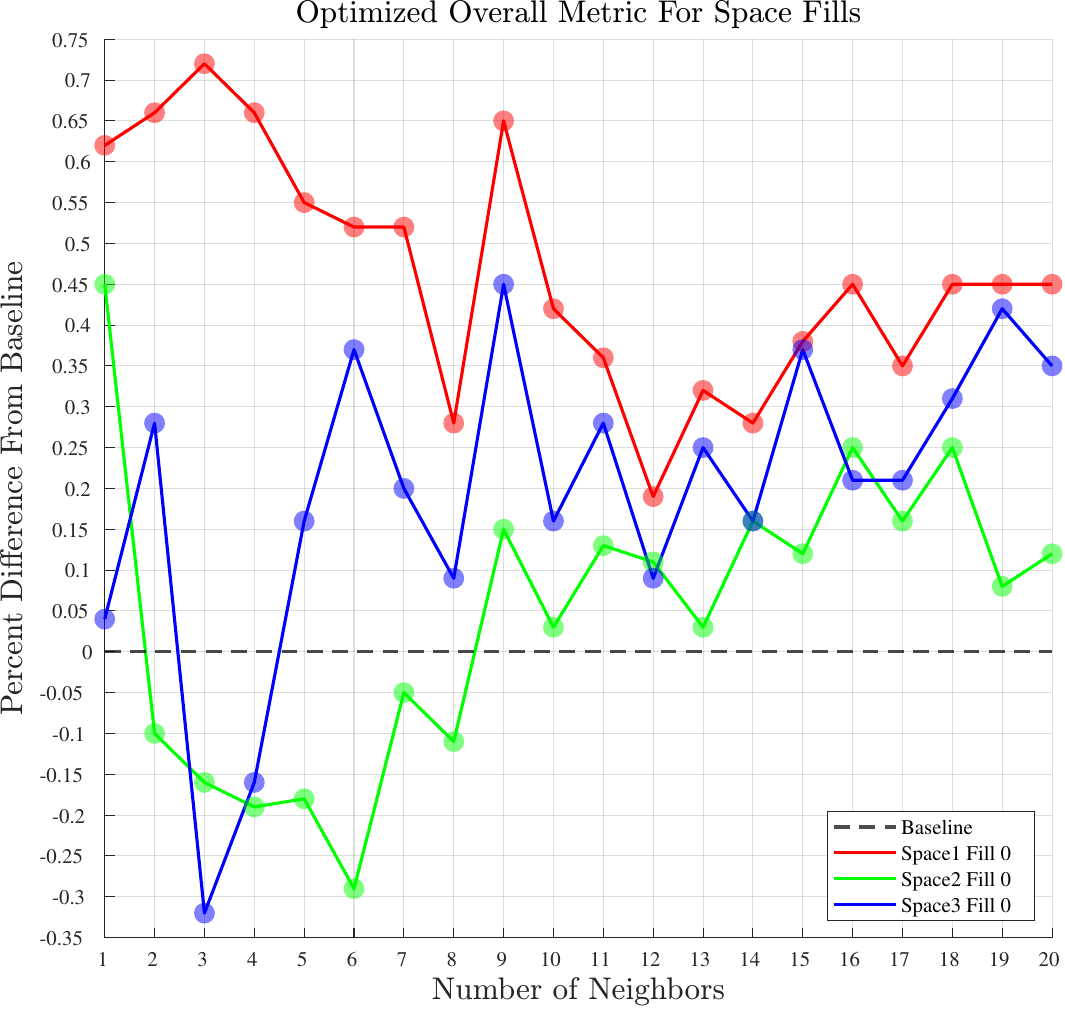}
    \caption{Percent difference in optimized overall metric for optical only results using Space1 Fill 0, Space2 Fill 0, and Space3 Fill 0. SpaceN Fill 0 replaces missing data with the mean of the k nearest neighbors in a rectangular prism centered at the missing data that extends out in space by N in both directions and up and down in time by 1, or replaces with zero if there are no neighbors in this region.}
    \label{results:space}
\end{figure}
The best result here, as can be seen in \textbf{Figure} \ref{results:nn}, is for Space Fill 0 with 3 neighbors, which gives a performance increase of 0.72 percent.  Strangely, the combination of Space Fill with Global Mean and Slice Mean underperform Space Fill 0 even though Fill Global Mean and Fill Slice Mean both overperform Fill 0.

Given that Space Fill 0 appears to be the best choice when only extending out by 1 in space, we also consider the results with larger regions in space, which can be seen in \textbf{Figure} \ref{results:space}.  SpaceN extends out by N in space to cover a square with side length $2N+1$, and still extends up and down in time by 1.  However, these larger regions do not improve performance, with the overall best choice still being Space1 Fill 0 using 3 neighbors.

\section{Additional Accuracy results}
\label{Supp:additional}

We present additional accuracy measures that are useful:

\begin{itemize}
    \item Balanced accuracy: This measure is useful in evaluating the classification performance for unbalanced datasets. Similarly to the overall metric approach, in \textbf{Figure} \ref{results:BA} the hybrid approach is significantly more robust, in particular, when the number of optical samples becomes low.  Furthermore, the accuracy variance is significantly lower for the hybrid approach, making it a more reliable classifier.
    \item F1-score: This measure is used to balance precision and recall. From \textbf{Figure} \ref{results:F1} under this measure, the hybrid method is significantly better than optical only.
    \item Users and Producers (Stable and Deforest) accuracies: \textbf{Figures} \ref{results:ps} through \ref{results:cd} show these four metrics, where the hybrid method is superior or comparable.
\end{itemize}
\newpage
\begin{figure*}[!htb]
\centering 
  \includegraphics[scale = 0.42]{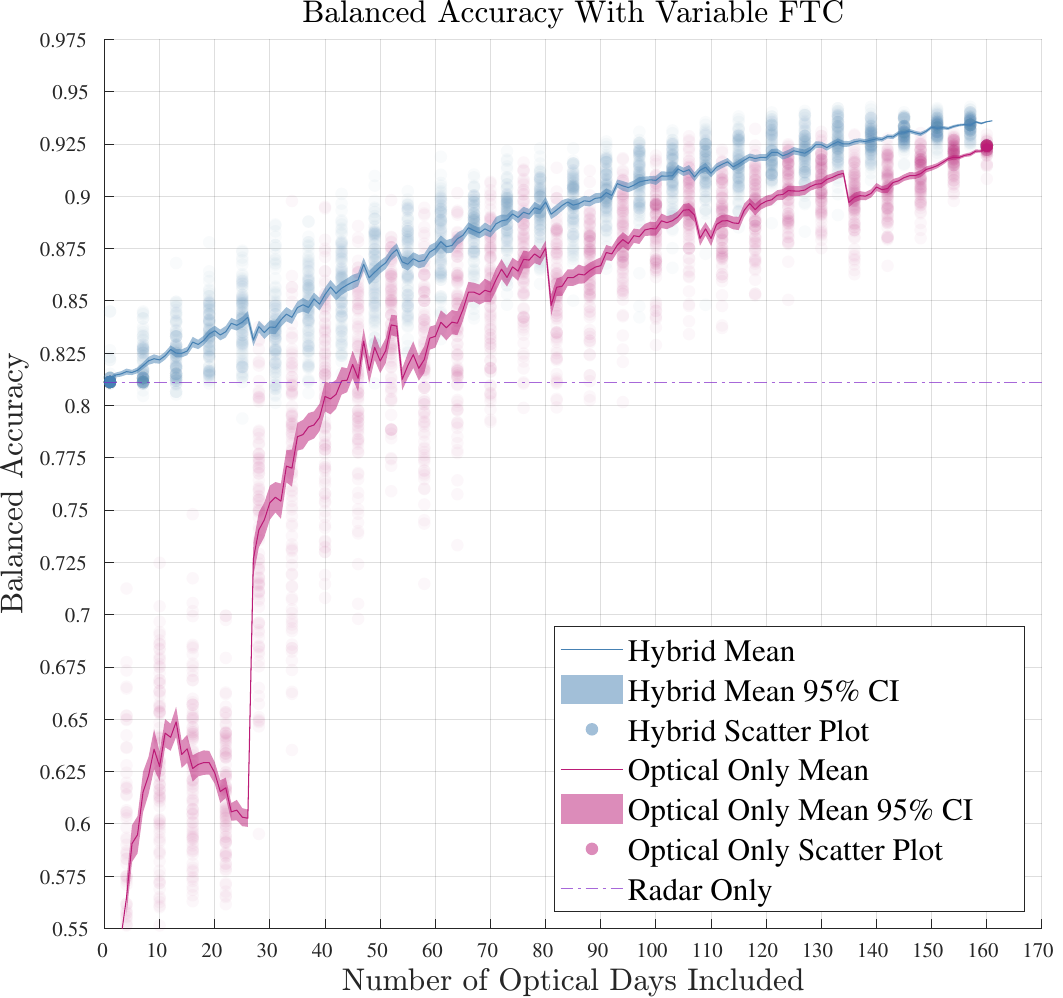}
  \includegraphics[scale = 0.42]{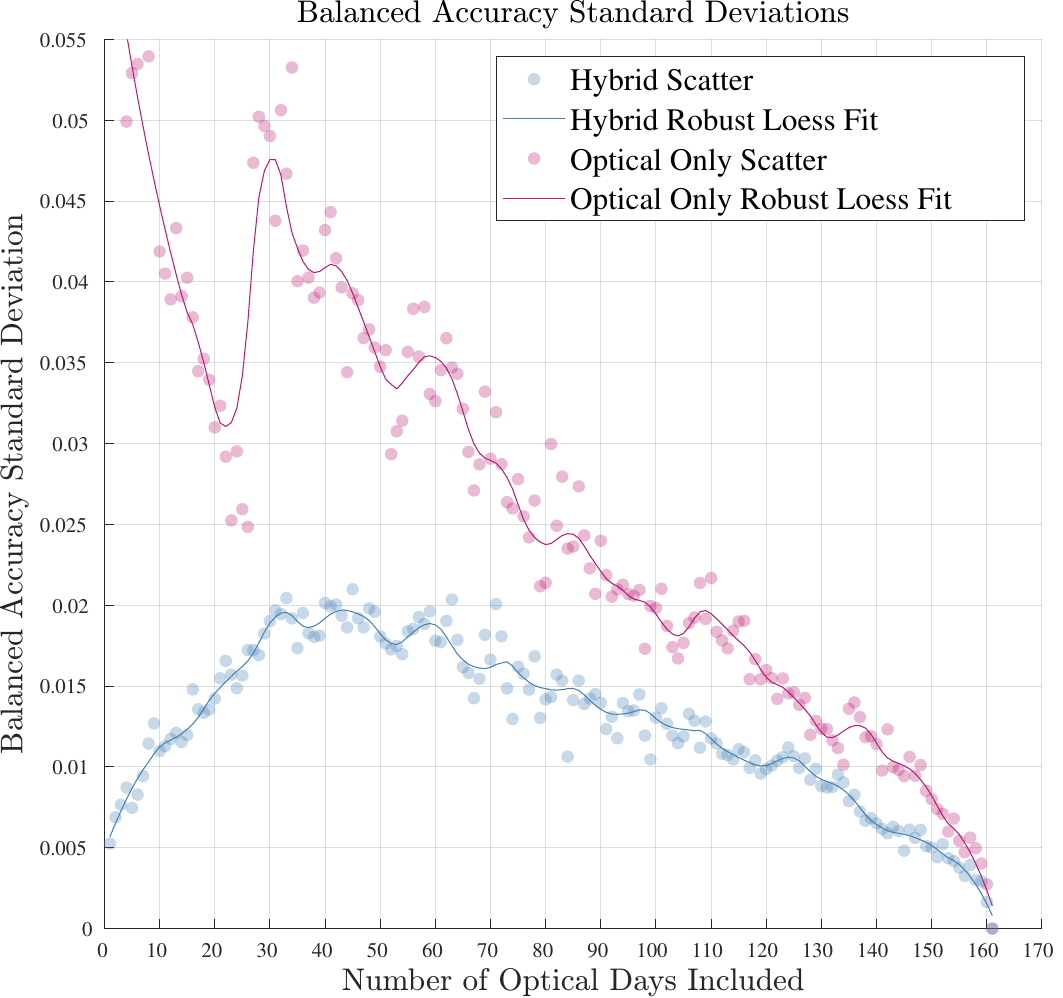}
    \caption{Balanced accuracy and standard deviations using variable FTC for hybrid and optical only.}
    \label{results:BA}
\end{figure*}

\begin{figure*}[!htb]
\centering 
  \includegraphics[scale = 0.42]{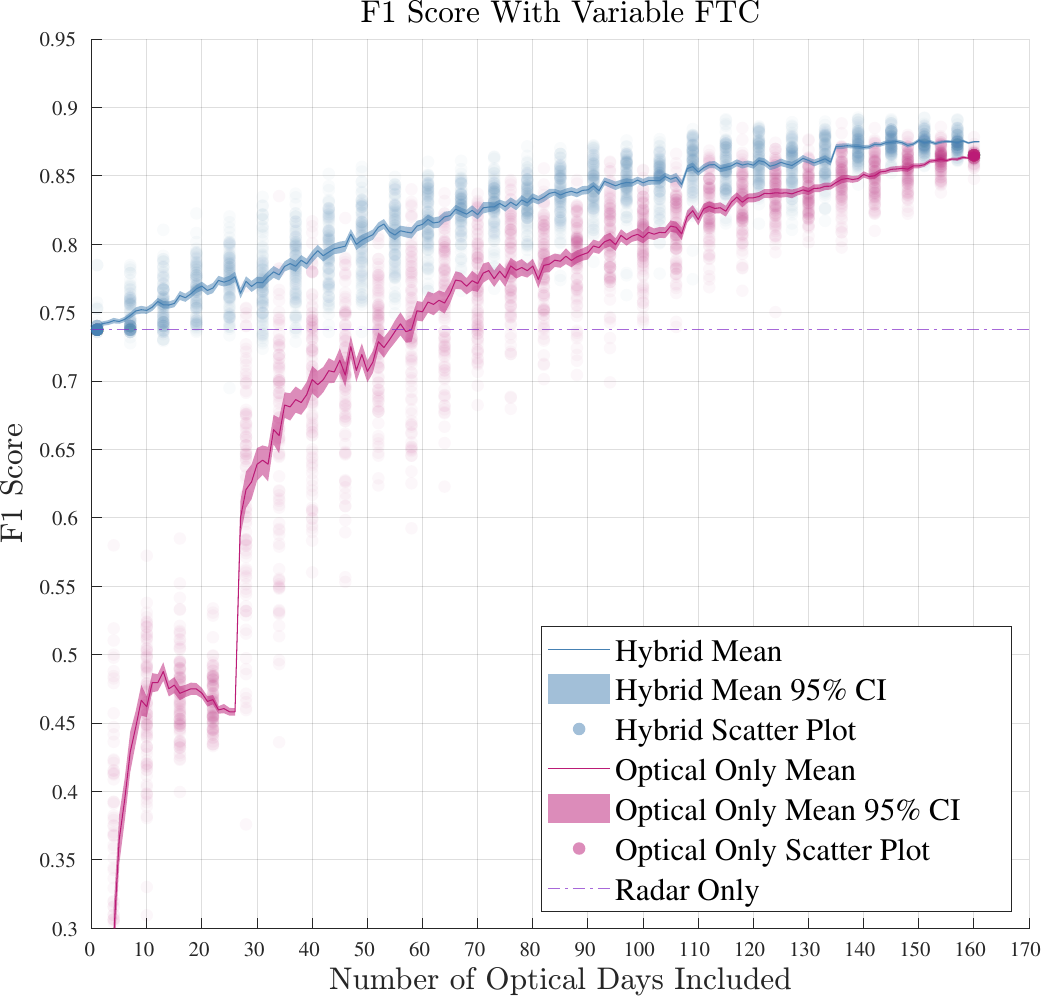}
  \includegraphics[scale = 0.42]{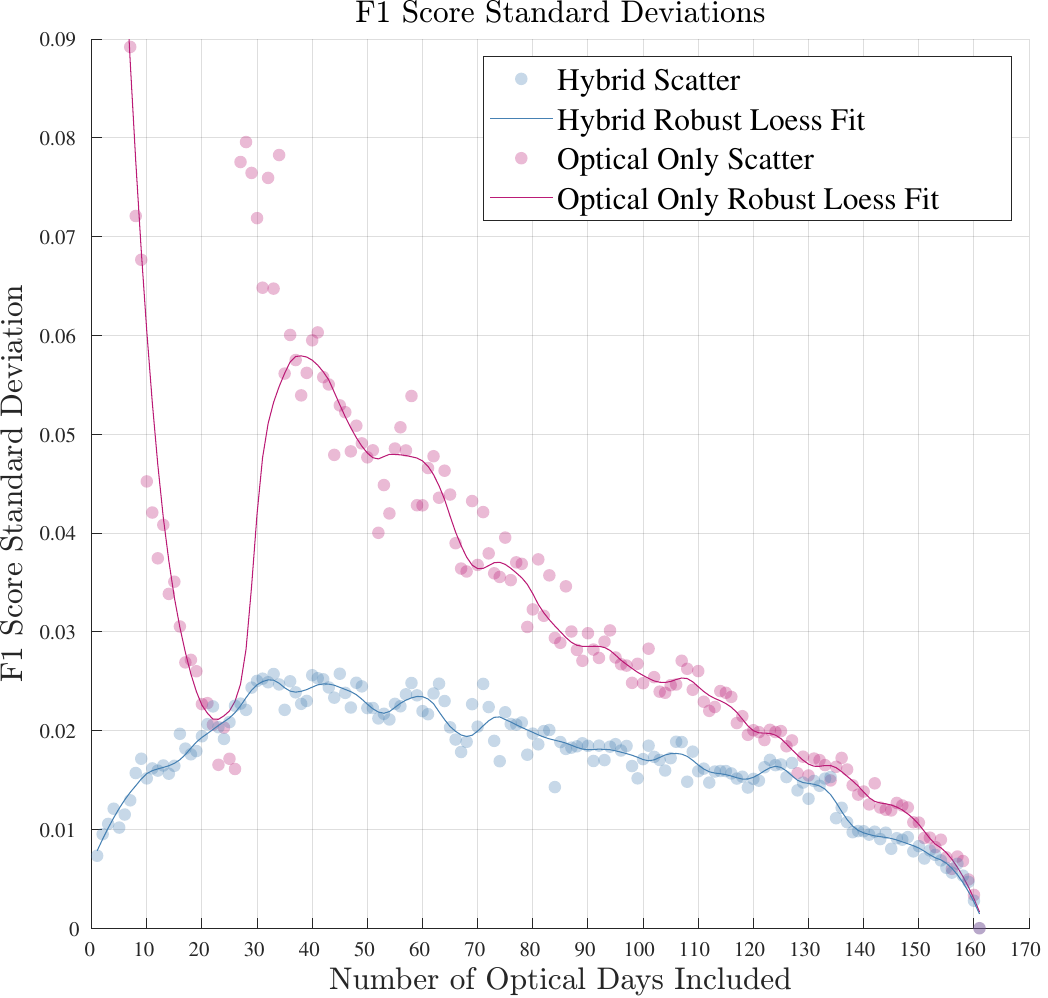}
    \caption{F1 Score accuracy and standard deviations using variable FTC for hybrid and optical only.}
    \label{results:F1}
\end{figure*}

\begin{figure*}[!htb]
\centering 
  \includegraphics[scale = 0.42]{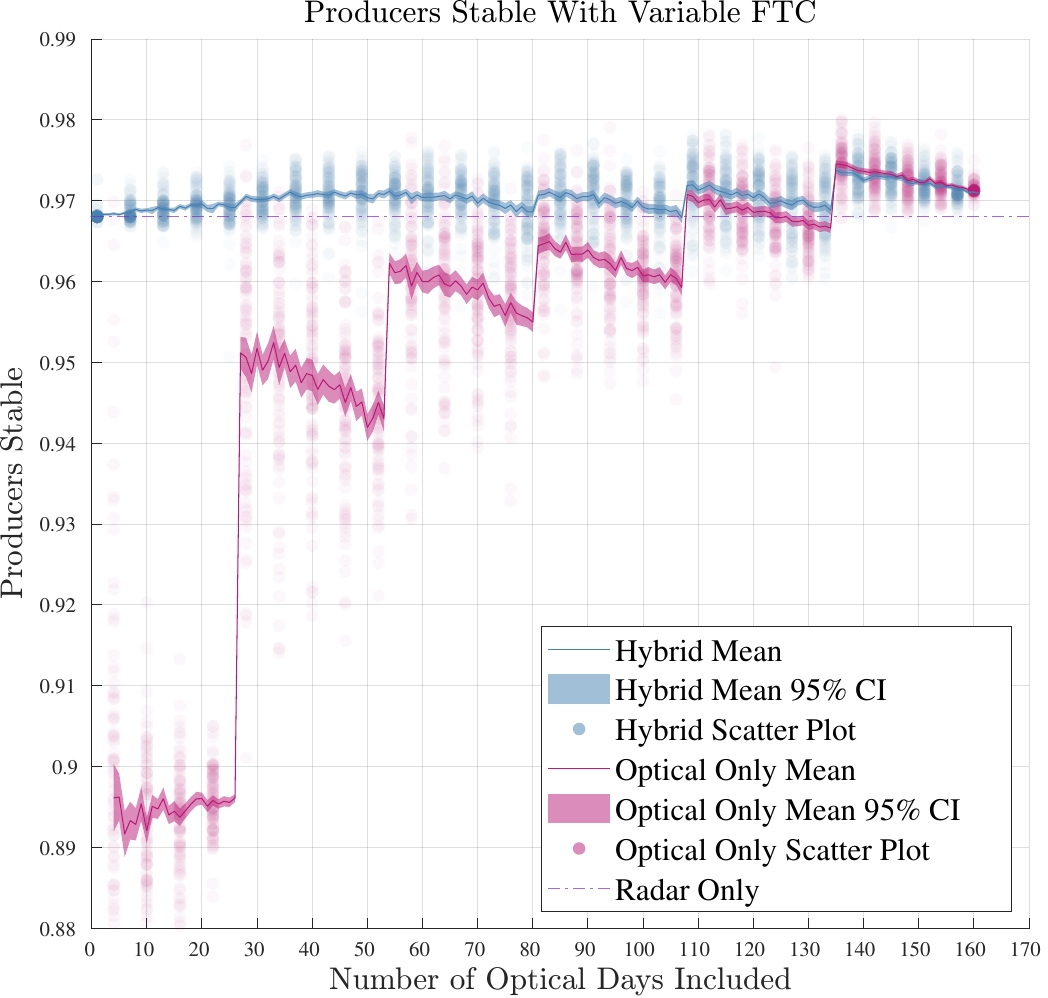}
     \includegraphics[scale = 0.42]{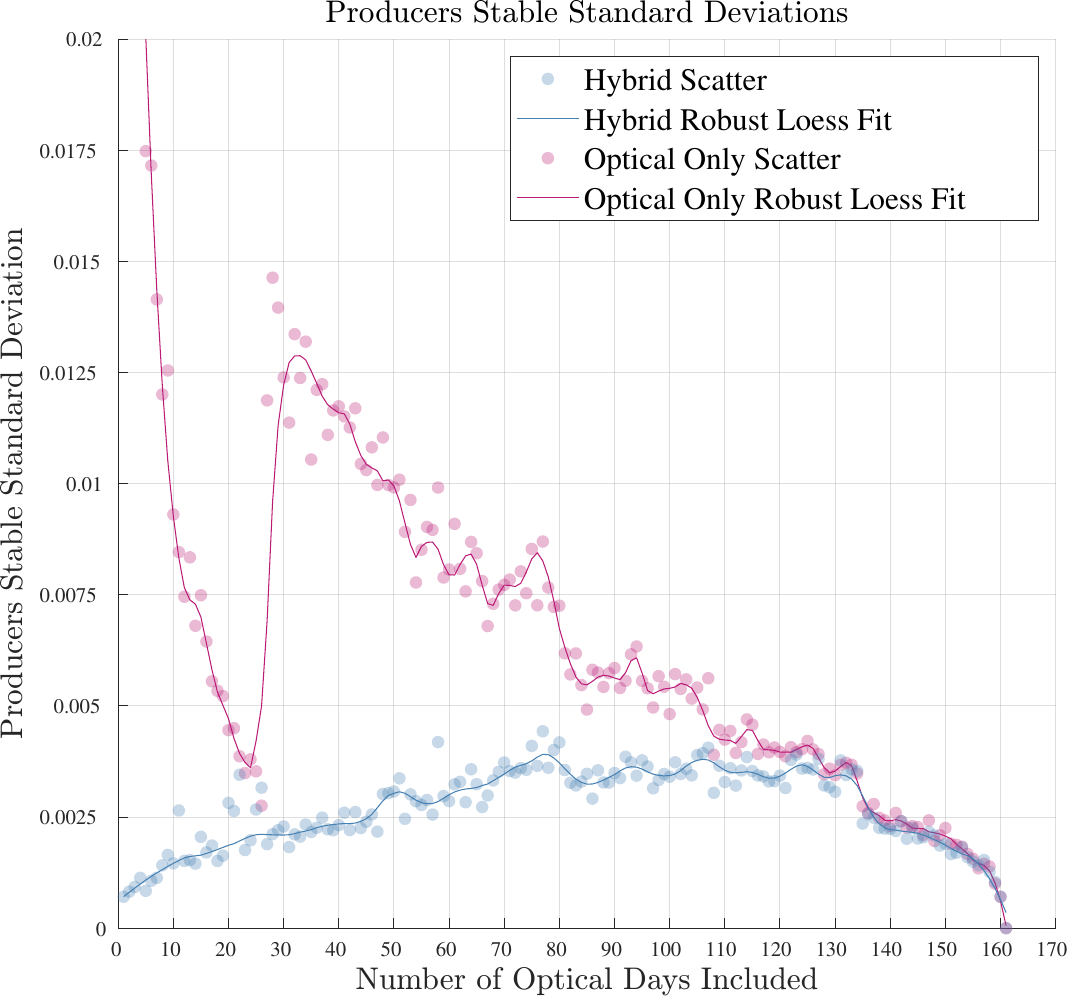}
    \caption{Producers Stable and standard deviations using variable FTC for hybrid and optical only.}
    \label{results:ps}
\end{figure*}

\begin{figure*}[!htb]
\centering 
  \includegraphics[scale = 0.42]{figures/validation/Producers_Deforest.pdf}
     \includegraphics[scale = 0.42]{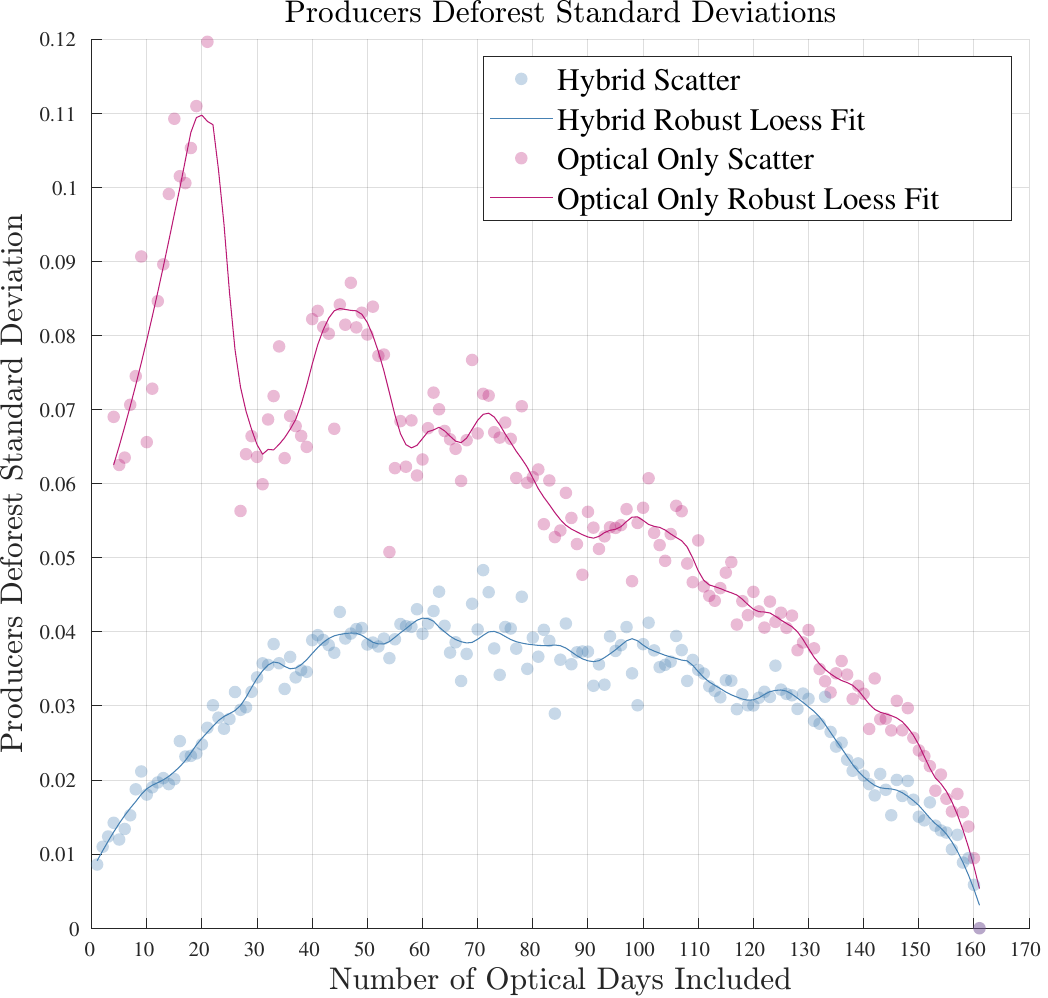}
    \caption{Producers Deforest and standard deviations using variable FTC for hybrid and optical only.}
    \label{results:pd}
\end{figure*}

\begin{figure*}[!htb]
\centering 
  \includegraphics[scale = 0.42]{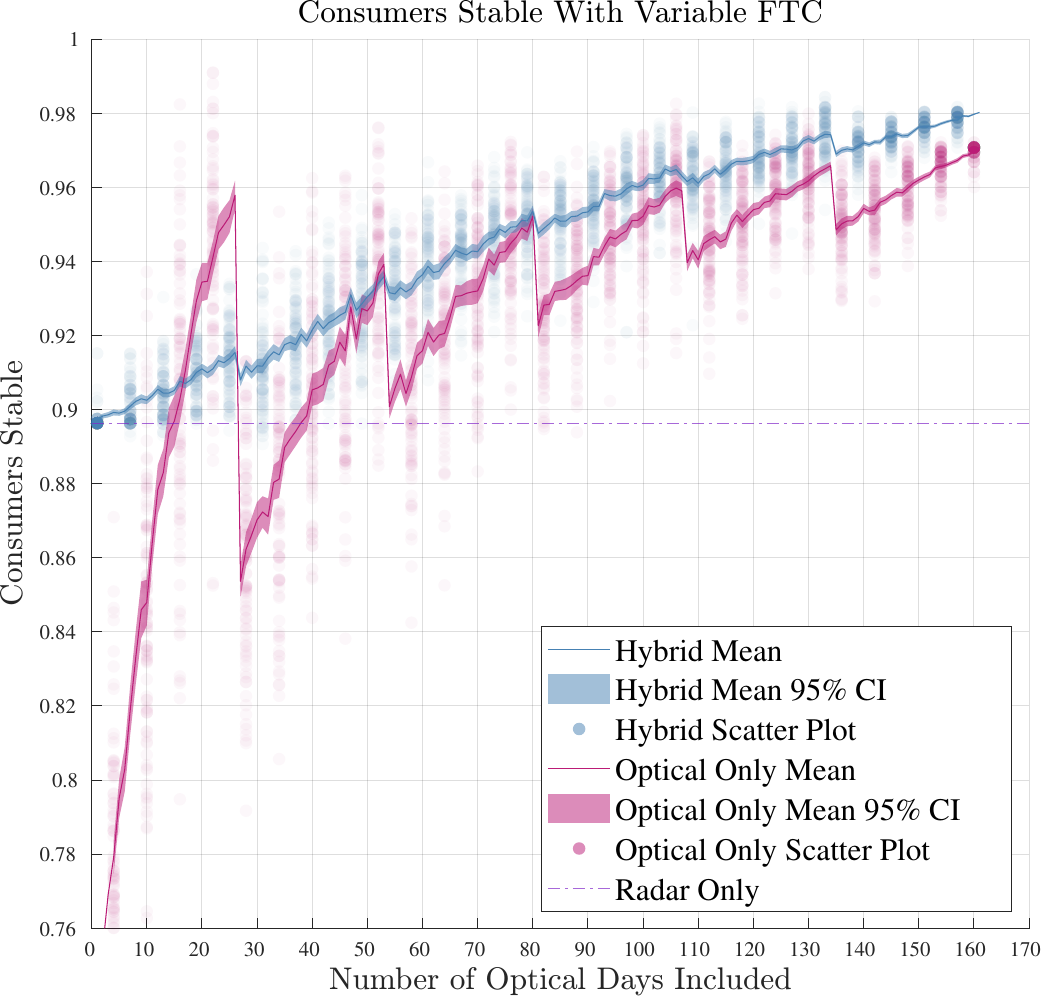}
     \includegraphics[scale = 0.42]{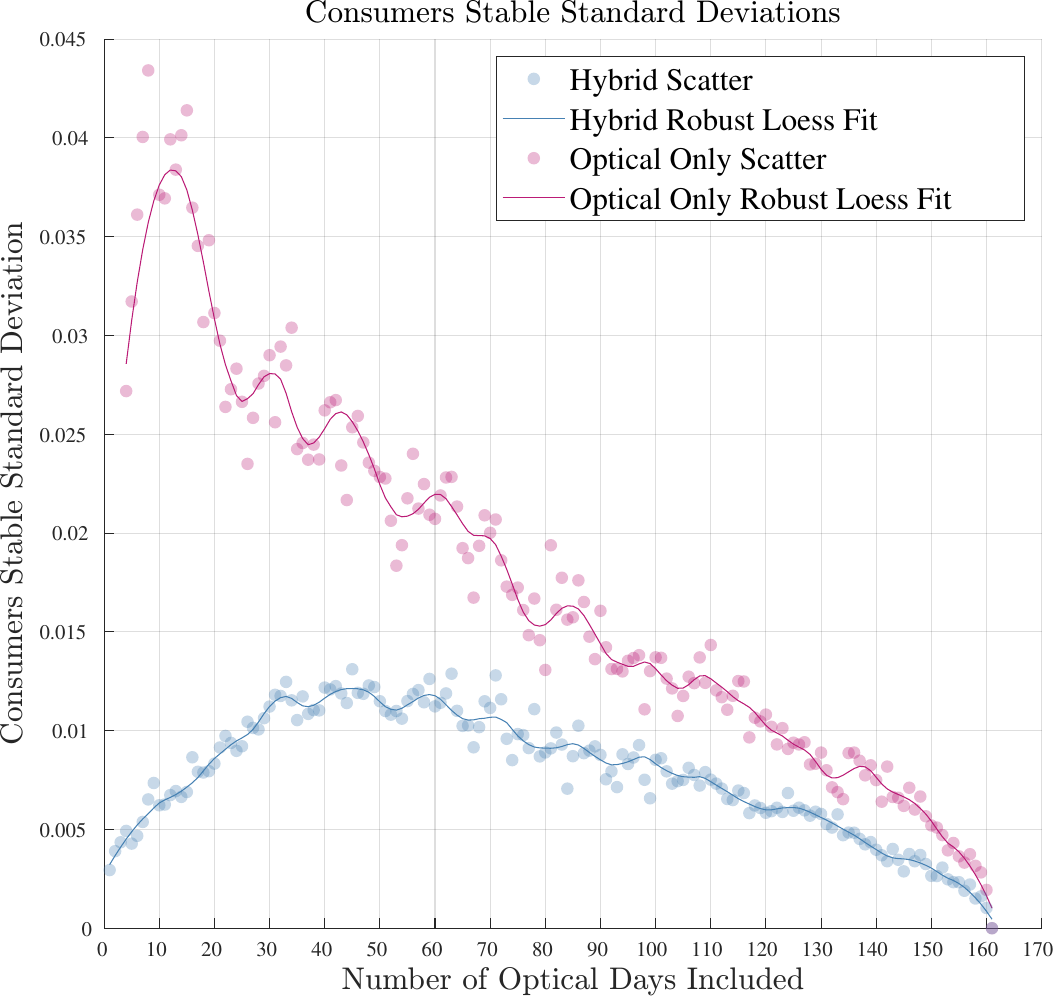}
    \caption{Consumers Stable and standard deviations using variable FTC for hybrid and optical only.}
    \label{results:cs}
\end{figure*}

\begin{figure*}[!htb]
\centering 
  \includegraphics[scale = 0.42]{figures/validation/Consumers_Deforest.pdf}
     \includegraphics[scale = 0.42]{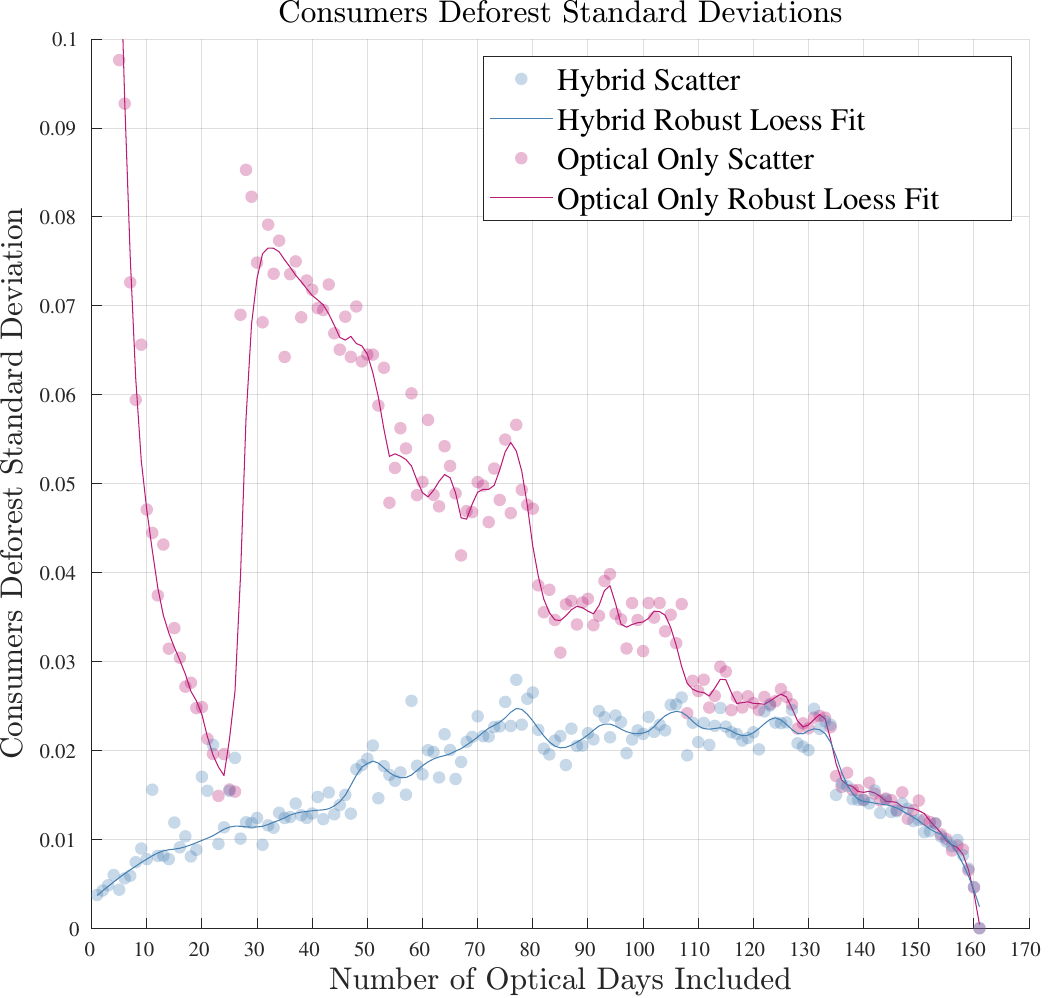}
    \caption{Consumers Deforest and standard deviations using variable FTC for hybrid and optical only.}
    \label{results:cd}
\end{figure*}


\end{document}